\documentclass[10pt]{article}

\usepackage[utf8]{inputenc}
\usepackage[T1]{fontenc}

\usepackage{epsf}
\usepackage{amsmath}

\allowdisplaybreaks

\usepackage[showframe=false]{geometry}
\usepackage{changepage}

\usepackage{epsfig}
\usepackage{amssymb}

\usepackage{amsthm}
\usepackage{setspace}
\usepackage{cite}
\usepackage{mcite}

\usepackage{algorithmic}  
\usepackage{algorithm}

\usepackage{shadow}
\usepackage{fancybox}
\usepackage{fancyhdr}

\usepackage{color}
\usepackage[usenames,dvipsnames,svgnames,table]{xcolor}
\newcommand{\bl}[1]{\textcolor{blue}{#1}}

\definecolor{mypurple}{rgb}{.4,.0,.5}

\usepackage[hyphens]{url}

\usepackage[colorlinks=true,
            linkcolor=black,
            urlcolor=blue,
            citecolor=purple]{hyperref}

\usepackage{breakurl}

\def\s{{\bf s}}

\def\y{{\bf y}}
\def\v{{\bf v}}
\def\x{{\bf x}}

\def\x{{\mathbf x}}

\def\s{{\bf s}}

\def\v{{\bf v}}
\def\x{{\bf x}}
\def\y{{\bf y}}
\def\z{{\bf z}}
\def\q{{\bf q}}

\def\c{{\bf c}}

\def\h{{\bf h}}

\def\tr{\mbox{Tr}}

\def\tr{{\rm tr}\,}

\def\be{\begin{equation}}
\def\ee{\end{equation}}
\def\ba{\left[\begin{array}}
\def\ea{\end{array}\right]}

\def\s{{\bf s}}

\def\v{{\bf v}}
\def\x{{\bf x}}
\def\y{{\bf y}}
\def\z{{\bf z}}
\def\q{{\bf q}}

\def\c{{\bf c}}

\def\e{{\bf e}}

\def\1{{\bf 1}}

\def\g{{\bf g}}
\def\0{{\bf 0}}

\def\mR{{\mathbb R}}

\def\mE{{\mathbb E}}

\def\mP{{\mathbb P}}

\def\os{\overline{\overline{\s}}}

\def\lp{\left (}
\def\rp{\right )}

\def\s{{\bf s}}

\def\y{{\bf y}}
\def\v{{\bf v}}
\def\x{{\bf x}}

\def\x{{\mathbf x}}

\def\s{{\bf s}}

\def\v{{\bf v}}
\def\x{{\bf x}}
\def\y{{\bf y}}
\def\z{{\bf z}}
\def\q{{\bf q}}

\def\c{{\bf c}}

\def\h{{\bf h}}

\def\tr{\mbox{Tr}}

\def\tr{{\rm tr}\,}

\def\be{\begin{equation}}
\def\ee{\end{equation}}
\def\ba{\left[\begin{array}}
\def\ea{\end{array}\right]}

\def\s{{\bf s}}

\def\v{{\bf v}}
\def\x{{\bf x}}
\def\y{{\bf y}}
\def\z{{\bf z}}
\def\q{{\bf q}}

\def\c{{\bf c}}

\def\e{{\bf e}}

\def\({\left (}
\def\){\right )}

\def\1{{\bf 1}}

\def\q{{\bf q}}

\def\g{{\bf g}}
\def\0{{\bf 0}}

\def\cL{{\mathcal L}}
\def\cA{{\mathcal A}}

\usepackage{xcolor}
\usepackage{color}

\definecolor{darkgreen}{rgb}{0, 0.4,0}

\definecolor{purplebrown}{rgb}{0.5,0.1,0.6}

\definecolor{ultclupcol}{rgb}{0.1,0.5,0.5}

\definecolor{mytrycolor}{rgb}{0.5,0.7,0.2}

\definecolor{ultclupcola}{rgb}{.5,0,.5}

\definecolor{shadebrown}{rgb}{0.1,0.1,0.9}
\definecolor{lightblue}{rgb}{0.2,0,1}

\usepackage{fancybox}
\usepackage{graphicx}
\usepackage{epstopdf}
\usepackage{epsfig}
\usepackage{wrapfig}
\usepackage{subfigure}

\usepackage{xcolor}
\usepackage{tcolorbox}
\tcbuselibrary{skins}

\newtcbox{\xmybox}{on line,
arc=7pt,
before upper={\rule[-3pt]{0pt}{10pt}},boxrule=0pt,
boxsep=0pt,left=6pt,right=6pt,top=0pt,bottom=0pt,enhanced, coltext=blue, colback=white!10!yellow}

\newtcbox{\xmyboxa}{on line,
arc=7pt,
before upper={\rule[-3pt]{0pt}{10pt}},boxrule=0pt,
boxsep=0pt,left=6pt,right=6pt,top=0pt,bottom=0pt,enhanced, colback=white!10!yellow}

\newtcbox{\xmyboxb}{on line,
arc=7pt,
before upper={\rule[-3pt]{0pt}{10pt}},boxrule=1pt,colframe=darkgreen!100!blue,
boxsep=0pt,left=6pt,right=6pt,top=0pt,bottom=0pt,enhanced, colback=white!10!yellow}

\newtcbox{\xmyboxc}{on line,
arc=7pt,
before upper={\rule[-3pt]{0pt}{10pt}},boxrule=.7pt,colframe=blue!100!blue,
boxsep=0pt,left=6pt,right=6pt,top=0pt,bottom=0pt,enhanced, coltext=blue, colback=white!10!yellow}

\newtcbox{\xmytboxa}{on line,
arc=7pt,
before upper={\rule[-3pt]{0pt}{10pt}},boxrule=.0pt,colframe=pink!50!yellow,
boxsep=0pt,left=6pt,right=6pt,top=0pt,bottom=0pt,enhanced, coltext=white, colback=blue!40!red}

\newtcbox{\xmytboxb}{on line,
arc=7pt,
before upper={\rule[-3pt]{0pt}{10pt}},boxrule=.0pt,colframe=pink!50!yellow,
boxsep=0pt,left=6pt,right=6pt,top=0pt,bottom=0pt,enhanced, coltext=white, colback=white!40!green}

\makeatletter
\newcommand\subsubsubsection{\@startsection{paragraph}{4}{\z@}{-2.5ex\@plus -1ex \@minus -.25ex}{1.25ex \@plus .25ex}{\normalfont\normalsize\bfseries}}
\newcommand\subsubsubsubsection{\@startsection{subparagraph}{5}{\z@}{-2.5ex\@plus -1ex \@minus -.25ex}{1.25ex \@plus .25ex}{\normalfont\normalsize\bfseries}}
\makeatother

\newtheorem{theorem}{Theorem}

\newtheorem{corollary}{Corollary}

\newtheorem{lemma}{Lemma}

\begin{document}

\begin{singlespace}

\title {Precise analysis of ridge interpolators under heavy correlations -- a Random Duality Theory view 
}
\author{
\textsc{Mihailo Stojnic
\footnote{e-mail: {\tt flatoyer@gmail.com}} }}
\date{}
\maketitle

\centerline{{\bf Abstract}} \vspace*{0.1in}

We consider fully row/column-correlated linear regression models and study several classical estimators (including minimum norm interpolators (GLS), ordinary least squares (LS), and ridge regressors). We show that \emph{Random Duality Theory} (RDT) can be utilized to obtain precise closed form characterizations of all estimators related optimizing quantities of interest, including the \emph{prediction risk} (testing or generalization error). On a qualitative level out results recover the risk's well known non-monotonic (so-called double-descent) behavior as the number of features/sample size ratio increases. On a quantitative level, our closed form results show how the risk explicitly depends on all key model parameters, including the problem dimensions and covariance matrices. Moreover, a special case of our results, obtained when intra-sample (or time-series) correlations are not present, precisely match the corresponding ones obtained via spectral methods in \cite{HMRT22,Dicker16,DobWag18,BHX20}.

\vspace*{0.25in} \noindent {\bf Index Terms: Random linear regression; Interpolators. Ridge estimators; Correlations}.

\end{singlespace}

\section{Introduction}
\label{sec:back}

While \emph{non-monotonic} behavior in parametric characterizations of various random structures has been known for a long time, it has received an enormous amount of attention in recent years. By a no surprise, a larger than ever popularity of machine learning (ML) and neural networks (NN) substantially contributed to this. Two lines of work, the neural network one (see, e.g., \cite{ZhangBHRV21,ZBHRV17,BelkinMM18,BelkinHM18,SGDSBW19,BHMM19}) and the statistical one (see, e.g., \cite{HMRT22,Dicker16,BHX20,DobWag18}), together with their interconnections, are among those that, in our view, led the way in bringing such a strong interest to these phenomena.

The first line, (in a way (re)initiated in \cite{ZhangBHRV21,ZBHRV17,BelkinMM18,BelkinHM18,SGDSBW19,BHMM19} (and further theoretically substantiated in e.g., \cite{MuthukumarVSS20,BHX20})), relates to the empirical observation of the so-called \emph{double-descent} phenomenon in neural networks generalization abilities. In particular, it was emphasized in \cite{BHMM19} that the generalization error has a U-shape dependence on the size of the network before and a (potentially surprising) second descent after the interpolating limit (see, also \cite{ASS20,MuthukumarVSS20,NeyshaburTS14,ZhangBHRV21,ZBHRV17} and particularly  \cite{VCR89} for early double-decent displays). This basically somewhat contradictorily  \cite{HastieFT01,HastieTF09} indicated that over-parameterizing networks and ensuring zero training error might not necessarily lead to a problem of poor generalization. Such a seeming contradiction was positioned within a bias-variance tradeoff discussion in \cite{BHMM19} and later on  re-discussed in a variety of contexts and via a variety of techniques throughout the vast literature (for an early view within a statistical learning context see also, e.g., \cite{BRT19}).

The second line of work is probably best represented through a particularly attractive setup of  \cite{HMRT22}. Namely,  \cite{HMRT22} seemingly circumvents a direct association with the zero-training of neural networks and instead focuses on potentially simpler models that might exhibit same or similar phenomena (concurrently, a similar ``considerations of simpler models'' approach was also undertaken in \cite{BLT20,BHX20}). Moreover, it chooses particularly attractive, famous linear regression model and study its classical associated linear estimators. While some of the results presented in \cite{HMRT22} were already known (see, e.g., \cite{Dicker16,DobWag18}), their further connection to the nonlinear models and neural networks (NN) enabled a broader view of the underlying phenomena and, in a way, reconnected back to the NN considerations of \cite{SGDSBW19,BHMM19,MuthukumarVSS20,BHX20}. That then brought on a strong attention from a variety of scientific communities substantially raising the awareness of their overall contemporary relevance and ultimately resulting in a remarkable popularity.

The key strength of \cite{HMRT22} (and earlier \cite{Dicker16,DobWag18}) actually lies in their technical contribution. Relying on spectral random matrix (free probability) theory, they conduct very \emph{precise} mathematical analyses which allow for a substantial superseding of mere qualitative confirmation of the existence of non-monotonicity (ascents, descents, double-descents, spikes, peaks and other irregularities). As is usually the case with such analyses, the analytical results of \cite{HMRT22,Dicker16,DobWag18} completely corroborate the numerical observations thereby providing a strong upgrade from substantially simpler, \emph{qualitative}, to way more challenging, \emph{quantitative}, non-monotonicity confirmations. Within the scenarios where they can be applied, the spectral methods are indeed super powerful. On the other hand, outside of them, they remain a bit limited, thereby rendering a need for potential alternatives.

In this paper we discuss such an alternative and highlight the utilization of an entirely different mathematical machinery, called Random Duality Theory (RDT) that we developed in a long line of work \cite{StojnicRegRndDlt10,StojnicCSetam09,StojnicICASSP10var,StojnicISIT2010binary,StojnicGenLasso10,StojnicGardGen13,StojnicDiscPercp13,StojnicGorEx10,StojnicUpper10}.  To enable a complementary view, we study similar, though \emph{fully correlated}, regression models and associated classical estimators as in \cite{HMRT22,Dicker16,DobWag18,BLT20}. We showcase the RDT's ability to produce very \emph{precise} and generic analyses as well. Moreover, for a special case of intra-sample uncorrelatedness, we obtain the very same results as in \cite{HMRT22,Dicker16,DobWag18} (given that we are dealing with precise analyses, such a concurrence must indeed be present if the axioms of mathematics are properly set).

\section{Mathematical model and definitions}
\label{sec:randlincons}

We start with the standard linear model
\begin{eqnarray}
\y=X\bar{\beta}+\e, \label{eq:model01}
\end{eqnarray}
where $X\in \mR^{m\times n}$ is a matrix of covariates or feature vectors, $\y\in\mR^m$ is a vector of responses, and $\e\in\mR^m$ is a noise vector (unless stated otherwise, throughout the paper all vectors are assumed to be \emph{column vectors}). The model is beyond well known and needs no further or extensively elaborate motivation. As such it has been utilized in a variety of scientific and engineering fields over the last at least a couple of centuries. For the easiness of exposition and the concreteness of presentation, we will connect it to the process of traning/testing data relations. Within such a context, we will view the rows of $X$ ($X_{i,:}, i=1,\dots,m$) as vectors of features that are presumed to create the responses $\y_i$ through a (potentially imperfect or noisy) linear combination
\begin{eqnarray}
\y_i=X_{i,:}\bar{\beta}+\e_i, \label{eq:model02}
\end{eqnarray}
where $\bar{\beta}\in\mR^n$ contains coefficients of the desired linearity. Given the access to $m$ data pairs $(X_{i,:},\y_i)$. one would want to recover $\bar{\beta}$ as accurately as possible. As hinted in say \cite{StojnicGenLasso10,StojnicGenSocp10,StojnicPrDepSocp10}, various measures of accuracy are of interest. We here consider the so-called \emph{prediction risk} or as it is often called (particularly within the machine learning and neural networks communities), the \emph{genratlization/testing error}. It will be formally defined in the following way. Assuming a statistical context, where $X$ and $\e$ are random, for an estimator of $\bar{\beta}$, say  $\hat{\beta}$, we have the prediction risk
\begin{eqnarray}
R(\bar{\beta},\hat{\beta})\triangleq \mE_{\x^{(t)}}\lp\lp\lp \x^{(t)}\rp^T\hat{\beta}-\lp \x^{(t)}\rp^T\bar{\beta} \rp^2 |X,\y\rp, \label{eq:model03}
\end{eqnarray}
where, $ \x^{(t)}\in\mR^n$ is the so-called \emph{testing} vector of features completely independent of the given \emph{training} data pairs $(X_{i,:},\y_i)$. While $R(\bar{\beta},\hat{\beta})$ is clearly deterministic when viewed through the prism of $\x^{(t)}$, it remains random when viewed through the prism of randomness od $X$ and $\y$. Along the same lines, throughout the paper, we adopt the convention that the subscripts next to $\mE$ and $\mP$ denote the randomness with respect to which the evaluation is taken (when it is clear from the context, the subscript(s) will be omitted). It should also be noted that the above prediction risk can alternatively be viewed as the so-called \emph{excess risk}. In such a context the prediction risk itself (given in (\ref{eq:model03})) could be redefined slightly differently with one option, for example, being replacement of  $\lp \x^{(t)}\rp^T\bar{\beta}$ by $\y^{(t)}$, where $\y^{(t)}$ would be the testing response. Switching between these alternative definitions is analytically relatively easy and, as it brings no conceptual novelty, we, throughout the paper, utilize the definition given in  (\ref{eq:model03}). One should also note that the prediction risk is defined as a function of both the ``true'' ($\bar{\beta}$) and the estimated ($\hat{\beta}$) linear coefficients. Mathematically speaking, its studying, therefore, falls into the category of -- what is in \cite{StojnicPrDepSocp10} termed as -- the problem dependent analysis.

The prediction risk defined in (\ref{eq:model03}) will be the central focus of our study below. Before we start the presentation of the technical analysis, we will need to introduce the types of scenarios that we consider. Three things are of critical importance to properly set the analytical infrastructure: \textbf{\emph{(i)}} The problem dimensions; \textbf{\emph{(ii)}} The problem statistics; and \textbf{\emph{(iii)}} The relevant estimators. We briefly discuss each of them in the following three subsections.

\subsection{Dimensions}
\label{sec:dims}

We consider large dimensional scenario, which will mean that \emph{all} key underlying dimensions go to infinity. Under such a context, we are interested in mathematically the hardest, so-called large $n$ \emph{linear} (proportional) regime, where
\begin{eqnarray}
\alpha=\lim_{n\rightarrow \infty} \frac{m}{n}, \label{eq:model03a0}
\end{eqnarray}
and $\alpha$ clearly remains constant as $n$ grows. It is not that difficult to see that $\alpha$ is the under-parametrization ratio (or the inverse of the over-parametrization ratio). We generically assume that all underlying matrices are of full rank. On occasion however, the context may allow considerations of lower ranks. While the appearance of any such instance will be followed by a specifically tailored discussion, we here emphasize that throughout the paper all matrix ranks always remain \emph{linearly} proportional to $n$. Also, to simplify the writing and to make the overall presentation look neater, we often avoid repeatedly using the $\lim_{n\rightarrow \infty}$ notation throughout derivations. It will be clear from the context that any expression involving $n$ is assumed to be taken in the $\lim_{n\rightarrow \infty}$ sense.

\subsection{Statistics}
\label{sec:stats}

To further emphasize the overall presentation neatness, we assume a Gaussian (albeit a very general one) statistical scenario. We assume that  $\bar{\beta}$ in (\ref{eq:model01}) is deterministic and fixed and that the only source of randomness are $X$ and $\e$. In particular, we take both $X$ and $\e$ to be comprised of Gaussian entries and, for two full rank (non-necessarily symmetric) matrices $A\in\mR^{n\times n}$ and $\overline{A}\in\mR^{m\times m}$, given in the following way
\begin{eqnarray}
X=ZA, \quad \quad \e=\sigma \overline{A}\v, \label{eq:model04}
\end{eqnarray}
where $Z\in\mR^{m\times n}$ and $\v\in\mR^m$ have iid standard normal entries and $\sigma$ is a fixed constant. One can then rewrite  (\ref{eq:model01}) as
\begin{eqnarray}
\y=X\bar{\beta}+\e=ZA\bar{\beta}+\sigma\overline{A}\v. \label{eq:model05}
\end{eqnarray}
It is then not that difficult to see that the rows of feature matrix $X$, $X_{i:,}$, are independent of each other but the elements within each row are correlated among themselves with the correlation matrix given as
\begin{eqnarray}
 \mE X_{i,:}^TX_{i,:}=A^TA. \label{eq:model06}
\end{eqnarray}
Analogously one also has for the noise correlations
\begin{eqnarray}
\mE \e\e^T=\overline{A}\overline{A}^T. \label{eq:model07}
\end{eqnarray}
Writing the SVDs of both $A$ and $\overline{A}$, we also have
\begin{eqnarray}
A=U\Sigma V^T, \quad\overline{A}=\overline{U}\overline{\Sigma} \overline{V}^T. \label{eq:model08}
\end{eqnarray}
Combining (\ref{eq:model05}) and (\ref{eq:model08}), we easily find
\begin{eqnarray}
\y=ZA\bar{\beta}+\overline{A}\v=ZU\Sigma V^T\bar{\beta}+\sigma \overline{U}\overline{\Sigma} \overline{V}^T\v. \label{eq:model09}
\end{eqnarray}
Due to rotational symmetry of the standard normal entries of $Z$ and $\v$, (\ref{eq:model09}) is statistically equivalent to the following
\begin{eqnarray}
\y=Z \Sigma V^T\bar{\beta}+\sigma \overline{U}\overline{\Sigma} \v. \label{eq:model010}
\end{eqnarray}
In other words, one can basically view $X$ and $\e$ as
\begin{eqnarray}
X=Z \Sigma V^T, \quad \e=\overline{U}\overline{\Sigma} \v. \label{eq:model010a0}
\end{eqnarray}
The above assumed statistics will be sufficient to present the key ideas. It essentially allows for the existence of the cross-sectional (or inter-features) correlations and intra-sample (or time-series) noise correlations. Later on (see, Section \ref{sec:rowcorr}), we will complete the model by allowing for intra-sample (time-series) type of correlations among the rows of $X$.

Unless otherwise stated, we will also assume that the testing features vector $\x^{(t)}$ has the same statistics as any of the rows of the training features matrix $X$. That basically means that we will assume that
\begin{eqnarray}
\lp\x^{(t)}\rp^T=\lp\z^{(t)}\rp^T \Sigma V^T, \label{eq:model010a1}
\end{eqnarray}
where $\z^{(t)}\in\mR^n$ is comprised of iid standard normals independent of $X$ and $\y$ (basically, $X$ and $\e$, or $Z$ and $\v$). It is useful to note that
\begin{eqnarray}
\mE \x^{(t)}\lp\x^{(t)}\rp^T = \mE V\Sigma \z^{(t)} \lp\z^{(t)}\rp^T \Sigma V^T = V\Sigma\Sigma V^T. \label{eq:model010a2}
\end{eqnarray}

Also, as there will typically be not much of a point in having both $\sigma$ and magnitude of $\bar{\beta}$ vary, we will (unless otherwise stated) scale one of them to pone. In other words, we will take $\|\bar{\beta}\|_2=1$.

\subsection{Estimators}
\label{sec:estimators}

Depending on the system dimensions (regimes of operation), we will consider three different estimators: 1) Minimum $\ell_2$ norm interpolator (also, often called the generalized least squares (GLS) estimator); 2) Ridge estimator; and 3) Plain least squares. As is well known, they are all directly related to each other. We will recall on some of these relations throughout the presentation as well. Here we formally introduce the estimators.

\underline{\textbf{\emph{1) Minimum $\ell_2$ norm interpolator (over-parameterized regime, $n>m$):}}}
\begin{eqnarray}
\beta_{gls}\triangleq \mbox{arg}\min_{\beta} & &  \|\beta\|_2^2 \nonumber \\
\mbox{subject to} & & X\beta=\y. \label{eq:model011}
\end{eqnarray}
It is not that difficult to see that the above optimization admits a closed form solution
\begin{eqnarray}
\beta_{gls} = X^T(XX^T)^{-1}\y, \label{eq:model012}
\end{eqnarray}
or, if written in an alternative form,
\begin{eqnarray}
\beta_{gls} = (X^TX)^{-1}X^T\y, \label{eq:model013}
\end{eqnarray}
where $(X^TX)^{-1}$ is the pseudo-inverse of $X^TX$ (only the positive eigenvalues are inverted).

\underline{\textbf{\emph{2) Ridge estimator (any $n$ and $m$):}}} For a given fixed $\lambda> 0$, we define the so-called ridge regression estimator as
\begin{eqnarray}
\beta_{rr}(\lambda)\triangleq \mbox{arg}\min_{\beta}  \lambda\|\beta\|_2^2 +\frac{1}{m}\|\y-X\beta\|_2^2. \label{eq:model014}
\end{eqnarray}
After trivially finding the derivatives, it is again not that difficult to see that the above optimization admits the following closed form solution
 \begin{eqnarray}
\beta_{rr}(\lambda) = \frac{1}{m}\lp \lambda I + \frac{1}{m} X^TX \rp^{-1}X^T \y= \lp \lambda m I +  X^TX \rp^{-1}X^T \y. \label{eq:model015}
\end{eqnarray}
Moreover, comparing (\ref{eq:model013}) and (\ref{eq:model015}), one easily recognizes that the GLS estimator is basically a special case of the ridge regression one, obtained for a particular value of ridge parameter $\lambda\rightarrow 0$. In other words, one has
 \begin{eqnarray}
\beta_{gls}=\lim_{\lambda\rightarrow 0} \beta_{rr}(\lambda). \label{eq:model016}
\end{eqnarray}
One should note that $\lambda$ is not allowed to take value zero in this regime.

\underline{\textbf{\emph{3) Plain least squares (under-parameterized regime, $n<m$):}}}
\begin{eqnarray}
\beta_{ls}\triangleq \mbox{arg}\min_{\beta} \frac{1}{m}\|\y-X\beta\|_2^2. \label{eq:model017}
\end{eqnarray}
After taking the derivative one again simply finds the closed form solution
\begin{eqnarray}
\beta_{ls} = (X^TX)^{-1}X^T\y. \label{eq:model018}
\end{eqnarray}
In this regime the definition of the ridge estimator can be extended to $\lambda\geq 0$ and one would then have
\begin{eqnarray}
\beta_{ls}= \beta_{rr}(0). \label{eq:model019}
\end{eqnarray}

All three estimators, (\ref{eq:model013}), (\ref{eq:model015}), and (\ref{eq:model018}) are clearly functions of $X$ and $\y$. In other words, they are functions of the available training data. Given the above assumed statistics of $X$, $\y$, and $\x^{(t)}$, we will in the rest of the paper be interested in proving a precise statistical characterization of the (random) prediction risk
\begin{eqnarray}
R(\bar{\beta},\hat{\beta})& =&  \mE_{\x^{(t)}}\lp\lp\lp \x^{(t)}\rp^T\hat{\beta}-\lp \x^{(t)}\rp^T\bar{\beta} \rp^2 |X,\y\rp \nonumber \\
& = & \mE_{\x^{(t)}} \lp \bar{\beta} -\hat{\beta}\rp^T {\x^{(t)}} \lp{\x^{(t)}}\rp^T \lp \bar{\beta} -\hat{\beta}\rp \nonumber \\
& = &  \lp \bar{\beta} -\hat{\beta}\rp^T V\Sigma \Sigma V^T\lp \bar{\beta} -\hat{\beta}\rp, \label{eq:model020}
\end{eqnarray}
where, depending on the regime of interest, $\hat{\beta}$ will be one of $\beta_{gls}$, $\beta_{rr}$, and $\beta_{ls}$. Our main concern however will be the over-parameterized regime and consequently the GLS estimator.

\subsection{Relevant literature and our contributions}
\label{sec:priorwork}


As mentioned earlier in the introduction, in our view, two lines of work \cite{ZhangBHRV21,ZBHRV17,BelkinMM18,BelkinHM18,SGDSBW19,BHMM19,ZBHRV17} and  \cite{HMRT22,Dicker16,DobWag18,BHX20} predominantly contributed to the growing interest in studying the non-monotonic random structures behavior.  Within the last few years, in both machine learning and statistical communities (as well as in many others), these initial considerations  have been followed with a large body of newly branched out different lines of work with many of them interconnecting among themselves and with others as well. As our results are of methodological type, we leave a detailed discussion related to each of these new directions for topic specific survey papers and here mention a few interesting ones that down the road might turn out to be particularly relevant.

For example, a strong connection to kernel based machine learning methods has been developed building further on initial interpolating type of observations \cite{SGDSBW19,BHMM19,NeyshaburTS14,ZhangBHRV21,ZBHRV17}. Relation between the problem architecture and implicit regularization (considered earlier \cite{NeyshaburTS14,GWBNS17} within deep learning and matrix factorization contexts) was studied in \cite{LiR21} within the kernel ridgeless regression  with generalization error (as function of spectral decay) displaying both monotonic and  U-shape (non-monotonic) behavior. An even, so to say, more non-monotonic behavior was observed through the appearance of multiple descents in \cite{LRZ20}. Another strong connection between the interpolators and running NN training algorithms (usually of the gradient type)  \cite{AliKT19,ChizatB18,LiL18a,OymSol19,ZCZG18,ADHLW19,DuZPS19,JGH18,LXSBNSP20,DuLL0Z19} was observed as well (in a way, it also relates to some earlier machine learning/regression cross-considerations \cite{CuckerS02,YRC07}).

On a more statistical side (and more closely related to our work), in parallel with \cite{HMRT22}, \cite{BHX20} first study similar mathematical regression setup as a way of providing a simple model for the appearance of the NN interpolating generalization double-descent phenomenon. Working in an uncorrelated (isotropic) standard normal context, it extends classical results of \cite{BF83} to over-parameterized regimes and achieves the double-descent. It then switches to random features model of \cite{RahimiR07}, utilizes empirical distribution of discrete Fourier transform (DFT) matrices (see, e.g., \cite{Farrell11}) and again obtains the double-descent. After these initial considerations \cite{HMRT22,Dicker16,DobWag18,BHX20}, further studies followed \cite{RMR20,WuXu20} (see also, e.g., \cite{XMRH21}). These contributions focus specifically on technical aspects. In particular, similarly to \cite{HMRT22}, \cite{RMR20,WuXu20} also rely on the spectral random matrix theory but utilize a bit different approach (more akin to \cite{LP11}) to obtain slightly different results (the initial results of \cite{HMRT22} assumed a random $\bar{\beta}$ which, in a way, is similar to a technical statistical assumption on $\bar{\beta}$  needed in \cite{RMR20}). The final results and main points of \cite{HMRT22} were, however, conceptually maintained in both \cite{RMR20,WuXu20}. There has been several contributions that also brought some conceptual novelties. For example, \cite{TB20,BLT20} consider potential benefits of the so-called benign over-fitting concept. In particular, \cite{BLT20} shows that, in linear regression, a mild over-parametrization can lead not only to a double-descent phenomenon but to a consistency of the interpolators provided that a specific ``\emph{benign}'' structure of the inter-features covariances is present (the so-called effective rank needs to be properly related to the problem dimensions, the sample size $m$ and the number of features $n$). Similar conclusion is reached for the ridge regression as well in \cite{TB20}. In \cite{BSMW22}, the factor regression models (FRM) are considered and the very same consistency property obtained in \cite{TB20,BLT20} is demonstrated for the FRM's GLS. Results of \cite{BSMW22} are in a way generalized in \cite{BBSMW21}, where a bit different estimator is introduced but \cite{BSMW22}'s`two key GLS properties, ``beating the null risk'' and consistency,  are shown as achievable again. \cite{MeiMon22} precisely analyzes nonlinear random features models of \cite{RahimiR07} and considers statistical universality (which was also discussed in \cite{MRSY19,HuL23}).

 As mentioned earlier, we study a standard linear regression setup (which is similar to the one from, say e.g., \cite{HMRT22,Dicker16,DobWag18,BHX20,BLT20}), albeit in a \emph{fully correlated} context where the intra-sample (or time-series) correlations are also allowed. Three classical estimators associated with such models given by (\ref{eq:model011}), (\ref{eq:model014}), and (\ref{eq:model017}) are considered. Differently from the spectral methods used in \cite{HMRT22,Dicker16,DobWag18,BHX20}, we here present the utilization of a completely different mathematical engine, called Random Duality Theory (RDT) (see, e.g., \cite{StojnicRegRndDlt10,StojnicCSetam09,StojnicICASSP10var,StojnicISIT2010binary,StojnicGenLasso10,StojnicGardGen13,StojnicDiscPercp13,StojnicGorEx10,StojnicUpper10}). Following step-by-step main RDT principles, we demonstrate the role of each of them in the analysis of the optimization programs
(\ref{eq:model011}), (\ref{eq:model014}), and (\ref{eq:model017}).  The RDT machinery is basically shown as sufficiently powerful to enable very precise determining of \emph{all} optimizing quantities associated with each of these three programs. Among them, of particular interest is the  \emph{prediction risk} (generalization/testing error)  from (\ref{eq:model020}). We observe that the GLS-interpolating risk (as a function of over-parametrization ratio $\frac{1}{\alpha}$) exhibits the double-descent with the potential for exhibiting the U-shape behavior in the interpolating regime ($\frac{1}{\alpha}>1$), thereby creating another aspect of non-monotonicity (this in general depends on the ratio $\mbox{SNR}=\frac{1}{\sigma^2}$, or just $\sigma$ if one keeps in mind the scaling invariance and the above mentioned fixing $\|\bar{\beta}\|_2=1$). Moreover, numerical evaluations suggest (possibly somewhat counter-intuitively) that the intra-sample correlations might not always imply necessarily higher GLS interpolator's prediction risk. Our results have explicit closed forms where both bias and variance parts of the risk can clearly be distinguished. Moreover, as they in special case of absent intra-sample correlations match the ones of \cite{HMRT22,Dicker16,DobWag18} (and in the fully isotropic case the ones of \cite{BHX20}), all insight, intuitions, and conclusions gained therein apply here as well.

\section{Precise analysis of all three estimators}
\label{sec:analuncorr}

We below provide a precise performance analysis of all the three above introduced classical estimators. Each of the estimators is discussed in a separate subsection. As mentioned earlier when discussing underlying statistics, we in this section consider scenarios with the cross-sectional (or inter-features) correlations and intra-sample (or time-series) noise correlations. The intra-sample (time-series) correlations among the rows of $X$ will be discussed in Section \ref{sec:rowcorr}. Separating these discussions allows to emphasize two key points: (i) there is a substantial analytical difference needed to complete the fully correlated models of Section \ref{sec:rowcorr}; and (ii) despite such a difference the analysis of the row-uncorrelated $X$ can still be utilized for the fully correlated scenario.

We start with the GLS interpolator as it relates to the over-parameterized regime that has gained a lot popularity in recent years due to its utilization in machine learning and the architecture design of neural networks.

\subsection{GLS interpolator}
\label{sec:analgls}

The analysis that we provide below will do slightly more than just providing the characterization of the GLS estimator. Namely, we will actually provide a precise statistical characterization of all the key quantiti4es associated with the optimization program (\ref{eq:model011}). It will turn out that the prediction risk from (\ref{eq:model020}) is among them as well. We start by defining the objective of the program (\ref{eq:model011})
\begin{eqnarray}
\xi_{gls} \triangleq \min_{\beta} & &  \|\beta\|_2^2 \nonumber \\
\mbox{subject to} & & X\beta=\y. \label{eq:randlincons1}
\end{eqnarray}
To ensure that the presentation is not overloaded with a tone of special cases and unnecessary tiny details that bring no conceptual insights, we assume that all the key quantities appearing in the derivations below are bounded (deterministically  or randomly). Recalling on (\ref{eq:model010}) and (\ref{eq:model010a0}), we proceed by writing a statistical equivalent to (\ref{eq:randlincons1})
\begin{eqnarray}
\xi_{gls} \triangleq \min_{\beta} & &  \|\beta\|_2^2 \nonumber \\
\mbox{subject to} & & Z\Sigma V^T\beta=Z\Sigma V^T\bar{\beta}+\sigma \overline{U}\overline{\Sigma}\v. \label{eq:randlincons1a0}
\end{eqnarray}
After writing the Lagrangian we further have
\begin{eqnarray}
\xi_{gls} = \min_{\beta} \max_{\nu} & & \lp \|\beta\|_2^2 +\nu^T \lp Z\Sigma V^T\beta-Z\Sigma V^T\bar{\beta}-\sigma \overline{U}\overline{\Sigma}\v \rp\rp. \label{eq:randlincons2}
\end{eqnarray}
One should now note that the above actually holds generically, i.e., it holds for any $Z$, $\v$, $V$, $\Sigma$, $\overline{U}$, $\overline{\Sigma}$, $\sigma$, and $\bar{\beta}$. It then automatically applies to the random instances as well. As precisely such instance are of our particular interest here, the following would be a simple summary of the probabilistic characterization of $\xi_{gls}$ that we are ultimately looking for.
 \begin{eqnarray}
\mbox{\bl{\textbf{Ultimate goal:}}}
\qquad\qquad \mbox{Given:} \quad && \alpha  =    \lim_{n\rightarrow \infty} \frac{m}{n}\in(1,\infty) \quad \mbox{and}\quad  V, \Sigma,\overline{U},\overline{\Sigma},\sigma,\bar{\beta}  \nonumber \\
\quad \mbox{find} \quad & & \xi_{gls}^{(opt)}\nonumber \\
\mbox{such that} \quad  && \forall \epsilon>0, \quad \lim_{n\rightarrow\infty}\mP_{Z,\v}\lp (1-\epsilon)\xi_{gls}^{(opt)}  \leq \xi_{gls} \leq (1+\epsilon)\xi_{gls}^{(opt)} \rp\longrightarrow 1.\nonumber \\
  \label{eq:ex4}
\end{eqnarray}
It goes without saying that, within their given intervals, $\alpha$ and $\epsilon$ are allowed to be arbitrarily large or small but they do not change as $n\rightarrow\infty$. As mentioned earlier, the subscripts next to $\mP$ denote all sources of randomness with respect to which the statistical evaluation is taken. To achieve the above goal we employ the powerful Random Duality Theory (RDT) machinery.

\subsubsection{Handling GLS estimator via RDT}
\label{sec:randlinconsrdt}

We find it useful to first briefly recall on the main RDT principles. After that we continue by showing, step-by-step, how each of them relates to the problems of our interest here.

\vspace{-.0in}\begin{center}
 	\tcbset{beamer,lower separated=false, fonttitle=\bfseries, coltext=black ,
		interior style={top color=yellow!20!white, bottom color=yellow!60!white},title style={left color=black!80!purple!60!cyan, right color=yellow!80!white},
		width=(\linewidth-4pt)/4,before=,after=\hfill,fonttitle=\bfseries}
 \begin{tcolorbox}[beamer,title={\small Summary of the RDT's main principles} \cite{StojnicCSetam09,StojnicRegRndDlt10}, width=1\linewidth]
\vspace{-.15in}
{\small \begin{eqnarray*}
 \begin{array}{ll}
\hspace{-.19in} \mbox{1) \emph{Finding underlying optimization algebraic representation}}
 & \hspace{-.0in} \mbox{2) \emph{Determining the random dual}} \\
\hspace{-.19in} \mbox{3) \emph{Handling the random dual}} &
 \hspace{-.0in} \mbox{4) \emph{Double-checking strong random duality.}}
 \end{array}
  \end{eqnarray*}}
\vspace{-.2in}
 \end{tcolorbox}
\end{center}\vspace{-.0in}
To make the presentation easier to follow, as in \cite{Stojnicridgefrm24}, we formalize all key results (simple to more complicated ones) as lemmas and theorems. Moreover, similarly to \cite{Stojnicridgefrm24}, although some analytical steps can occasionally be done a bit faster, we usually opt for a fully systematic approach ensuring that the presentation is self-contained.

\vspace{.1in}

\noindent \underline{1) \textbf{\emph{Algebraic characterization:}}}  The above algebraic discussions are summarized in the following lemma which ultimately provides a convenient representation of the objective $\xi_{gls}$.

\begin{lemma}(Algebraic optimization representation) Let  $V\in\mR^{n\times n}$ and $\overline{U}\in\mR^{m\times m}$ be two given unitary (orthogonal) matrices and let $\Sigma\in\mR^{n\times n}$ and $\overline{\Sigma}\in\mR^{m\times m}$ be two given diagonal positive definite matrices. Also, let vector  $\bar{\beta}\in\mR^n$ and scalar  $\sigma\geq 0$ be given as well. Assume that the components of any of these given objects are fixed real numbers that do not change as $n\rightarrow\infty$ and, for (possibly random) matrix $Z\in\mR^{m\times n}$ and vector $\v\in\mR^m$, let $\xi_{gls}$ be as in (\ref{eq:randlincons1}) or (\ref{eq:randlincons1a0}). Set
\begin{eqnarray}\label{eq:ta11}
f_{rp}(Z,\v) & \triangleq & \min_{\beta} \max_{\nu} \lp \|\beta\|_2^2 +\nu^T \lp Z\Sigma V^T\beta-Z\Sigma V^T\bar{\beta}-\sigma \overline{U}\overline{\Sigma}\v \rp\rp
 \hspace{.8in} (\bl{\textbf{random primal}})
\nonumber \\
\xi_{rp} & \triangleq & \lim_{n\rightarrow\infty } \mE_{Z,\v} f_{rp}(Z,\v).   \end{eqnarray}
Then
\begin{equation}\label{eq:ta11a0}
\xi_{gls}=f_{rp}(Z,\v) \quad \mbox{and} \quad \lim_{n\rightarrow\infty} \mE_{Z,\v}\xi_{gls} =\xi_{rp}.
\end{equation}
\label{lemma:lemma1}
\end{lemma}
\begin{proof}
 Follows automatically via the Lagrangian from (\ref{eq:randlincons2}).
\end{proof}

Clearly, assuming that $V$, $\Sigma$, $\overline{U}$, $\overline{\Sigma}$, $\sigma$, and $\beta$ are fixed and given, the above lemma holds for any $Z$ and $\v$. The RDT proceeds by imposing a statistics  on $Z$ and $\v$.


\vspace{.1in}
\noindent \underline{2) \textbf{\emph{Determining the random dual:}}} Following the common practice within the RDT, we utilize the concentration of measure, which here implies that for any fixed $\epsilon >0$,  one can write (see, e.g. \cite{StojnicCSetam09,StojnicRegRndDlt10,StojnicICASSP10var})
\begin{equation}
\lim_{n\rightarrow\infty}\mP_{Z,\v}\left (\frac{|f_{rp}(Z,\v)-\mE_{Z,\v}(f_{rp}(Z,\v))|}{\mE_{Z,\v}(f_{rp}(Z,\v))}>\epsilon\right )\longrightarrow 0.\label{eq:ta15}
\end{equation}
The following, so-called random dual theorem, is another key RDT ingredient. It provides a specific characterization of the objective of optimization programs (\ref{eq:model011}) and (\ref{eq:randlincons1a0}) used to obtain the GLS estimator, $\beta_{gls}$.

\begin{theorem}(Objective characterization via random dual) Assume the setup of Lemma \ref{lemma:lemma1} and let the components of $Z$ and $\v$ be iid standard normals. Additionally, let $\g\in\mR^m$ and $\h\in\mR^n$ be vectors that are also comprised of iid standard normals. Set
\vspace{-.0in}
\begin{eqnarray}
c_2 & \triangleq & \|\Sigma V^T\lp \beta -\bar{\beta}\rp\|_2 \nonumber \\
  f_{rd}(\g,\h,\v) & \triangleq &
 \min_{\beta}\max_{\nu} \lp \|\beta\|_2^2+c_2\nu^T\g+\|\nu\|_2\h^T\Sigma V^T\lp \beta -\bar{\beta}\rp -\sigma \nu^T\overline{U}\overline{\Sigma} \v \rp
   \hspace{.7in} (\bl{\textbf{random dual}})
  \nonumber \\
 \xi_{rd} & \triangleq & \lim_{n\rightarrow\infty} \mE_{\g,\h,\v} f_{rd}(\g,\h,\v)  .\label{eq:ta16}
\vspace{-.0in}\end{eqnarray}
One then has \vspace{-.0in}
\begin{eqnarray}
  \xi_{rd} & \triangleq & \lim_{n\rightarrow\infty} \mE_{\g,\h,\v} f_{rd}(\g,\h,\v)
  \leq
  \lim_{n\rightarrow\infty} \mE_{Z,\v} f_{rp}(Z,\v)  \triangleq  \xi_{rp}. \label{eq:ta16a0}
\vspace{-.0in}\end{eqnarray}
and
\begin{eqnarray}
 \lim_{n\rightarrow\infty}\mP_{\g,\h,\v} \lp f_{rd}(\g,\h,\v)\geq (1-\epsilon)\xi_{rd}\rp
 \leq  \lim_{n\rightarrow\infty}\mP_{Z,\v} \lp f_{rp}(Z,\v)\geq (1-\epsilon)\xi_{rd}\rp.\label{eq:ta17}
\end{eqnarray}
\label{thm:thm1}
\end{theorem}\vspace{-.17in}
\begin{proof}
  Follows as an immediate application of the Gordon's probabilistic comparison theorem (see, e.g., Theorem B in \cite{Gordon88} and also Theorem 1, Corollary 1, and Section 2.7.2 in \cite{Stojnicgscomp16}).
\end{proof}

 \vspace{.1in}
\noindent \underline{3) \textbf{\emph{Handling the random dual:}}} We start by solving the inner maximization over $\nu$ and obtain
\begin{eqnarray}
  f_{rd}(\g,\h,\v)
  & \triangleq &
\min_{\beta}\max_{\nu} \lp \|\beta\|_2^2+c_2\nu^T\g+\|\nu\|_2\h^T\Sigma V^T\lp \beta -\bar{\beta}\rp -\sigma \nu^T\overline{U}\overline{\Sigma} \v \rp
\nonumber \\
  & = &
\min_{\beta}\max_{\|\nu\|_2} \lp \|\beta\|_2^2+\|\nu\|_2\h^T\Sigma V^T\lp \beta -\bar{\beta}\rp +\|\nu\|_2\|c_2\g-\sigma \overline{U}\overline{\Sigma} \v \|_2\rp \nonumber \\
  & = &
\min_{\beta}\max_{\|\nu\|_2} \lp \|V^T\beta\|_2^2+\|\nu\|_2\h^T\Sigma \lp V^T\beta -V^T\bar{\beta}\rp +\|\nu\|_2\|c_2\g-\sigma \overline{U}\overline{\Sigma} \v \|_2\rp \nonumber \\
  & = &
\min_{\x,c_2\geq 0}\max_{\nu_s\geq 0} \lp \|\x\|_2^2+\nu_s\h^T\Sigma \lp \x -\c\rp +\nu_s\|c_2\g-\sigma \overline{U}\overline{\Sigma} \v \|_2\rp, \label{eq:ta18a0}
\end{eqnarray}
where the third equality follows due to $V$ being unitary and the fourth is a consequence of the following change of variables
\begin{eqnarray}
  \x=V^T\beta, \quad \c\triangleq V^T\bar{\beta}, \quad \nu_s=\|\nu\|_2, \label{eq:ta18a0b0}
\end{eqnarray}
and we also note that
\begin{eqnarray}
  c_2=\|\Sigma\lp\x-\c\rp\|_2. \label{eq:ta18a0b0c0}
\end{eqnarray}
After setting
\begin{eqnarray}
  \cL\lp\x,c_2,\nu_s;\g,\h,\v\rp \triangleq \|\x\|_2^2+ \nu_s\h^T\Sigma \lp \x -\c\rp + \nu_s\|c_2\g-\sigma \overline{U}\overline{\Sigma} \v \|_2, \label{eq:ta18a0b1}
\end{eqnarray}
it is not that difficult to see that (\ref{eq:ta18a0}) can be rewritten as
\begin{eqnarray}
  f_{rd}(\g,\h,\v)
  =
\min_{\x,c_2\geq 0}\max_{\nu_s\geq 0}  \cL\lp\x,c_2,\nu_s;\g,\h,\v\rp. \label{eq:ta18a0b2}
\end{eqnarray}
 Keeping $c_2$ and $\nu_s$ fixed, we consider the following minimization  over $\x$
\begin{eqnarray}
  \cL_1\lp c_2,\nu_s\rp \triangleq  \min_{\x} & & \|\x\|_2^2+ \nu_s\h^T\Sigma \lp \x -\c\rp + \nu_s\|c_2\g-\sigma \overline{U}\overline{\Sigma} \v \|_2 \nonumber \\
  \mbox{subject to} & &   \|\Sigma\lp\x-\c\rp\|_2=c_2, \label{eq:ta18a0b2c0}
\end{eqnarray}
where the functional dependence on randomness is kept implicit to lighten the exposition. After writing the Lagrangian we then obtain
\begin{eqnarray}
  \cL_1\lp c_2,\nu_s\rp
  & = &  \min_{\x}\max_{\gamma}  \|\x\|_2^2+ \nu_s\h^T\Sigma \lp \x -\c\rp + \nu_s\|c_2\g-\sigma \overline{U}\overline{\Sigma} \v \|_2 +\gamma \|\Sigma\lp\x-\c\rp\|_2^2 -\gamma c_2^2 \nonumber \\
  & = &  \max_{\gamma}   \min_{\x} \|\x\|_2^2+ \nu_s\h^T\Sigma \lp \x -\c\rp + \nu_s\|c_2\g-\sigma \overline{U}\overline{\Sigma} \v \|_2 +\gamma \|\Sigma\lp\x-\c\rp\|_2^2 -\gamma c_2^2 \nonumber \\
  & = &  \max_{\gamma}   \min_{\x}   \cL_2\lp \x,\gamma;c_2,\nu_s\rp, \label{eq:ta18a0b2c1}
\end{eqnarray}
where the second equality is implied by the strong Lagrange duality and the third by introducing
\begin{eqnarray}
  \cL_2\lp \x,\gamma;c_2,\nu_s\rp
  \triangleq
  \|\x\|_2^2+ \nu_s\h^T\Sigma \lp \x -\c\rp + \nu_s\|c_2\g-\sigma \overline{U}\overline{\Sigma} \v \|_2 +\gamma \|\Sigma\lp\x-\c\rp\|_2^2 -\gamma c_2^2 . \label{eq:ta18a0b2c2}
\end{eqnarray}
 After taking the derivative, we find
 \begin{eqnarray}
  \frac{\cL_2\lp\x,\gamma;c_2,\nu_s\rp}{d\x} = 2\x+\nu_s\Sigma\h +2\gamma\Sigma\Sigma\lp \x-\c\rp. \label{eq:ta18a0b3}
\end{eqnarray}
Equalling the above derivative to zero gives
\begin{eqnarray}
 \hat{\x}= \lp I +\gamma\Sigma\Sigma\rp^{-1}\lp -\frac{1}{2}\nu_s \Sigma\h +\gamma\Sigma\Sigma\c\rp. \label{eq:ta18a0b4}
\end{eqnarray}
Plugging this back in (\ref{eq:ta18a0b2c2})   further gives
\begin{eqnarray}
 \min_{\x}   \cL_2\lp \x,\gamma;c_2,\nu_s\rp & = &
  \cL_2\lp \hat{\x},\gamma;c_2,\nu_s\rp \nonumber \\
 & = &
-\lp -\frac{1}{2}\nu_s \Sigma\h +\gamma\Sigma\Sigma\c\rp^T   \lp I +\gamma\Sigma\Sigma\rp^{-1}\lp -\frac{1}{2}\nu_s \Sigma\h +\gamma\Sigma\Sigma\c\rp
\nonumber \\
   & & -\nu_s\h^T\Sigma\c+ \nu_s\|c_2\g-\sigma \overline{U}\overline{\Sigma} \v \|_2 +\gamma \c^T\Sigma\Sigma\c -\gamma c_2^2. \label{eq:ta18a0b4c0}
\end{eqnarray}
With appropriate scaling $\nu_s\rightarrow \frac{\nu_1}{\sqrt{n}}$ (where $\nu_1$ does not change as $n$ grows), (\ref{eq:ta18a0b4c0}) can be rewritten as \begin{eqnarray}
  \cL_2\lp \hat{\x},\gamma;c_2,\nu_1\rp
 & = &
-\lp -\frac{1}{2\sqrt{n}}\nu_1 \Sigma\h +\gamma\Sigma\Sigma\c\rp^T   \lp I +\gamma\Sigma\Sigma\rp^{-1}\lp -\frac{1}{2\sqrt{n}}\nu_1 \Sigma\h +\gamma\Sigma\Sigma\c\rp
\nonumber \\
   & & -\frac{1}{\sqrt{n}}\nu_1\h^T\Sigma\c+ \frac{1}{\sqrt{n}}\nu_1\|c_2\g-\sigma \overline{U}\overline{\Sigma} \v \|_2 +\gamma \c^T\Sigma\Sigma\c -\gamma c_2^2. \label{eq:ta18a0b4c1}
\end{eqnarray}
Setting $\s$ to be the column vector comprised of diagonal elements of $\Sigma$, i.e., setting
\begin{eqnarray}
 \s \triangleq \mbox{diag}\lp\Sigma\rp, \label{eq:ta18a0b4c2}
\end{eqnarray}
and keeping $c_2$, $\nu_1$, and $\gamma$ fixed, we then relying on concentrations find
\begin{align}
\lim_{n\rightarrow\infty}  \mE_{\g,\h,\v}\cL_2\lp \hat{\x},\gamma;c_2,\nu_1\rp
 & =
\lim_{n\rightarrow\infty}  \mE \Bigg ( \Big. -\lp -\frac{\nu_1}{2\sqrt{n}} \Sigma\h +\gamma\Sigma\Sigma\c\rp^T   \lp I +\gamma\Sigma\Sigma\rp^{-1}\lp -\frac{\nu_1}{2\sqrt{n}} \Sigma\h +\gamma\Sigma\Sigma\c\rp
\nonumber \\
   & \quad -\frac{1}{\sqrt{n}}\nu_1\h^T\Sigma\c+ \frac{1}{\sqrt{n}}\nu_1\|c_2\g-\sigma \overline{U}\overline{\Sigma} \v \|_2 +\gamma \c^T\Sigma\Sigma\c -\gamma c_2^2 \Big.\Bigg ) \nonumber \\
  & =  \lim_{n\rightarrow\infty} \Bigg ( \Big. -\frac{\nu_1^2}{4n}\sum_{i=1}^n \frac{\s_i^2}{1+\gamma\s_i^2}
  - \sum_{i=1}^n \frac{\gamma^2\s_i^4\c_i^2}{1+\gamma\s_i^2}
+\nu_1\sqrt{\alpha}\sqrt{c_2^2+\bar{\sigma}^2}   + \gamma \sum_{i=1}^n \s_i^2\c_i^2 -\gamma c_2^2  \Big.\Bigg ), \nonumber \\ \label{eq:ta18a0b4c3}
\end{align}
where
\begin{eqnarray}
\bar{\sigma}\triangleq \sigma \sqrt{\lim_{n\rightarrow\infty} \frac{\tr\lp\overline{\Sigma}\overline{\Sigma} \rp}{m}}. \label{eq:ta18a0b4c4}
\end{eqnarray}
As mentioned earlier, to make writing easier and neater, we adopt the convention of avoiding repeatedly using $\lim_{n\rightarrow\infty}$. It is assumed that all expressions involving $n$ (or $m$) are written within the $\lim_{n\rightarrow\infty}$ context. Also, as mentioned earlier, all such expressions are assumed to be bounded and the corresponding limiting values well defined and existing. Combining (\ref{eq:ta18a0b1})-(\ref{eq:ta18a0b2c1}), (\ref{eq:ta18a0b4c0}), and  (\ref{eq:ta18a0b4c3}), we arrive at the following optimization
\begin{eqnarray}
\lim_{n\rightarrow\infty}  \mE_{\g,\h,\v}f_{rd}(\g,\h,\v)
   & = & \lim_{n\rightarrow\infty} \min_{c_2\geq 0} \max_{\nu_1\geq 0,\gamma}  f_0(c_2,\nu_1,\gamma), \label{eq:ta18a0b4c5}
\end{eqnarray}
where
\begin{eqnarray}
f_0(c_2,\nu_1,\gamma) \triangleq  -\frac{\nu_1^2}{4n}\sum_{i=1}^n \frac{\s_i^2}{1+\gamma\s_i^2}
  - \sum_{i=1}^n \frac{\gamma^2\s_i^4\c_i^2}{1+\gamma\s_i^2}
+\nu_1\sqrt{\alpha}\sqrt{c_2^2+\bar{\sigma}^2} + \gamma \sum_{i=1}^n \s_i^2\c_i^2 -\gamma c_2^2.  \label{eq:ta18a0b4c6}
\end{eqnarray}
After taking the derivative with respect to $\nu_1$, we find
\begin{eqnarray}
\frac{df_0(c_2,\nu_1,\gamma)}{d\nu_1} =   -\frac{\nu_1}{2n}\sum_{i=1}^n \frac{\s_i^2}{1+\gamma\s_i^2}
 +\sqrt{\alpha}\sqrt{c_2^2+\bar{\sigma}^2}. \label{eq:ta18a0b4c7}
\end{eqnarray}
Setting the above derivative to zero gives the following optimal $\nu_1$
\begin{eqnarray}
\hat{\nu}_1 =
 \frac{2\sqrt{\alpha}\sqrt{c_2^2+\bar{\sigma}^2}}{\frac{1}{n}\sum_{i=1}^n \frac{\s_i^2}{1+\gamma\s_i^2}}. \label{eq:ta18a0b4c8}
\end{eqnarray}
Plugging this value for $\nu_1$ back in (\ref{eq:ta18a0b4c6}) gives
\begin{eqnarray}
\max_{\nu_1\geq 0} f_0(c_2,\nu_1,\gamma) = f_0(c_2,\hat{\nu}_1,\gamma) =
  - \sum_{i=1}^n \frac{\gamma^2\s_i^4\c_i^2}{1+\gamma\s_i^2}
+ \frac{\alpha\lp c_2^2+\bar{\sigma}^2\rp}{\frac{1}{n}\sum_{i=1}^n \frac{\s_i^2}{1+\gamma\s_i^2}} + \gamma \sum_{i=1}^n \s_i^2\c_i^2 -\gamma c_2^2.  \label{eq:ta18a0b4c9}
\end{eqnarray}
A combination of (\ref{eq:ta18a0b4c5}) and (\ref{eq:ta18a0b4c9}) gives
\begin{eqnarray}
\lim_{n\rightarrow\infty}  \mE_{\g,\h,\v}f_{rd}(\g,\h,\v)
   & = & \lim_{n\rightarrow\infty}  \min_{c_2\geq 0} \max_{\gamma}  f_0(c_2,\hat{\nu}_1,\gamma), \label{eq:ta18a0b4c10}
\end{eqnarray}
where $f_0(c_2,\hat{\nu}_1,\gamma)$ as in (\ref{eq:ta18a0b4c9}). Taking derivatives with respect $c_2$ and $\gamma$ and equalling them to zero, we further have
\begin{eqnarray}
\frac{d f_0(c_2,\hat{\nu}_1,\gamma)}{d c_2} =
  \frac{2\alpha c_2}{\frac{1}{n}\sum_{i=1}^n \frac{\s_i^2}{1+\gamma\s_i^2}}   -2\gamma c_2 =0.  \label{eq:ta18a0b4c11}
\end{eqnarray}
and
\begin{equation}
\frac{d f_0(c_2,\hat{\nu}_1,\gamma)}{d \gamma} =
  - \sum_{i=1}^n \frac{2\gamma\s_i^4\c_i^2}{1+\gamma\s_i^2}
  + \sum_{i=1}^n \frac{\gamma^2\s_i^6\c_i^2}{\lp 1+\gamma\s_i^2\rp^2}
 +\frac{\alpha\lp c_2^2+\bar{\sigma}^2\rp}{\lp \frac{1}{n}\sum_{i=1}^n \frac{\s_i^2}{1+\gamma\s_i^2}\rp^2}
 \lp \frac{1}{n}\sum_{i=1}^n \frac{\s_i^4}{\lp 1+\gamma\s_i^2\rp^2}\rp
 + \sum_{i=1}^n \s_i^2\c_i^2 - c_2^2=0. \label{eq:ta18a0b4c12}
\end{equation}
From (\ref{eq:ta18a0b4c11}) we have that the optimal $\hat{\gamma}$ satisfies the following equation
\begin{eqnarray}
  \frac{1}{n}\sum_{i=1}^n \frac{\hat{\gamma}\s_i^2}{1+\hat{\gamma}\s_i^2} =\alpha.  \label{eq:ta18a0b4c13}
\end{eqnarray}
From (\ref{eq:ta18a0b4c12}) we have that the optimal $\hat{c}_2$ and $\hat{\gamma}$ satisfy the following equation
\begin{eqnarray}
   \sum_{i=1}^n \frac{\s_i^2\c_i^2}{\lp 1+\hat{\gamma}\s_i^2\rp^2}
 +\frac{\alpha\lp \hat{c}_2^2+\bar{\sigma}^2\rp}{\lp \frac{1}{n}\sum_{i=1}^n \frac{\s_i^2}{1+\hat{\gamma}\s_i^2}\rp^2}
 \lp \frac{1}{n}\sum_{i=1}^n \frac{\s_i^4}{\lp 1+\hat{\gamma}\s_i^2\rp^2}\rp
   - \hat{c}_2^2=0.  \label{eq:ta18a0b4c14}
\end{eqnarray}
Setting
\begin{eqnarray}
a_2 =
\frac{  \frac{1}{n}\sum_{i=1}^n \frac{\s_i^4}{\lp 1+\hat{\gamma}\s_i^2\rp^2}}{\lp  \frac{1}{n}\sum_{i=1}^n \frac{\s_i^2}{1+\hat{\gamma}\s_i^2}\rp^2},  \label{eq:ta18a0b4c15}
\end{eqnarray}
we from (\ref{eq:ta18a0b4c14}) find
\begin{eqnarray}
\hat{c}_2^2= \frac{  \sum_{i=1}^n \frac{\s_i^2\c_i^2}{\lp 1+\hat{\gamma}\s_i^2\rp^2} +\alpha \bar{\sigma}^2a_2 }{1-\alpha a_2}.  \label{eq:ta18a0b4c16}
\end{eqnarray}
From (\ref{eq:ta18a0b4c9}) and (\ref{eq:ta18a0b4c10}), we then have
\begin{eqnarray}
\lim_{n\rightarrow\infty}  \mE_{\g,\h,\v}f_{rd}(\g,\h,\v)
   & = & \lim_{n\rightarrow\infty} \min_{c_2\geq 0} \max_{\gamma}  f_0(c_2,\hat{\nu}_1,\gamma) \nonumber \\
   & = & \lim_{n\rightarrow\infty} f_0(\hat{c}_2,\hat{\nu}_1,\hat{\gamma}) \nonumber \\
   & = &  \lim_{n\rightarrow\infty} \lp  - \sum_{i=1}^n \frac{\hat{\gamma}^2\s_i^4\c_i^2}{1+\hat{\gamma}\s_i^2}
+ \frac{\alpha\lp \hat{c}_2^2+\bar{\sigma}^2\rp}{\frac{1}{n}\sum_{i=1}^n \frac{\s_i^2}{1+\hat{\gamma}\s_i^2}} + \hat{\gamma} \sum_{i=1}^n \s_i^2\c_i^2 - \hat{\gamma} \hat{c}_2^2 \rp \nonumber \\
   & = &   \lim_{n\rightarrow\infty}  \lp \sum_{i=1}^n \frac{\hat{\gamma}\s_i^2\c_i^2}{1+\hat{\gamma}\s_i^2}
+ \frac{\alpha\lp \hat{c}_2^2+\bar{\sigma}^2\rp}{\frac{1}{n}\sum_{i=1}^n \frac{\s_i^2}{1+\hat{\gamma}\s_i^2}} - \hat{\gamma} \hat{c}_2^2 \rp \nonumber \\
   & = &   \lim_{n\rightarrow\infty}   \sum_{i=1}^n \frac{\hat{\gamma}\s_i^2\c_i^2}{1+\hat{\gamma}\s_i^2}
+ \hat{\gamma}\lp \hat{c}_2^2+\bar{\sigma}^2\rp  - \hat{\gamma} \hat{c}_2^2  \nonumber \\
   & = &   \lim_{n\rightarrow\infty}  \sum_{i=1}^n \frac{\hat{\gamma}\s_i^2\c_i^2}{1+\hat{\gamma}\s_i^2}
+ \hat{\gamma}\bar{\sigma}^2, \label{eq:ta18a10}
\end{eqnarray}
where $\hat{c}_2$ and $\hat{\gamma}$ are given through (\ref{eq:ta18a0b4c13}), (\ref{eq:ta18a0b4c15}), and (\ref{eq:ta18a0b4c16}). For the very same $\hat{c}_2$ and $\hat{\gamma}$, we can then from (\ref{eq:ta18a0b4c8})  obtain for the optimal $\nu_1$
\begin{eqnarray}
\hat{\nu}_1 =
 \frac{2\sqrt{\alpha}\sqrt{\hat{c}_2^2+\bar{\sigma}^2}}{\frac{1}{n}\sum_{i=1}^n \frac{\s_i^2}{1+\hat{\gamma}\s_i^2}}=
 \frac{2\hat{\gamma}\sqrt{\hat{c}_2^2+\bar{\sigma}^2}}{\sqrt{\alpha}}. \label{eq:ta18a0b4c17}
\end{eqnarray}

 We summarize the above discussion in the following lemma.
\begin{lemma}(Characterization of random dual) Assume the setup of Theorem \ref{thm:thm1} with $\xi_{rd}$ as in (\ref{eq:ta16}) and consider a large $n$ linear regime with $\alpha=\lim_{n\rightarrow\infty}\frac{m}{n}$ that does not change as $n$ grows. Assume that all the considered limiting quantities are well defined and bounded. Let $\c=V^T\bar{\beta}$, $\s=\mbox{diag}(\Sigma)$, and $\bar{\sigma}\triangleq \sigma \sqrt{\lim_{n\rightarrow\infty} \frac{\tr\lp\overline{\Sigma}\overline{\Sigma} \rp}{m}}$. Moreover, let $\hat{\gamma}$ be the unique solution of
\begin{eqnarray}
\lim_{n\rightarrow\infty} \frac{1}{n}\sum_{i=1}^n \frac{\hat{\gamma}\s_i^2}{1+\hat{\gamma}\s_i^2} =\alpha.  \label{eq:thmaaeq1}
\end{eqnarray}
Set
\begin{eqnarray}
a_2 =
\frac{\lim_{n\rightarrow\infty} \frac{1}{n}\sum_{i=1}^n \frac{\s_i^4}{\lp 1+\hat{\gamma}\s_i^2\rp^2}}{\lp \lim_{n\rightarrow\infty} \frac{1}{n}\sum_{i=1}^n \frac{\s_i^2}{1+\hat{\gamma}\s_i^2}\rp^2}
=\frac{\hat{\gamma}^2}{\alpha^2} \lim_{n\rightarrow\infty} \frac{1}{n}\sum_{i=1}^n \frac{\s_i^4}{\lp 1+\hat{\gamma}\s_i^2\rp^2}.  \label{eq:thmaaeq2}
\end{eqnarray}
One then has \vspace{-.0in}
\begin{eqnarray}
  \xi_{rd} & = & \lim_{n\rightarrow\infty} \mE_{\g,\h,\v} f_{rd}(\g,\h,\v)
 =
\lim_{n\rightarrow\infty}  \sum_{i=1}^n \frac{\hat{\gamma}\s_i^2\c_i^2}{1+\hat{\gamma}\s_i^2}
+ \hat{\gamma}\bar{\sigma}^2 \nonumber \\
&  &  \lim_{n\rightarrow\infty}  \hat{c}_2^2
 =
 \frac{  \lim_{n\rightarrow\infty} \sum_{i=1}^n \frac{\s_i^2\c_i^2}{\lp 1+\hat{\gamma}\s_i^2\rp^2} +\alpha \bar{\sigma}^2a_2 }{1-\alpha a_2} \nonumber \\
 & & \lim_{n\rightarrow\infty}  \hat{\nu}_1
 = \frac{2\hat{\gamma}\sqrt{\hat{c}_2^2+\bar{\sigma}^2}}{\sqrt{\alpha}},
\label{eq:thmaaeq3}
\end{eqnarray}
and for any fixed $\epsilon>0$
\begin{eqnarray}
 \lim_{n\rightarrow\infty}\mP_{\g,\h,\v} \lp  (1-\epsilon)\xi_{rd} \leq  f_{rd}(\g,\h,\v)\leq (1+\epsilon)\xi_{rd}\rp
& \longrightarrow & 1 \nonumber \\
 \lim_{n\rightarrow\infty}\mP_{\g,\h,\v} \lp  (1-\epsilon)\hat{c}_2^2 \leq  c_2^2\leq (1+\epsilon)\hat{c}_2^2\rp
& \longrightarrow & 1 \nonumber \\
 \lim_{n\rightarrow\infty}\mP_{\g,\h,\v} \lp  (1-\epsilon)\hat{\nu}_1 \leq  \nu_1 \leq (1+\epsilon)\hat{\nu}_1\rp
& \longrightarrow & 1.\label{eq:thmaaeq4}
\end{eqnarray}
\label{lemma:lemma2}
\end{lemma}\vspace{-.17in}
\begin{proof}
Follows from the above discussion and trivial concentrations of $f_{rd}(\g,\h,\v)$, $c_2$, and $\nu_1$.
\end{proof}

The above discussion allows us then to also formulate the following theorem.

\begin{theorem}(Characterization of GLS estimator) Assume the setup of Lemma \ref{lemma:lemma2} and let $\hat{\beta}=\beta_{gls}$ where $\beta_{gls}$ is as in
(\ref{eq:model011}). Also, let $\hat{\gamma}$, $a_2$, and $\hat{c}_2$ be as in (\ref{eq:thmaaeq1}), (\ref{eq:thmaaeq2}), and (\ref{eq:thmaaeq3}), respectively/. Finally let the objective value of the GLS estimating optimization, $\xi_{gls}$, be as in (\ref{eq:randlincons1}) or (\ref{eq:randlincons1a0}) and let the GLS estimator prediction risk, $R(\bar{\beta},\hat{\beta})=R(\bar{\beta},\beta_{gls})$, be as in (\ref{eq:model020}). Then
\begin{eqnarray}\label{eq:thm2a11a0}
 \lim_{n\rightarrow\infty} \mE_{Z,\v}\xi_{gls}
 &= &\xi_{rp}=\xi_{rd}
 =
\lim_{n\rightarrow\infty}  \sum_{i=1}^n \frac{\hat{\gamma}\s_i^2\c_i^2}{1+\hat{\gamma}\s_i^2}
+ \hat{\gamma}\bar{\sigma}^2 \nonumber \\
 \lim_{n\rightarrow\infty} \mE_{Z,\v} R(\bar{\beta},\beta_{gls})
& = &
\lim_{n\rightarrow\infty} \mE_{Z,\v} \lp \bar{\beta}- \beta_{gls}\rp^T V\Sigma\Sigma V^T \lp \bar{\beta}- \beta_{gls}\rp
 =
 \frac{  \lim_{n\rightarrow\infty} \sum_{i=1}^n \frac{\s_i^2\c_i^2}{\lp 1+\hat{\gamma}\s_i^2\rp^2} +\alpha \bar{\sigma}^2a_2 }{1-\alpha a_2}. \nonumber \\
\end{eqnarray}
Moreover,
\begin{eqnarray}
 \lim_{n\rightarrow\infty}\mP_{Z,\v} \lp  (1-\epsilon)\mE_{Z,\v}\xi_{gls} \leq  \xi_{gls} \leq (1+\epsilon)\mE_{Z,\v}\xi_{gls}\rp
& \longrightarrow & 1 \nonumber \\
 \lim_{n\rightarrow\infty}\mP_{Z,\v} \lp  (1-\epsilon)\mE_{Z,\v} R(\bar{\beta},\beta_{gls})  \leq  R(\bar{\beta},\beta_{gls}) \leq (1+\epsilon)\mE_{Z,\v}R(\bar{\beta},\beta_{gls})\rp
& \longrightarrow & 1.\label{eq:thm2ta17}
\end{eqnarray}
 \label{thm:thm2}
\end{theorem}
\begin{proof}
For the first part of (\ref{eq:thm2a11a0}), we observe that
 \begin{equation}\label{eq:proofthm2eq1}
 \lim_{n\rightarrow\infty} \mE_{Z,\v}\xi_{gls} =\xi_{rp}\geq \xi_{rd}
 =
\lim_{n\rightarrow\infty}  \sum_{i=1}^n \frac{\hat{\gamma}\s_i^2\c_i^2}{1+\hat{\gamma}\s_i^2}
+ \hat{\gamma}\bar{\sigma}^2,
\end{equation}
holds due to Theorem \ref{thm:thm1} (in particular, due to the inequality in (\ref{eq:ta16a0})) and (\ref{eq:thmaaeq3}). The underlying convexity and the fact that strong random duality is in place ensure that  all the reversal arguments from \cite{StojnicRegRndDlt10,StojnicGorEx10} apply and that the reversal inequality in (\ref{eq:proofthm2eq1}) holds as well.
For the second part of (\ref{eq:thm2a11a0}), we first observe
that (\ref{eq:model020}), (\ref{eq:ta18a0b0}),  and (\ref{eq:ta18a0b0c0}) imply
 \begin{equation}\label{eq:proofthm2eq1a0}
  R(\bar{\beta},\beta_{gls})
  = \lp \bar{\beta}- \beta_{gls}\rp^T V\Sigma\Sigma V^T \lp \bar{\beta}- \beta_{gls}\rp
  =c_2^2.
 \end{equation}
One then immediately has that a combination of (\ref{eq:thmaaeq3}) and (\ref{eq:proofthm2eq1a0}) implies the second part of (\ref{eq:thm2a11a0}).

We further observe that (\ref{eq:thm2ta17}) is basically rewritten (\ref{eq:thmaaeq4}) and holds due to the concentration of $\xi_{gls}$ and $R(\bar{\beta},\beta_{gls})$. In other words,
\begin{eqnarray}
 \lim_{n\rightarrow\infty}\mP_{\g,\h,\v} \lp  (1-\epsilon)\xi_{rd} \leq  f_{rd}(\g,\h,\v)\leq (1+\epsilon)\xi_{rd}\rp
&  \leq &
 \lim_{n\rightarrow\infty}\mP_{\g,\h,\v} \lp f_{rd}(\g,\h,\v)\geq (1-\epsilon)\xi_{rd}\rp \nonumber \\
& \leq & \lim_{n\rightarrow\infty}\mP_{Z,\v} \lp f_{rp}(Z,\v)\geq (1-\epsilon)\xi_{rd}\rp  \nonumber \\
& \leq & \lim_{n\rightarrow\infty}\mP_{Z,\v} \lp \xi_{gls} \geq (1-\epsilon) \mE_{Z,\v}\xi_{gls}\rp,\label{eq:proofthm2eq2}
\end{eqnarray}
where the first inequality is implied by the probability axioms, the second by (\ref{eq:ta17}), and the third by a combination of (\ref{eq:ta11a0})
and (\ref{eq:thm2a11a0}). Combining further (\ref{eq:thmaaeq4}) and (\ref{eq:proofthm2eq2}) gives
\begin{eqnarray}
\lim_{n\rightarrow\infty}\mP_{Z,\v} \lp \xi_{gls} \geq (1-\epsilon) \mE_{Z,\v}\xi_{gls}\rp \longrightarrow 1.\label{eq:proofthm2eq3}
\end{eqnarray}
The reversal arguments of \cite{StojnicRegRndDlt10,StojnicGorEx10} also ensure that one has
\begin{eqnarray}
\lim_{n\rightarrow\infty}\mP_{Z,\v} \lp \xi_{gls} \leq (1+\epsilon) \mE_{Z,\v}\xi_{gls}\rp \longrightarrow 1.\label{eq:proofthm2eq4}
\end{eqnarray}
A combination of (\ref{eq:proofthm2eq3}) and (\ref{eq:proofthm2eq4}) finally confirms the concentration in the first part of (\ref{eq:thm2ta17}). Keeping in mind (\ref{eq:proofthm2eq1a0}), the concentration stated ni the second part of  (\ref{eq:thm2ta17}) is literally the same as the one stated in the second part of (\ref{eq:thmaaeq4}).
\end{proof}

\vspace{.1in}
\noindent \underline{ 4) \textbf{\emph{Double checking strong random duality:}}} As mentioned above, due to the underlying convexity,  the strong random duality and the reversal arguments of \cite{StojnicGorEx10,StojnicRegRndDlt10} are in place as well, and the above bounding results are tight and consequently the exact characterizations.

It is not that difficult to see that when $\bar{\sigma}=\sigma$, the results of the above theorem match the ones obtained in \cite{HMRT22,Dicker16,DobWag18}.

\subsection{Ridge estimator}
\label{sec:analridge}

The mathematical machinery used above for the analysis of the GLS estimator is very generic. In this section we show how it extends to cover the ridge estimator. To ensure the easiness of following we will try to parallel the presentation of the previous sections as closely as possible. However, we often proceed in a much faster fashion by avoiding repeating same (or fairly similar) arguments and focusing on the key differences.

Recalling on (\ref{eq:model014}), we introduce the objective value of the underlying ridge regression as
 \begin{eqnarray}
\xi_{rr}(\lambda) \triangleq \min_{\beta} \lp \lambda\|\beta\|_2^2 +\frac{1}{m}\|\y- X\beta\|_2^2\rp. \label{eq:ridgerandlincons1}
\end{eqnarray}
One should note that the above objective is a function of the ridge parameter $\lambda$ (throughout the ensuing analysis we take $\lambda$ as a given nonnegative real number that does not depend on $n$ or any other quantity and keep it fixed). It is also not that difficult to see that the above optimization can also be rewritten as
 \begin{eqnarray}
\xi_{rr}(\lambda) = \min_{\beta,\z,c_{3,s}} & &   \lambda\|\beta\|_2^2 + \frac{c_{3,s}^2}{m} \nonumber \\
\mbox{subject to} & & X\beta=\y+\z \nonumber \\
& & \|\z\|_2=c_{3,s}. \label{eq:ridgerandlincons1a0}
\end{eqnarray}
Recalling on (\ref{eq:model010}) and (\ref{eq:model010a0}), we can also write a statistical equivalent to (\ref{eq:ridgerandlincons1a0})
\begin{eqnarray}
\xi_{rr}(\lambda) \triangleq \min_{\beta,\|\z\|_2=c_{3,s}} & &  \lambda\|\beta\|_2^2 + \frac{c_{3,s}^2}{m} \nonumber \\
\mbox{subject to} & & Z\Sigma V^T\beta=Z\Sigma V^T\bar{\beta}+\sigma \overline{U}\overline{\Sigma}\v+\z. \label{eq:ridgerandlincons1a1}
\end{eqnarray}
Following the main RDT algebraic fundamental ideas and writing the Lagrangian we further have
\begin{eqnarray}
\xi_{rr}(\lambda) = \min_{\beta,\|\z\|_2=c_{3,s}} \max_{\nu} \lp \lambda\|\beta\|_2^2 +\nu^T \lp Z\Sigma V^T\beta-Z\Sigma V^T\bar{\beta}-\sigma \overline{U}\overline{\Sigma}\v -\z \rp  + \frac{c_{3,s}^2}{m}\rp. \label{eq:ridgerandlincons2}
\end{eqnarray}
To analytically characterize the above problem we again proceed by employing the remaining portions of the RDT machinery.

\subsubsection{Handling Ridge estimator via RDT}
\label{sec:randlinconsrdtridge}

As earlier, we systematically discuss, one-by-one, four main RDT principles and start with the algebraic part.

\vspace{.1in}

\noindent \underline{1) \textbf{\emph{Algebraic characterization:}}}  The following lemma (basically a ridge analogue to Lemma \ref{lemma:lemma1}) summarizes the above algebraic discussions and provides a convenient representation of the objective $\xi_{rr}(\lambda)$.

\begin{lemma}(Algebraic optimization representation) assume the setup of Lemma \ref{lemma:lemma1} and let $\lambda\geq 0$ be a fixed real scalar independent of $n$ or any other quantity. let $\xi_{rr}(\lambda)$ be as in (\ref{eq:ridgerandlincons1}) or (\ref{eq:ridgerandlincons1a0}). Set ${\mathcal D}=\{\beta,\z,c_{3,s}|\beta\in\mR^n,\z\in\mR^m,c_{3,s}\in\mR,\|\z\|_2=c_{3,s}\}$ and
\begin{eqnarray}\label{eq:ridgeta11}
f_{rp,r}(Z,\v) & \triangleq & \min_{{\mathcal D}} \max_{\nu} \lp \lambda\|\beta\|_2^2 +\nu^T \lp Z\Sigma V^T\beta-Z\Sigma V^T\bar{\beta}-\sigma \overline{U}\overline{\Sigma}\v -\z\rp +\frac{c_{3,s}^2}{m} \rp
 \hspace{.2in} (\bl{\textbf{random primal}})
\nonumber \\
\xi_{rp,r} & \triangleq & \lim_{n\rightarrow\infty } \mE_{Z,\v} f_{rp,r}(Z,\v).   \end{eqnarray}
Then
\begin{equation}\label{eq:ridgeta11a0}
\xi_{rr}(\lambda)=f_{rp,r}(Z,\v) \quad \mbox{and} \quad \lim_{n\rightarrow\infty} \mE_{Z,\v}\xi_{rr}(\lambda) =\xi_{rp,r}.
\end{equation}
\label{lemma:ridgelemma1}
\end{lemma}
\begin{proof}
 Follows automatically via the Lagrangian from (\ref{eq:ridgerandlincons2}).
\end{proof}

Clearly, assuming that $V$, $\Sigma$, $\overline{U}$, $\overline{\Sigma}$, $\sigma$, and $\beta$ are fixed and given, the above lemma holds for any $Z$ and $\v$. The RDT proceeds by imposing a statistics  on $Z$ and $\v$.

\vspace{.1in}
\noindent \underline{2) \textbf{\emph{Determining the random dual:}}} As earlier, we utilize the concentration of measure, which here implies that for any fixed $\epsilon >0$,  one can write (see, e.g. \cite{StojnicCSetam09,StojnicRegRndDlt10,StojnicICASSP10var})
\begin{equation}
\lim_{n\rightarrow\infty}\mP_{Z,\v}\left (\frac{|f_{rp,r}(Z,\v)-\mE_{Z,\v}(f_{rp,r}(Z,\v))|}{\mE_{Z,\v}(f_{rp,r}(Z,\v))}>\epsilon\right )\longrightarrow 0.\label{eq:ridgeta15}
\end{equation}
We then have the following random dual ridge analogue to Theorem \ref{thm:thm1}).

\begin{theorem}(Objective characterization via random dual -- Ridge estimator) Assume the setup of Lemma \ref{lemma:ridgelemma1} and Theorem \ref{thm:thm1}. Set
\vspace{-.0in}
\begin{eqnarray}
   f_{rd,r}(\g,\h,\v) & \triangleq &
\hspace{-.0in} \min_{{\mathcal D}}\max_{\nu} \lp \lambda \|\beta\|_2^2+c_2\nu^T\g+\|\nu\|_2\h^T\Sigma V^T\lp \beta -\bar{\beta}\rp -\sigma \nu^T\overline{U}\overline{\Sigma} \v -\nu^T\z  +\frac{c_{3,s}^2}{m} \rp \nonumber \\
& &\hspace{4.3in}    (\bl{\textbf{random dual}})
  \nonumber \\
 \xi_{rd,r} & \triangleq &  \lim_{n\rightarrow\infty} \mE_{\g,\h,\v} f_{rd,r}(\g,\h,\v)  .\label{eq:ridgeta16}
\vspace{-.0in}
\end{eqnarray}
One then has \vspace{-.0in}
\begin{eqnarray}
  \xi_{rd,r} & \triangleq & \lim_{n\rightarrow\infty} \mE_{\g,\h,\v} f_{rd,r}(\g,\h,\v)
  \leq
  \lim_{n\rightarrow\infty} \mE_{Z,\v} f_{rp,r}(Z,\v)  \triangleq  \xi_{rp,r}. \label{eq:ridgeta16a0}
\vspace{-.0in}\end{eqnarray}
and
\begin{eqnarray}
 \lim_{n\rightarrow\infty}\mP_{\g,\h,\v} \lp f_{rd,r}(\g,\h,\v)\geq (1-\epsilon)\xi_{rd,r}\rp
 \leq  \lim_{n\rightarrow\infty}\mP_{Z,\v} \lp f_{rp,r}(Z,\v)\geq (1-\epsilon)\xi_{rd,r}\rp.\label{eq:ridgeta17}
\end{eqnarray}
\label{thm:ridgethm1}
\end{theorem}\vspace{-.17in}
\begin{proof}
As the proof of Theorem \ref{thm:thm1}, it follows immediately as a direct application of the Gordon's probabilistic comparison theorem (see, e.g., Theorem B in \cite{Gordon88} and also Theorem 1, Corollary 1, and Section 2.7.2 in \cite{Stojnicgscomp16}).
\end{proof}

 \vspace{.1in}
\noindent \underline{3) \textbf{\emph{Handling the random dual:}}} We start by solving the inner maximization over $\nu$ and obtain analogously to (\ref{eq:ta18a0})
\begin{eqnarray}
  f_{rd,r}(\g,\h,\v)
  & \triangleq &
\min_{\beta,\|\z\|_2=\c_{3,s}}\max_{\nu} \lp \|\beta\|_2^2+c_2\nu^T\g+\|\nu\|_2\h^T\Sigma V^T\lp \beta -\bar{\beta}\rp -\sigma \nu^T\overline{U}\overline{\Sigma} \v -\nu^T\z +\frac{c_{3,s}^2}{m}\rp \nonumber \\
   & = &
\min_{\x,\|\z\|_2=c_{3,s},c_2\geq 0}\max_{\nu_s\geq 0} \lp \lambda \|\x\|_2^2+\nu_s\h^T\Sigma \lp \x -\c\rp +\nu_s\|c_2\g-\sigma \overline{U}\overline{\Sigma} \v -\z\|_2 +\frac{c_{3,s}^2}{m}\rp, \label{eq:ridgeta18a0}
\end{eqnarray}
where $\x,\c,\nu_s,$ and $c_2$ are as in (\ref{eq:ta18a0b0}) and (\ref{eq:ta18a0b0c0}). After setting
\begin{eqnarray}
  \cL_r\lp\x,\z,\c_{3,s},c_2,\nu_s;\g,\h,\v\rp \triangleq  \lambda \|\x\|_2^2+ \nu_s\h^T\Sigma \lp \x -\c\rp + \nu_s\|c_2\g-\sigma \overline{U}\overline{\Sigma} \v -\z\|_2 +\frac{c_{3,s}^2}{m}, \label{eq:ridgeta18a0b1}
\end{eqnarray}
we rewrite (\ref{eq:ridgeta18a0}) as
\begin{eqnarray}
  f_{rd,r}(\g,\h,\v)
  =
\min_{\x,\|\z\|_2=c_{3,s},c_2\geq 0}\max_{\nu_s\geq 0}  \cL_r\lp\x,\z,c_{3,s},c_2,\nu_s;\g,\h,\v\rp. \label{eq:ridgeta18a0b2}
\end{eqnarray}
 Keeping $c_{3,s}$, $c_2$, and $\nu_s$ fixed, we, analogously to (\ref{eq:ta18a0b2c0}), consider the following minimization  over $\x$ and $\z$
\begin{eqnarray}
  \cL_{1,r}\lp c_{3,s},c_2,\nu_s\rp \triangleq  \min_{\x,\z} & & \lambda \|\x\|_2^2+ \nu_s\h^T\Sigma \lp \x -\c\rp + \nu_s\|c_2\g-\sigma \overline{U}\overline{\Sigma} \v -\z\|_2 +\frac{c_{3,s}^2}{m} \nonumber \\
  \mbox{subject to} & &   \|\Sigma\lp\x-\c\rp\|_2=c_2, \|\z\|_2=c_{3,s}. \label{eq:ridgeta18a0b2c0}
\end{eqnarray}
To lighten the notation, we again keep the functional dependence on randomness ($\g,\h,\v$) implicit. We then, analogously to (\ref{eq:ta18a0b2c1}), further have
\begin{eqnarray}
  \cL_{1,r}\lp c_{3,s},c_2,\nu_s\rp
   & = &  \max_{\gamma}   \min_{\x,\|\z\|_2=c_{3,s}}   \cL_{2,r}\lp \x,\z,\gamma;c_{3,s},c_2,\nu_s\rp, \label{eq:ridgeta18a0b2c1}
\end{eqnarray}
with
\begin{eqnarray}
  \cL_{2,r}\lp \x,\z,\gamma;c_{3,s},c_2,\nu_s\rp
  & \triangleq &
 \lambda \|\x\|_2^2+ \nu_s\h^T\Sigma \lp \x -\c\rp + \nu_s\|c_2\g-\sigma \overline{U}\overline{\Sigma} \v -\z\|_2 +\frac{c_{3,s}^2}{m}\nonumber \\
  & & +\gamma \|\Sigma\lp\x-\c\rp\|_2^2 -\gamma c_2^2 . \label{eq:ridgeta18a0b2c2}
\end{eqnarray}
Analogously to (\ref{eq:ta18a0b4c0}), we then have for optimal $\x$
\begin{eqnarray}
 \hat{\x}= \lp \lambda I +\gamma\Sigma\Sigma\rp^{-1}\lp -\frac{1}{2}\nu_s \Sigma\h +\gamma\Sigma\Sigma\c\rp, \label{eq:ridgta18a0b4}
\end{eqnarray}
which (together with appropriate scaling $\nu_s\rightarrow \frac{\nu_1}{\sqrt{n}}$ (where $\nu_1$ does not change as $n$ grows)) gives the following analogue to (\ref{eq:ta18a0b4c1})
\begin{eqnarray}
  \cL_{2,r}\lp \hat{\x},\z,\gamma;c_{3,s},c_2,\nu_1\rp
 & = &
-\lp -\frac{1}{2\sqrt{n}}\nu_1 \Sigma\h +\gamma\Sigma\Sigma\c\rp^T   \lp \lambda I +\gamma\Sigma\Sigma\rp^{-1}\lp -\frac{1}{2\sqrt{n}}\nu_1 \Sigma\h +\gamma\Sigma\Sigma\c\rp
\nonumber \\
   & & -\frac{1}{\sqrt{n}}\nu_1\h^T\Sigma\c+ \frac{1}{\sqrt{n}}\nu_1\|c_2\g-\sigma \overline{U}\overline{\Sigma} \v -\z \|_2 +\frac{c_{3,s}^2}{m}+\gamma \c^T\Sigma\Sigma\c -\gamma c_2^2. \nonumber \\ \label{eq:ridgeta18a0b4c1}
\end{eqnarray}
One then also easily obtains for optimal $\z$
\begin{eqnarray}
     \hat{\z}=c_{3,s}\frac{c_2\g-\sigma \overline{U}\overline{\Sigma} \v}{\|c_2\g-\sigma \overline{U}\overline{\Sigma} \v\|_2}, \label{eq:ridgeta18a0b4c1d0}
\end{eqnarray}
which together with (\ref{eq:ridgeta18a0b4c1}) gives
\begin{eqnarray}
  \cL_{2,r}\lp \hat{\x},\hat{\z},\gamma;c_{3,s},c_2,\nu_1\rp
 & = &
-\lp -\frac{1}{2\sqrt{n}}\nu_1 \Sigma\h +\gamma\Sigma\Sigma\c\rp^T   \lp \lambda I +\gamma\Sigma\Sigma\rp^{-1}\lp -\frac{1}{2\sqrt{n}}\nu_1 \Sigma\h +\gamma\Sigma\Sigma\c\rp
\nonumber \\
   & & -\frac{1}{\sqrt{n}}\nu_1\h^T\Sigma\c+ \frac{1}{\sqrt{n}}\nu_1|\|c_2\g-\sigma \overline{U}\overline{\Sigma} \v\|_2 -c_{3,s}| +\frac{c_{3,s}^2}{m}+\gamma \c^T\Sigma\Sigma\c -\gamma c_2^2. \nonumber \\ \label{eq:ridgeta18a0b4c1d1}
\end{eqnarray}
Utilizing appropriate scaling $c_3\rightarrow\frac{c_{3,s}}{\sqrt{m}}$ and keeping $c_3$, $c_2$, $\nu_1$, and $\gamma$ fixed, we then,  relying on concentrations, find analogously to (\ref{eq:ta18a0b4c3})
\begin{align}
\lim_{n\rightarrow\infty} \hspace{-.06in} \mE_{\g,\h,\v}\cL_{2,r}\lp \hat{\x},\hat{\z},\gamma;c_3,c_2,\nu_1\rp
 & = \hspace{-.05in}
\lim_{n\rightarrow\infty} \hspace{-.06in} \mE \Bigg ( \Big. -\lp -\frac{\nu_1}{2\sqrt{n}} \Sigma\h +\gamma\Sigma\Sigma\c\rp^T   \lp \lambda I +\gamma\Sigma\Sigma\rp^{-1}\lp -\frac{\nu_1}{2\sqrt{n}} \Sigma\h +\gamma\Sigma\Sigma\c\rp
\nonumber \\
   & \hspace{-.05in}\quad -\frac{1}{\sqrt{n}}\nu_1\h^T\Sigma\c+ \frac{1}{\sqrt{n}}\nu_1 | \|c_2\g-\sigma \overline{U}\overline{\Sigma} \v \|_2 -c_3 m| +c_3^2 +\gamma \c^T\Sigma\Sigma\c -\gamma c_2^2 \Big.\Bigg ) \nonumber \\
  & =  \lim_{n\rightarrow\infty} \Bigg ( \Big. -\frac{\nu_1^2}{4n}\sum_{i=1}^n \frac{\s_i^2}{\lambda+\gamma\s_i^2}
  - \sum_{i=1}^n \frac{\gamma^2\s_i^4\c_i^2}{\lambda+\gamma\s_i^2}
+\nu_1\sqrt{\alpha}\left | \sqrt{c_2^2+\bar{\sigma}^2}-c_3\right | +c_3^2  \nonumber \\
&  \quad + \gamma \sum_{i=1}^n \s_i^2\c_i^2 -\gamma c_2^2  \Big.\Bigg ), \nonumber \\ \label{eq:ridgeta18a0b4c3}
\end{align}
with $\bar{\sigma}$ as in(\ref{eq:ta18a0b4c4}). A combination of (\ref{eq:ridgeta18a0b1})-(\ref{eq:ridgeta18a0b2c1}), (\ref{eq:ridgeta18a0b4c1d1}), and  (\ref{eq:ridgeta18a0b4c3}) allows us to arrive at the following optimization
\begin{eqnarray}
\lim_{n\rightarrow\infty}  \mE_{\g,\h,\v}f_{rd,r}(\g,\h,\v)
   & = & \lim_{n\rightarrow\infty} \min_{c_3,c_2\geq 0} \max_{\nu_1\geq 0,\gamma}  f_{0,r}(c_2,\nu_1,\gamma), \label{eq:ridgeta18a0b4c5}
\end{eqnarray}
where
\begin{eqnarray}
f_{0,r}(c_3,c_2,\nu_1,\gamma) \triangleq  -\frac{\nu_1^2}{4n}\sum_{i=1}^n \frac{\s_i^2}{\lambda+\gamma\s_i^2}
  - \sum_{i=1}^n \frac{\gamma^2\s_i^4\c_i^2}{\lambda+\gamma\s_i^2}
+\nu_1\sqrt{\alpha}\left | \sqrt{c_2^2+\bar{\sigma}^2} -c_3\right | +c_3^2  + \gamma \sum_{i=1}^n \s_i^2\c_i^2 -\gamma c_2^2.  \label{eq:ridgeta18a0b4c6}
\end{eqnarray}
After taking the derivative with respect to $\nu_1$ we, analogously to (\ref{eq:ta18a0b4c8}), find the following optimal $\nu_1$
\begin{eqnarray}
\hat{\nu}_1 =
 \frac{2\sqrt{\alpha}\left | \sqrt{c_2^2+\bar{\sigma}^2} -c_3\right |}{\frac{1}{n}\sum_{i=1}^n \frac{\s_i^2}{\lambda+\gamma\s_i^2}}. \label{eq:ridgeta18a0b4c8}
\end{eqnarray}
which then,  together with (\ref{eq:ridgeta18a0b4c6}), gives
\begin{equation}
\max_{\nu_1\geq 0} f_{0,r}(c_3,c_2,\nu_1,\gamma) = f_{0,r}(c_3,c_2,\hat{\nu}_1,\gamma) =
  - \sum_{i=1}^n \frac{\gamma^2\s_i^4\c_i^2}{\lambda+\gamma\s_i^2}
+ \frac{\alpha\lp \sqrt{c_2^2+\bar{\sigma}^2} -c_3\rp^2}{\frac{1}{n}\sum_{i=1}^n \frac{\s_i^2}{\lambda+\gamma\s_i^2}} +c_3^2 + \gamma \sum_{i=1}^n \s_i^2\c_i^2 -\gamma c_2^2.  \label{eq:ridgeta18a0b4c9}
\end{equation}
Combining (\ref{eq:ridgeta18a0b4c5}) and (\ref{eq:ridgeta18a0b4c9}) we obtain
\begin{eqnarray}
\lim_{n\rightarrow\infty}  \mE_{\g,\h,\v}f_{rd,r}(\g,\h,\v)
   & = & \lim_{n\rightarrow\infty}  \min_{c_3,c_2\geq 0} \max_{\gamma}  f_{0,r}(c_3,c_2,\hat{\nu}_1,\gamma), \label{eq:ridgeta18a0b4c10}
\end{eqnarray}
with $f_{0,r}(c_3,c_2,\hat{\nu}_1,\gamma)$ (concave in $\gamma$ and convex in $c_3$) as in (\ref{eq:ridgeta18a0b4c9}). After taking the derivative with respect to $c_3$, we find the following optimal $c_3$
\begin{eqnarray}
\hat{c}_3=\frac{\sqrt{c_2^2+\bar{\sigma}^2}}{1+\frac{1}{\alpha n}\sum_{i=1}^n \frac{\s_i^2}{\lambda+\gamma\s_i^2}}, \label{eq:ridgeta18a0b4c10d0}
\end{eqnarray}
which, together with (\ref{eq:ridgeta18a0b4c9}), gives
\begin{eqnarray}
\max_{c_3\geq 0} f_{0,r}(c_3,c_2,\hat{\nu}_1,\gamma) = f_{0,r}(\hat{c}_3,c_2,\hat{\nu}_1,\gamma) =
  - \sum_{i=1}^n \frac{\gamma^2\s_i^4\c_i^2}{\lambda+\gamma\s_i^2}
+ \frac{c_2^2+\bar{\sigma}^2}{1+\frac{1}{\alpha n}\sum_{i=1}^n \frac{\s_i^2}{\lambda+\gamma\s_i^2}}   + \gamma \sum_{i=1}^n \s_i^2\c_i^2 -\gamma c_2^2.  \label{eq:ridgeta18a0b4c10d1}
\end{eqnarray}
Taking derivatives with respect $c_2$ and $\gamma$ and equalling them to zero, we further have
\begin{eqnarray}
\frac{d f_{0,r}(\hat{c}_3,c_2,\hat{\nu}_1,\gamma)}{d c_2} =
  \frac{2 c_2}{1+\frac{1}{\alpha n}\sum_{i=1}^n \frac{\s_i^2}{\lambda+\gamma\s_i^2}}   -2\gamma c_2 =0.  \label{eq:ridgeta18a0b4c11}
\end{eqnarray}
and
\begin{eqnarray}
\frac{d f_{0,r}(\hat{c}_3,c_2,\hat{\nu}_1,\gamma)}{d \gamma}
& = &
  - \sum_{i=1}^n \frac{2\gamma\s_i^4\c_i^2}{\lambda+\gamma\s_i^2}
  + \sum_{i=1}^n \frac{\gamma^2\s_i^6\c_i^2}{\lp \lambda +\gamma\s_i^2\rp^2}
 +\frac{\lp c_2^2+\bar{\sigma}^2\rp}{\lp 1+ \frac{1}{\alpha n}\sum_{i=1}^n \frac{\s_i^2}{\lambda+\gamma\s_i^2}\rp^2}
 \lp \frac{1}{\alpha n}\sum_{i=1}^n \frac{\s_i^4}{\lp \lambda  +\gamma\s_i^2\rp^2}\rp
  \nonumber \\
 & &
 + \sum_{i=1}^n \s_i^2\c_i^2 - c_2^2=0. \label{eq:ridgeta18a0b4c12}
\end{eqnarray}
From (\ref{eq:ridgeta18a0b4c11}) we first have that the optimal $\hat{\gamma}$ satisfies the following equation
\begin{eqnarray}
  \frac{1}{n}\sum_{i=1}^n \frac{\hat{\gamma}\s_i^2}{\lambda+\hat{\gamma}\s_i^2} =\alpha(1-\hat{\gamma}).  \label{eq:ridgeta18a0b4c13}
\end{eqnarray}
For such a $\hat{\gamma}$ we then from (\ref{eq:ridgeta18a0b4c12}) have that the optimal $\hat{c}_2$ satisfies the following equation
\begin{eqnarray}
   \sum_{i=1}^n \frac{\lambda^2 \s_i^2\c_i^2}{\lp \lambda +\hat{\gamma}\s_i^2\rp^2}
 +\frac{\lp \hat{c}_2^2+\bar{\sigma}^2\rp}{\lp 1 + \frac{1}{\alpha n}\sum_{i=1}^n \frac{\s_i^2}{\lambda+\hat{\gamma}\s_i^2}\rp^2}
 \lp \frac{1}{\alpha n}\sum_{i=1}^n \frac{\s_i^4}{\lp \lambda +\hat{\gamma}\s_i^2\rp^2}\rp
   - \hat{c}_2^2=0.  \label{eq:ridgeta18a0b4c14}
\end{eqnarray}
Setting
\begin{eqnarray}
a_{2,r} =
\frac{  \frac{1}{n}\sum_{i=1}^n \frac{\s_i^4}{\lp \lambda+\hat{\gamma}\s_i^2\rp^2}}{\lp \alpha+ \frac{1}{  n}\sum_{i=1}^n \frac{\s_i^2}{\lambda +\hat{\gamma}\s_i^2}\rp^2},  \label{eq:ridgeta18a0b4c15}
\end{eqnarray}
we from (\ref{eq:ridgeta18a0b4c14}) find
\begin{eqnarray}
\hat{c}_2^2= \frac{  \sum_{i=1}^n \frac{\lambda^2 \s_i^2\c_i^2}{\lp \lambda+\hat{\gamma}\s_i^2\rp^2} + \alpha \bar{\sigma}^2a_{2,r} }{1- \alpha a_{2,r}}.  \label{eq:ridgeta18a0b4c16}
\end{eqnarray}
From (\ref{eq:ridgeta18a0b4c9}) and (\ref{eq:ridgeta18a0b4c10}), we then have
\begin{eqnarray}
\lim_{n\rightarrow\infty}  \mE_{\g,\h,\v}f_{rd,r}(\g,\h,\v)
   & = & \lim_{n\rightarrow\infty} \min_{c_2\geq 0} \max_{\gamma}  f_0(\hat{c}_3,c_2,\hat{\nu}_1,\gamma) \nonumber \\
   & = & \lim_{n\rightarrow\infty} f_0(\hat{c}_3,\hat{c}_2,\hat{\nu}_1,\hat{\gamma}) \nonumber \\
   & = &  \lim_{n\rightarrow\infty} \lp  - \sum_{i=1}^n \frac{\hat{\gamma}^2\s_i^4\c_i^2}{\lambda+\hat{\gamma}\s_i^2}
+ \frac{\alpha\lp \hat{c}_2^2+\bar{\sigma}^2\rp}{\alpha +\frac{1}{n}\sum_{i=1}^n \frac{\s_i^2}{\lambda+\hat{\gamma}\s_i^2}} + \hat{\gamma} \sum_{i=1}^n \s_i^2\c_i^2 - \hat{\gamma} \hat{c}_2^2 \rp \nonumber \\
   & = &   \lim_{n\rightarrow\infty}  \lp \sum_{i=1}^n \frac{\hat{\gamma}\s_i^2\c_i^2}{\lambda +\hat{\gamma}\s_i^2}
+ \frac{\alpha\lp \hat{c}_2^2+\bar{\sigma}^2\rp}{\alpha +\frac{1}{n}\sum_{i=1}^n \frac{\s_i^2}{\lambda +\hat{\gamma}\s_i^2}} - \hat{\gamma} \hat{c}_2^2 \rp \nonumber \\
   & = &   \lim_{n\rightarrow\infty}   \sum_{i=1}^n \frac{\hat{\gamma}\s_i^2\c_i^2}{\lambda+\hat{\gamma}\s_i^2}
+ \hat{\gamma}\lp \hat{c}_2^2+\bar{\sigma}^2\rp  - \hat{\gamma} \hat{c}_2^2  \nonumber \\
   & = &   \lim_{n\rightarrow\infty}  \sum_{i=1}^n \frac{\hat{\gamma}\s_i^2\c_i^2}{\lambda +\hat{\gamma}\s_i^2}
+ \hat{\gamma}\bar{\sigma}^2, \label{eq:ridgeta18a10}
\end{eqnarray}
where $\hat{c}_2$ and $\hat{\gamma}$ are given through (\ref{eq:ridgeta18a0b4c13}), (\ref{eq:ridgeta18a0b4c15}), and (\ref{eq:ridgeta18a0b4c16}). For the very same $\hat{c}_2$ and $\hat{\gamma}$, we can then from (\ref{eq:ridgeta18a0b4c10d0}) and (\ref{eq:ridgeta18a0b4c8})  obtain for the optimal $c_3$ and $\nu_1$
\begin{eqnarray}
\hat{c}_3=\frac{\sqrt{\hat{c}_2^2+\bar{\sigma}^2}}{1+\frac{1}{\alpha n}\sum_{i=1}^n \frac{\s_i^2}{\lambda+\gamma\s_i^2}}
=\hat{\gamma}\sqrt{\hat{c}_2^2+\bar{\sigma}^2},  \label{eq:ridgeta18a0b4c17d0}
\end{eqnarray}
\begin{eqnarray}
\hat{\nu}_1 =
 \frac{2\sqrt{\alpha}\left | \sqrt{\hat{c}_2^2+\bar{\sigma}^2} -\hat{c}_3 \right |}{\frac{1}{n}\sum_{i=1}^n \frac{\s_i^2}{\lambda+\hat{\gamma}\s_i^2}}
 =
 \frac{2\hat{\gamma}\left | \sqrt{\hat{c}_2^2+\bar{\sigma}^2} -\hat{c}_3\right |}{\sqrt{\alpha}(1-\hat{\gamma}) }
 =
 \frac{2\hat{\gamma} \sqrt{\hat{c}_2^2+\bar{\sigma}^2} }{\sqrt{\alpha} }. \label{eq:ridgeta18a0b4c17d1}
\end{eqnarray}
For the completeness we also determine the optimal $\hat{c}_4^2=\lim_{n\rightarrow \infty} \mE_{\g,\h,\v}\|\hat{\x}\|_2^2$
\begin{eqnarray}
\hat{c}_4^2 & = & \lim_{n\rightarrow \infty} \mE_{\g,\h,\v}\| \hat{\x}\|_2^2 \nonumber \\
 & = & \lim_{n\rightarrow \infty} \mE_{\g,\h,\v}\left \| \lp \lambda I +\hat{\gamma}\Sigma\Sigma\rp^{-1}\lp -\frac{\hat{\nu}_1}{2\sqrt{n}} \Sigma\h +\gamma\Sigma\Sigma\c\rp \right \|_2^2 \nonumber \\
 & = &
\lim_{n\rightarrow \infty}  \frac{\hat{\nu}_1^2}{4n}\sum_{i=1}^n\frac{\s_i^2}{\lp \lambda+\hat{\gamma}\s_i^2\rp^2}
+
\lim_{n\rightarrow \infty}  \sum_{i=1}^n\frac{\hat{\gamma}^2\s_i^4\c_i^2}{\lp \lambda+\hat{\gamma}\s_i^2\rp^2}. \label{eq:ridgeta18a0b4c17d1e0}
\end{eqnarray}
We summarize the above discussion in the following theorem (basically a joint Ridge analogoue to the GLS related Lemma \ref{lemma:lemma2} and Theorem \ref{thm:thm2}).

\begin{theorem}(Characterization of Ridge estimator) For a given fixed real positive number $\lambda$, assume the setup of Lemma \ref{lemma:ridgelemma1} and Theorem \ref{thm:ridgethm1} with  $\xi_{rr}(\lambda)$ as in (\ref{eq:ridgerandlincons1}) or (\ref{eq:ridgerandlincons1a0}) and $\xi_{rd,r}$ as in (\ref{eq:ridgeta16}). Let $\hat{\beta}=\beta_{rr}(\lambda)$ where $\beta_{rr}(\lambda)$ is as in
(\ref{eq:model014}). Also, let the Ridge estimator prediction risk, $R(\bar{\beta},\hat{\beta})=R(\bar{\beta},\beta_{rr}(\lambda))$, be defined via (\ref{eq:model020}). Assuming a large $n$ linear regime with $\alpha=\lim_{n\rightarrow\infty}\frac{m}{n}$, let unique $\hat{\gamma}$ be such that
\begin{eqnarray}
\lim_{n\rightarrow\infty} \frac{1}{n}\sum_{i=1}^n \frac{\hat{\gamma}\s_i^2}{\lambda+\hat{\gamma}\s_i^2} =\alpha (1-\hat{\gamma}).  \label{eq:ridgethmaaeq1}
\end{eqnarray}
Set
\begin{eqnarray}
a_2 =
\frac{\lim_{n\rightarrow\infty} \frac{1}{n}\sum_{i=1}^n \frac{\s_i^4}{\lp \lambda +\hat{\gamma}\s_i^2\rp^2}}{\lp \alpha + \lim_{n\rightarrow\infty} \frac{1}{n}\sum_{i=1}^n \frac{\s_i^2}{\lambda+\hat{\gamma}\s_i^2}\rp^2}
=\frac{\hat{\gamma}^2}{\alpha^2 } \lim_{n\rightarrow\infty} \frac{1}{n}\sum_{i=1}^n \frac{\s_i^4}{\lp \lambda+\hat{\gamma}\s_i^2\rp^2}.  \label{eq:ridgethmaaeq2}
\end{eqnarray}
One then has \vspace{-.0in}
\begin{eqnarray}
 \lim_{n\rightarrow\infty} \mE_{Z,\v}\xi_{rr}(\lambda)
 &= &\xi_{rp,r}=  \xi_{rd,r}  = \lim_{n\rightarrow\infty} \mE_{\g,\h,\v} f_{rd,r}(\g,\h,\v)
 =
\lim_{n\rightarrow\infty}  \sum_{i=1}^n \frac{\hat{\gamma}\s_i^2\c_i^2}{\lambda+\hat{\gamma}\s_i^2}
+ \hat{\gamma}\bar{\sigma}^2 \nonumber \\
 \lim_{n\rightarrow\infty} \mE_{Z,\v} R(\bar{\beta},\beta_{rr}(\lambda))
& = &
\lim_{n\rightarrow\infty} \hat{c}_2^2
  =
 \frac{  \lim_{n\rightarrow\infty} \sum_{i=1}^n \frac{\lambda^2 \s_i^2\c_i^2}{\lp \lambda+\hat{\gamma}\s_i^2\rp^2} +\alpha \bar{\sigma}^2a_2 }{1-\alpha a_2} \nonumber \\
 & & \lim_{n\rightarrow\infty}\hat{\nu}_1
  = \frac{2\hat{\gamma}\sqrt{\hat{c}_2^2+\bar{\sigma}^2}}{\sqrt{\alpha}} \nonumber \\
 \lim_{n\rightarrow\infty} \mE_{Z,\v} \|\y-X\beta_{rr}(\lambda)\|_2^2
& = &
\lim_{n\rightarrow\infty} \hat{c}_3^2
  = \hat{\gamma}^2(\hat{c}_2^2+\bar{\sigma}^2) \nonumber \\
 \lim_{n\rightarrow\infty} \mE_{Z,\v} \|\beta_{rr}(\lambda)\|_2^2
& = &
 \lim_{n\rightarrow\infty}\hat{c}_4^2
 =
 \lim_{n\rightarrow \infty}  \frac{\hat{\nu}_1^2}{4n}\sum_{i=1}^n\frac{\s_i^2}{\lp \lambda+\hat{\gamma}\s_i^2\rp^2}
+
\lim_{n\rightarrow \infty}  \sum_{i=1}^n\frac{\hat{\gamma}^2\s_i^4\c_i^2}{\lp \lambda+\hat{\gamma}\s_i^2\rp^2},
\label{eq:ridgethmaaeq3}
\end{eqnarray}
and for any fixed $\epsilon>0$
\begin{eqnarray}
 \lim_{n\rightarrow\infty}\mP_{Z,\v} \lp  (1-\epsilon)\mE_{Z,\v}\xi_{rr}(\lambda) \leq  \xi_{rr}(\lambda) \leq (1+\epsilon)\mE_{Z,\v}\xi_{rr}(\lambda)\rp
& \longrightarrow & 1 \nonumber \\
 \lim_{n\rightarrow\infty}\mP_{Z,\v} \lp  (1-\epsilon)\mE_{Z,\v} R(\bar{\beta},\beta_{rr}(\lambda))  \leq  R(\bar{\beta},\beta_{rr}(\lambda)) \leq (1+\epsilon)\mE_{Z,\v}R(\bar{\beta},\beta_{rr}(\lambda))\rp
& \longrightarrow & 1  \nonumber \\
 \lim_{n\rightarrow\infty}\mP_{\g,\h,\v} \lp  (1-\epsilon)\hat{\nu}_1 \leq  \nu_1 \leq (1+\epsilon)\hat{\nu}_1\rp
& \longrightarrow & 1 \nonumber \\
 \lim_{n\rightarrow\infty}\mP_{Z,\v} \lp  (1-\epsilon)\mE_{Z,\v}\|\y-X\beta_{rr}(\lambda)\|_2^2 \leq \|\y-X\beta_{rr}(\lambda)\|_2^2  \leq (1+\epsilon)\mE_{Z,\v}\|\y-X\beta_{rr}(\lambda)\|_2^2 \rp
& \longrightarrow & 1\nonumber \\
 \lim_{n\rightarrow\infty}\mP_{Z,\v} \lp  (1-\epsilon)\mE_{Z,\v}\|\beta_{rr}(\lambda)\|_2^2 \leq \|\beta_{rr}(\lambda)\|_2^2  \leq (1+\epsilon)\mE_{Z,\v}\|\beta_{rr}(\lambda)\|_2^2 \rp
& \longrightarrow & 1.\nonumber \\
\label{eq:ridgethmaaeq4}
\end{eqnarray}
\label{thm:ridgethm2}
\end{theorem}\vspace{-.17in}
\begin{proof}
Follows from the above discussion, trivial concentrations of $f_{rd,r}(\g,\h,\v)$, $c_2$, and $\nu_1$, and utilization of literally the asme arguments as in the proof of Theorem  \ref{thm:thm2}.
\end{proof}

\vspace{.1in}
\noindent \underline{ 4) \textbf{\emph{Double checking strong random duality:}}} As earlier due to the underlying convexity,  the strong random duality and the reversal arguments of \cite{StojnicGorEx10,StojnicRegRndDlt10} are in place as well.

As was the case earlier when we considered the GLS estimator, it is again not that difficult to see that when $\bar{\sigma}=\sigma$, the results of the above theorem match the ones obtained in \cite{HMRT22,Dicker16,DobWag18}.

\subsection{Precise analysis of the LS estimator}
\label{sec:analls}

While the formal definition of ridge regression typically assumes that the value of the ridge parameter  $\lambda$ is positive, the mathematical analysis from Section \ref{sec:analridge} holds unaltered even if $\lambda=0$ provided that $\alpha>1$. As that is precisely the scenario of the LS estimator (see (\ref{eq:model017})-(\ref{eq:model019})), all of the above results automatically translate. However, many of the evaluations simplify and one ultimately has the following (much simpler) corollary of Theorem \ref{thm:ridgethm2}.

\begin{corollary}(Characterization of LS estimator) Assume the setup of Theorem \ref{thm:ridgethm2} with $\lambda =0$ and $\alpha>1$. Let unique $\hat{\gamma}$ be such that
\begin{eqnarray}
\hat{\gamma}=1-\frac{1}{\alpha}.  \label{eq:lsridgethmaaeq1}
\end{eqnarray}
Set
\begin{eqnarray}
a_2 =\frac{1}{\alpha^2}.  \label{eq:lsridgethmaaeq2}
\end{eqnarray}
One then has \vspace{-.0in}
\begin{eqnarray}
 \lim_{n\rightarrow\infty} \mE_{Z,\v}\xi_{ls}
 &= &
  \lim_{n\rightarrow\infty} \mE_{Z,\v}\xi_{rr}(0)
=
\|\bar{\beta}|_2^2+ \hat{\gamma}\bar{\sigma}^2 \nonumber \\
 \lim_{n\rightarrow\infty} \mE_{Z,\v} R(\bar{\beta},\beta_{ls})
& = &
 \lim_{n\rightarrow\infty} \mE_{Z,\v} R(\bar{\beta},\beta_{ls}(0))
  =
 \frac{\alpha \bar{\sigma}^2a_2 }{1-\alpha a_2}=
 \frac{\bar{\sigma}^2 }{\alpha -1} \nonumber \\
 & & \lim_{n\rightarrow\infty} \hat{\nu}_1
  = 2\bar{\sigma}\frac{\sqrt{\alpha -1}}{\alpha} \nonumber \\
 \lim_{n\rightarrow\infty} \mE_{Z,\v} \|\y-X\beta_{ls}\|_2^2
& = &
 \lim_{n\rightarrow\infty} \mE_{Z,\v} \|\y-X\beta_{rr}(0)\|_2^2
= \lim_{n\rightarrow\infty}\hat{c}_3^2
 =   \bar{\sigma}^2\frac{\alpha -1}{\alpha} \nonumber \\
 \lim_{n\rightarrow\infty} \mE_{Z,\v} \|\beta_{ls}\|_2^2
& = &
 \lim_{n\rightarrow\infty} \mE_{Z,\v} \|\beta_{rr}(0)\|_2^2
= 1+
 \lim_{n\rightarrow \infty}  \frac{\hat{\nu}_1^2}{4\hat{\gamma}^2n}\sum_{i=1}^n\frac{1}{\s_i^2}
  =\|\bar{\beta}\|_2^2+
\frac{\bar{\sigma}^2}{\alpha-1} \lim_{n\rightarrow \infty}  \frac{1}{n}\sum_{i=1}^n\frac{1}{\s_i^2},\nonumber \\
\label{eq:lsridgethmaaeq3}
\end{eqnarray}
and for any fixed $\epsilon>0$
\begin{eqnarray}
 \lim_{n\rightarrow\infty}\mP_{Z,\v} \lp  (1-\epsilon)\mE_{Z,\v}\xi_{ls} \leq  \xi_{ls} \leq (1+\epsilon)\mE_{Z,\v}\xi_{ls}\rp
& \longrightarrow & 1 \nonumber \\
 \lim_{n\rightarrow\infty}\mP_{Z,\v} \lp  (1-\epsilon)\mE_{Z,\v} R(\bar{\beta},\beta_{ls})  \leq  R(\bar{\beta},\beta_{ls}) \leq (1+\epsilon)\mE_{Z,\v}R(\bar{\beta},\beta_{ls})\rp
& \longrightarrow & 1  \nonumber \\
 \lim_{n\rightarrow\infty}\mP_{\g,\h,\v} \lp  (1-\epsilon)\hat{\nu}_1 \leq  \nu_1 \leq (1+\epsilon)\hat{\nu}_1\rp
& \longrightarrow & 1 \nonumber \\
 \lim_{n\rightarrow\infty}\mP_{Z,\v} \lp  (1-\epsilon)\mE_{Z,\v}\|\y-X\beta_{ls}\|_2^2 \leq \|\y-X\beta_{ls}\|_2^2  \leq (1+\epsilon)\mE_{Z,\v}\|\y-X\beta_{ls}\|_2^2 \rp
& \longrightarrow & 1\nonumber \\
 \lim_{n\rightarrow\infty}\mP_{Z,\v} \lp  (1-\epsilon)\mE_{Z,\v}\|\beta_{ls}\|_2^2 \leq \|\beta_{ls}\|_2^2  \leq (1+\epsilon)\mE_{Z,\v}\|\beta_{ls}\|_2^2 \rp
& \longrightarrow & 1.\label{eq:lsridgethmaaeq4}
\end{eqnarray}
\label{thm:lsridgethm2}
\end{corollary}\vspace{-.17in}
\begin{proof}
 The discussion above Theorem \ref{thm:ridgethm2} holds for $\lambda=0$ as well which implies that all the expressions given in Theorem \ref{thm:ridgethm2} are true for $\lambda=0$ as well. Plugging $\lambda=0$ in (\ref{eq:ridgethmaaeq1}),  (\ref{eq:ridgethmaaeq2}), and  (\ref{eq:ridgethmaaeq3}) then gives (\ref{eq:lsridgethmaaeq1}),  (\ref{eq:lsridgethmaaeq2}), and  (\ref{eq:lsridgethmaaeq3}).
\end{proof}

\subsection{Numerical results}
\label{sec:numresuncorr}

We supplement the above theoretical results with the ones obtained through numerical simulations. Let a matrix $\bar{\cA}(q)$ be defined via its $ij$-th entry, $\bar{\cA}_{ij}(q)$, in the following way
 \begin{eqnarray}
\bar{\cA}_{ij}(q)=q^{|j-i|}.  \label{eq:numuncorr1}
\end{eqnarray}
Let then $\cA(q)$ be
 \begin{eqnarray}
\cA(q)=I + \bar{\cA}(q).  \label{eq:numuncorr2}
\end{eqnarray}
 For the covariance matrices, we take
\begin{eqnarray}
A & = &\cA(q) \nonumber \\
\overline{A} &= & \cA(q_v).
  \label{eq:numuncorr3}
\end{eqnarray}
In Figure \ref{fig:fig1}, we choose $q=0.5,q_v=0.4$, and  $n=600$. The full curves are the theoretical predictions and the dots are the corresponding simulated values. Even for fairly small dimensions, we observe an excellent agreement between the theoretical predictions and simulations results.

Statistical setups in all simulations are chosen to perfectly match the ones used in the corresponding theoretical analyses. However, as the theoretical results hold way more generally, we conducted simulations for two additional (and different) scenarios. In one of them we generated $A$ and $E$ as  iid $\pm1$ Bernoullis and in the other, as iid $\sqrt{12}\mbox{Unif}[-0.5,0.5]$. The obtained results were virtually identical to those already shown in Figure \ref{fig:fig1} and there was no point including them..

\begin{figure}[h]
\centering
\centerline{\includegraphics[width=1\linewidth]{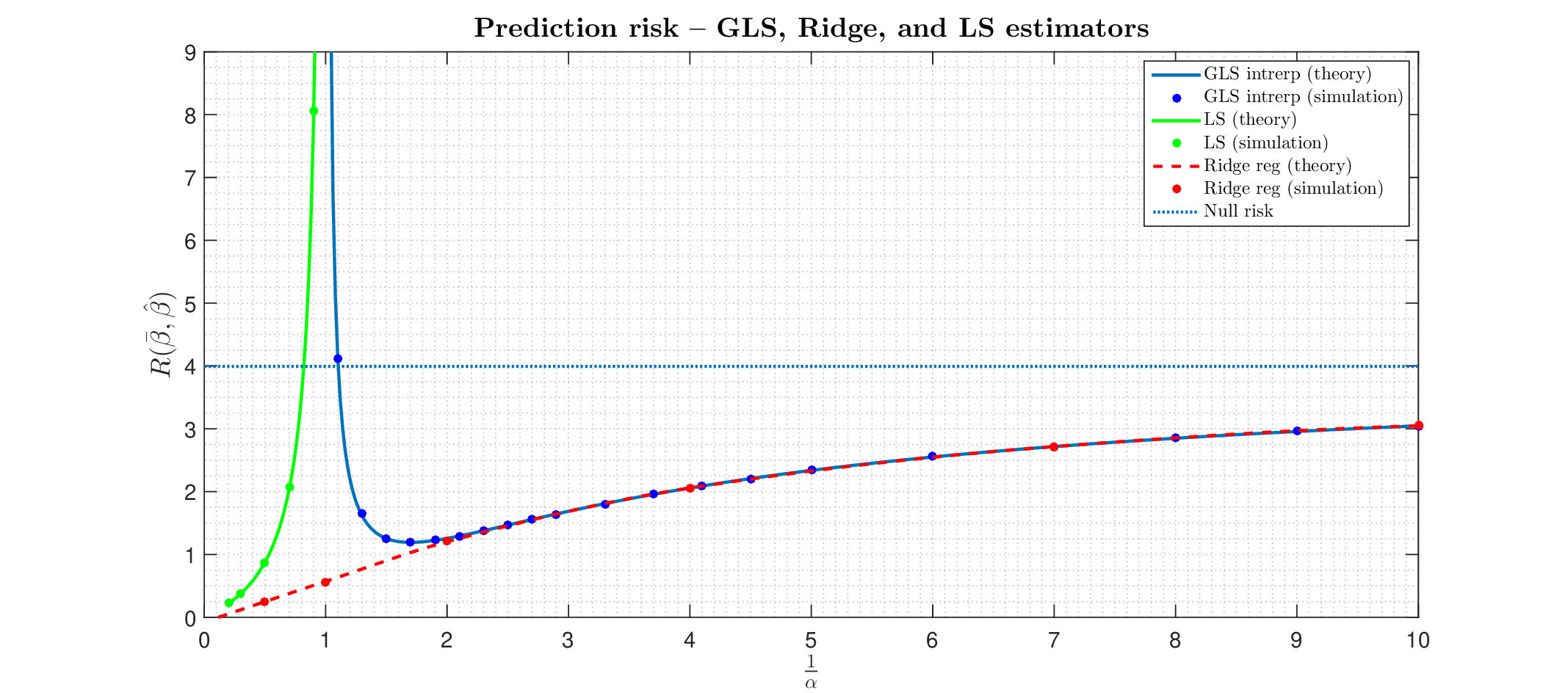}}
\caption{Prediction risk -- all three estimators (GLS, Ridge, and LS); row-correlated features $X$; Covariance matrices are: $A=\cA(q)$, and $\overline{A}=\cA(q_v)$; $q=0.5,q_v=0.4$.}
\label{fig:fig1}
\end{figure}

\section{Inter-rows correlations}
\label{sec:rowcorr}

Models considered so far assumed that all feature vectors $X_{i,:}$ (basically the rows of matrix $X$) are independent of each other. In this section we show how the mathematical machinery introduced earlier covers scenarios with correlated rows of $X$. To that end we now assume for a full rank $\overline{\overline{A}}\in\mR^{m\times m}$ that
\begin{eqnarray}
X=\overline{\overline{A}}ZA.\label{eq:corra0}
\end{eqnarray}
Let the SVD of $\overline{\overline{A}}$ be given as
\begin{eqnarray}
 \overline{\overline{A}}=\overline{\overline{U}}\overline{\overline{\Sigma}} \overline{\overline{V}}^T.\label{eq:corra1}
\end{eqnarray}
One can then rewrite  (\ref{eq:model01})  and (\ref{eq:model05}) as
\begin{eqnarray}
\y=X\bar{\beta}+\e=\overline{\overline{A}}ZA\bar{\beta}+\sigma\overline{A}\v. \label{eq:corrmodel05}
\end{eqnarray}
Using the SVDs ((\ref{eq:model08}) and (\ref{eq:corra1}) one can then write
 \begin{eqnarray}
\y=\overline{\overline{A}}ZA\bar{\beta}+\sigma \overline{A}\v
=
\overline{\overline{U}}\overline{\overline{\Sigma}} \overline{\overline{V}}^T ZU\Sigma V^T\bar{\beta}+\sigma \overline{U}\overline{\Sigma} \overline{V}^T\v. \label{eq:corrmodel09}
\end{eqnarray}
Due to rotational symmetry of both rows and columns of $Z$ and rows of $\v$, (\ref{eq:model09}) and (\ref{eq:model010}) are  statistically equivalent to the following
\begin{eqnarray}
\y=\overline{\overline{U}}\overline{\overline{\Sigma}} Z \Sigma V^T\bar{\beta}+\sigma \overline{U}\overline{\Sigma} \v. \label{eq:corrmodel010}
\end{eqnarray}
This basically mean that one can view $X$ and $\e$ as
\begin{eqnarray}
X=\overline{\overline{U}}\overline{\overline{\Sigma}}  Z \Sigma V^T, \quad \e=\overline{U}\overline{\Sigma} \v. \label{eq:corrmodel010a0}
\end{eqnarray}
All three estimators (GLS, Ridge, and LS) analyzed earlier can be analyzed again within the context of the generalized model given by (\ref{eq:corrmodel010}) and (\ref{eq:corrmodel010a0}). In fact, all Section \ref{sec:analgls} GLS related results including  Theorem \ref{thm:thm2}, continue to hold with only $\bar{\sigma}$ redefined as
\begin{eqnarray}
 \bar{\sigma} =\sigma \sqrt{\frac{\tr\lp \overline{\Sigma}\overline{U}^T   \overline{\overline{U}} \overline{\overline{\Sigma}} \overline{\overline{\Sigma}}    \overline{\overline{U}}^T \overline{U}\overline{\Sigma}\rp}{m}}. \label{eq:corrmodel010a0b0}
\end{eqnarray}
On the other hand the Ridge estimator results need to be carefully readdressed. We do so by parallelling the presentation of Section \ref{sec:analridge} as closely as possible. As expected, we proceed in a much faster fashion by avoiding duplicating already shown arguments and focusing on emphasizing the key differences.

We start by recognizing that, relying on (\ref{eq:corrmodel010}) and (\ref{eq:corrmodel010a0}) one, analogously to (\ref{eq:ridgerandlincons1a1}),  now has
\begin{eqnarray}
\bar{\xi}_{rr}(\lambda) \triangleq \min_{\beta,\|\z\|_2=c_{3,s}} & &  \lambda\|\beta\|_2^2 + \frac{c_{3,s}^2}{m} \nonumber \\
\mbox{subject to} & & \overline{\overline{U}}\overline{\overline{\Sigma}}Z\Sigma V^T\beta=\overline{\overline{U}}\overline{\overline{\Sigma}}Z\Sigma V^T\bar{\beta}+\sigma \overline{U}\overline{\Sigma}\v+\z. \label{eq:corrridgerandlincons1a1}
\end{eqnarray}
it is also not that difficult to see that writing the Lagrangian gives the following analogue to (\ref{eq:ridgerandlincons2})
\begin{eqnarray}
\bar{\xi}_{rr}(\lambda) = \min_{\beta,\|\z\|_2=c_{3,s}} \max_{\nu} \lp \lambda\|\beta\|_2^2 +
\nu^T \lp
\overline{\overline{U}}\overline{\overline{\Sigma}} Z\Sigma V^T\beta-\overline{\overline{U}}\overline{\overline{\Sigma}} Z\Sigma V^T\bar{\beta}-\sigma \overline{U}\overline{\Sigma}\v -\z \rp  + \frac{c_{3,s}^2}{m}\rp. \label{eq:corrridgerandlincons2}
\end{eqnarray}
For the completeness we also define
\begin{eqnarray}
\bar{\beta}_{rr}(\lambda) \triangleq \mbox{arg} \min_{\beta,\|\z\|_2=c_{3,s}} & &  \lambda\|\beta\|_2^2 + \frac{c_{3,s}^2}{m} \nonumber \\
\mbox{subject to} & & \overline{\overline{U}}\overline{\overline{\Sigma}}Z\Sigma V^T\beta=\overline{\overline{U}}\overline{\overline{\Sigma}}Z\Sigma V^T\bar{\beta}+\sigma \overline{U}\overline{\Sigma}\v+\z. \label{eq:corrmodel014}
\end{eqnarray}

To analytically characterize the above problem we again utilize the RDT machinery.

\subsection{Analysis of the Ridge estimator under $X$ row-correlations }
\label{sec:corrrandlinconsrdtridge}

As usual, we proceed by discussing in detail each of the four main RDT principles.

\vspace{.1in}

\noindent \underline{1) \textbf{\emph{Algebraic characterization:}}}  The following lemma is basically a correlated rows analogue to Lemma \ref{lemma:ridgelemma1}) and conveniently summarizes the above algebraic considerations.

\begin{lemma}(Algebraic optimization representation) Assume the setup of Lemmas \ref{lemma:lemma1} and \ref{lemma:ridgelemma1}   and let the unitary matrix $\overline{\overline{U}}\in\mR^{m\times m}$ and diagonal matrix $\overline{\overline{\Sigma}}\in\mR^{m\times m}$ be given. let $\bar{\xi}_{rr}(\lambda)$ be as in (\ref{eq:corrridgerandlincons1a1}). Set ${\mathcal D}=\{\beta,\z,c_{3,s}|\beta\in\mR^n,\z\in\mR^m,c_{3,s}\in\mR,\|\z\|_2=c_{3,s}\}$ and
\begin{align}\label{eq:corrridgeta11}
\bar{f}_{rp,r}(Z,\v) & \triangleq  \min_{{\mathcal D}} \max_{\nu} \lp \lambda\|\beta\|_2^2 +\nu^T
\lp \overline{\overline{U}}\overline{\overline{\Sigma}} Z\Sigma V^T\beta
- \overline{\overline{U}}\overline{\overline{\Sigma}}Z\Sigma V^T\bar{\beta}-\sigma \overline{U}\overline{\Sigma}\v -\z\rp +\frac{c_{3,s}^2}{m} \rp
 \hspace{.02in} (\bl{\textbf{random primal}})
\nonumber \\
\bar{\xi}_{rp,r} & \triangleq  \hspace{-.04in} \lim_{n\rightarrow\infty } \mE_{Z,\v} \bar{f}_{rp,r}(Z,\v).
\end{align}
Then
\begin{equation}\label{eq:corrridgeta11a0}
\bar{\xi}_{rr}(\lambda)=\bar{f}_{rp,r}(Z,\v) \quad \mbox{and} \quad \lim_{n\rightarrow\infty} \mE_{Z,\v}\bar{\xi}_{rr}(\lambda) =\bar{\xi}_{rp,r}.
\end{equation}
\label{lemma:corrridgelemma1}
\end{lemma}
\begin{proof}
 Follows automatically via the Lagrangian from (\ref{eq:corrridgerandlincons2}).
\end{proof}

\vspace{.1in}
\noindent \underline{2) \textbf{\emph{Determining the random dual:}}} Utilizing concentration
\begin{equation}
\lim_{n\rightarrow\infty}\mP_{Z,\v}\left (\frac{|\bar{f}_{rp,r}(Z,\v)-\mE_{Z,\v}(\bar{f}_{rp,r}(Z,\v))|}{\mE_{Z,\v}(\bar{f}_{rp,r}(Z,\v))}>\epsilon\right )\longrightarrow 0,\label{eq:corrridgeta15}
\end{equation}
we have the following random dual analogue to Theorems \ref{thm:thm1} and \ref{thm:ridgethm1}.

\begin{theorem}(Random dual -- Row-correlated $X$ Ridge estimator) Assume the setup of Lemmas \ref{lemma:ridgelemma1} and \ref{lemma:corrridgelemma1} and Theorems \ref{thm:thm1} and \ref{thm:ridgethm1}. Set
\vspace{-.0in}
\begin{align}
   \bar{f}_{rd,r}(\g,\h,\v) & \triangleq
\hspace{-.0in} \min_{{\mathcal D}}\max_{\nu} \lp \lambda \|\beta\|_2^2+c_2\nu^T\overline{\overline{U}}\overline{\overline{\Sigma}}\g+\|\overline{\overline{\Sigma}}\overline{\overline{U}}^T\nu\|_2\h^T\Sigma V^T\lp \beta -\bar{\beta}\rp -\sigma \nu^T\overline{U}\overline{\Sigma} \v -\nu^T\z  +\frac{c_{3,s}^2}{m} \rp \nonumber \\
&
   \hspace{4.5in} (\bl{\textbf{random dual}})
  \nonumber \\
 \bar{\xi}_{rd,r} & \triangleq  \lim_{n\rightarrow\infty} \mE_{\g,\h,\v} \bar{f}_{rd,r}(\g,\h,\v)  .\label{eq:corridgeta16}
\vspace{-.0in}\end{align}
One then has \vspace{-.0in}
\begin{eqnarray}
  \bar{\xi}_{rd,r} & \triangleq & \lim_{n\rightarrow\infty} \mE_{\g,\h,\v} \bar{f}_{rd,r}(\g,\h,\v)
  \leq
  \lim_{n\rightarrow\infty} \mE_{Z,\v} \bar{f}_{rp,r}(Z,\v)  \triangleq  \bar{\xi}_{rp,r}. \label{eq:corridgeta16a0}
\vspace{-.0in}\end{eqnarray}
and
\begin{eqnarray}
 \lim_{n\rightarrow\infty}\mP_{\g,\h,\v} \lp \bar{f}_{rd,r}(\g,\h,\v)\geq (1-\epsilon)\bar{\xi}_{rd,r}\rp
 \leq  \lim_{n\rightarrow\infty}\mP_{Z,\v} \lp \bar{f}_{rp,r}(Z,\v)\geq (1-\epsilon)\bar{\xi}_{rd,r}\rp.\label{eqcor:ridgeta17}
\end{eqnarray}
\label{thm:corrridgethm1}
\end{theorem}\vspace{-.17in}
\begin{proof}
As the proof of Theorems \ref{thm:thm1} and \ref{thm:ridgethm1}, it follows immediately as a direct application of the Gordon's probabilistic comparison theorem (see, e.g., Theorem B in \cite{Gordon88} and also Theorem 1, Corollary 1, and Section 2.7.2 in \cite{Stojnicgscomp16}).
\end{proof}

 \vspace{.1in}
\noindent \underline{3) \textbf{\emph{Handling the random dual:}}} We start by solving the inner maximization over $\nu$ and obtain analogously to (\ref{eq:ta18a0}) and (\ref{eq:ridgeta18a0})
{\small \begin{eqnarray}
  \bar{f}_{rd,r}(\g,\h,\v)
  & \triangleq &
\min_{\beta,\|\z\|_2=\c_{3,s}}\max_{\nu} \lp \|\beta\|_2^2+c_2\nu^T\overline{\overline{U}}\overline{\overline{\Sigma}}\g+\|\overline{\overline{\Sigma}}\overline{\overline{U}}^T\nu\|_2\h^T\Sigma V^T\lp \beta -\bar{\beta}\rp -\sigma \nu^T\overline{U}\overline{\Sigma} \v -\nu^T\z +\frac{c_{3,s}^2}{m}\rp \nonumber \\
  & = &
\min_{\beta,\|\z\|_2=\c_{3,s}}\max_{\nu} \Bigg (\Big. \|\beta\|_2^2+c_2\nu^T\overline{\overline{U}}\overline{\overline{\Sigma}}\g+\|\overline{\overline{\Sigma}}\overline{\overline{U}}^T\nu\|_2\h^T\Sigma V^T\lp \beta -\bar{\beta}\rp \nonumber \\
& & -\sigma \nu^T \overline{\overline{U}} \overline{\overline{\Sigma}}\overline{\overline{\Sigma}}^{-1}\overline{\overline{U}}^T\overline{U}\overline{\Sigma} \v -\nu^T\overline{\overline{U}} \overline{\overline{\Sigma}}\overline{\overline{\Sigma}}^{-1}\overline{\overline{U}}^T\z +\frac{c_{3,s}^2}{m}\Big. \Bigg) \nonumber \\
   & = &
\hspace{-.1in}\min_{\x,\|\z\|_2=c_{3,s},c_2\geq 0}\max_{\nu_s\geq 0} \lp \lambda \|\x\|_2^2+\nu_s\h^T\Sigma \lp \x -\c\rp +\nu_s\|c_2\g-\sigma \overline{\overline{\Sigma}}^{-1}\overline{\overline{U}}^T \overline{U}\overline{\Sigma} \v -\overline{\overline{\Sigma}}^{-1}\overline{\overline{U}}^T\z\|_2 +\frac{c_{3,s}^2}{m}\rp \nonumber \\
   & = &
\hspace{-.1in}\min_{\x,\|\z\|_2=c_{3,s},c_2\geq 0}\max_{\nu_s\geq 0} \lp \lambda \|\x\|_2^2+\nu_s\h^T\Sigma \lp \x -\c\rp +\nu_s\|c_2\g-\sigma \overline{\overline{\Sigma}}^{-1}\overline{\overline{U}}^T \overline{U}\overline{\Sigma} \v -\overline{\overline{\Sigma}}^{-1}\z\|_2 +\frac{c_{3,s}^2}{m}\rp, \nonumber \\\label{eq:corrridgeta18a0}
\end{eqnarray}
}

\noindent where $\x,\c,$ and $c_2$ are as in (\ref{eq:ta18a0b0}) and (\ref{eq:ta18a0b0c0}) and the last equality follows after  unitary invariant change $\z\rightarrow \overline{\overline{U}}^T\z$. After setting
\begin{eqnarray}
  \bar{\cL}_r\lp\x,\z,\c_{3,s},c_2,\nu_s;\g,\h,\v\rp \triangleq  \lambda \|\x\|_2^2+\nu_s\h^T\Sigma \lp \x -\c\rp +\nu_s\|c_2\g-\sigma \overline{\overline{\Sigma}}^{-1}\overline{\overline{U}}^T \overline{U}\overline{\Sigma} \v -\overline{\overline{\Sigma}}^{-1}\z\|_2 +\frac{c_{3,s}^2}{m}, \label{eq:corrridgeta18a0b1}
\end{eqnarray}
one rewrites (\ref{eq:corrridgeta18a0}) as
\begin{eqnarray}
  \bar{f}_{rd,r}(\g,\h,\v)
  =
\min_{\x,\|\z\|_2=c_{3,s},c_2\geq 0}\max_{\nu_s\geq 0}  \bar{\cL}_r\lp\x,\z,c_{3,s},c_2,\nu_s;\g,\h,\v\rp, \label{eq:corridgeta18a0b2}
\end{eqnarray}
and keeping $c_{3,s}$, $c_2$, and $\nu_s$ fixed, continues, analogously to (\ref{eq:ta18a0b2c0}) and (\ref{eq:ridgeta18a0b2c0}), by considering the following minimization  over $\x$ and $\z$
\begin{eqnarray}
  \bar{\cL}_{1,r}\lp c_{3,s},c_2,\nu_s\rp \triangleq  \min_{\x,\z} & & \lambda \|\x\|_2^2+\nu_s\h^T\Sigma \lp \x -\c\rp +\nu_s\|c_2\g-\sigma \overline{\overline{\Sigma}}^{-1}\overline{\overline{U}}^T \overline{U}\overline{\Sigma} \v -\overline{\overline{\Sigma}}^{-1}\z\|_2 +\frac{c_{3,s}^2}{m} \nonumber \\
  \mbox{subject to} & &   \|\Sigma\lp\x-\c\rp\|_2=c_2, \|\z\|_2=c_{3,s}, \label{eq:corrridgeta18a0b2c0}
\end{eqnarray}
where as earlier, in order to simplify the notation, the functional dependence on randomness ($\g,\h,\v$) is kept implicit. Moreover, analogously to (\ref{eq:ta18a0b2c1}) and (\ref{eq:ridgeta18a0b2c1}), we further have
\begin{eqnarray}
  \bar{\cL}_{1,r}\lp c_{3,s},c_2,\nu_s\rp
   & = &  \max_{\gamma}   \min_{\x,\|\z\|_2=c_{3,s}}   \bar{\cL}_{2,r}\lp \x,\z,\gamma;c_{3,s},c_2,\nu_s\rp, \label{eq:corrridgeta18a0b2c1}
\end{eqnarray}
with
\begin{eqnarray}
  \bar{\cL}_{2,r}\lp \x,\z,\gamma;c_{3,s},c_2,\nu_s\rp
  & \triangleq &
\lambda \|\x\|_2^2+\nu_s\h^T\Sigma \lp \x -\c\rp +\nu_s\|c_2\g-\sigma \overline{\overline{\Sigma}}^{-1}\overline{\overline{U}}^T \overline{U}\overline{\Sigma} \v -\overline{\overline{\Sigma}}^{-1}\z\|_2 +\frac{c_{3,s}^2}{m}
\nonumber \\
  & & +\gamma \|\Sigma\lp\x-\c\rp\|_2^2 -\gamma c_2^2 . \label{eq:corrridgeta18a0b2c2}
\end{eqnarray}
Taking the optimal $\x$ as in (\ref{eq:ta18a0b4c0}) and (\ref{eq:ridgta18a0b4})
\begin{eqnarray}
 \hat{\x}= \lp \lambda I +\gamma\Sigma\Sigma\rp^{-1}\lp -\frac{1}{2}\nu_s \Sigma\h +\gamma\Sigma\Sigma\c\rp, \label{eq:corrridgta18a0b4}
\end{eqnarray}
with appropriate scaling $\nu_s\rightarrow \frac{\nu_1}{\sqrt{n}}$  gives the following analogue to (\ref{eq:ridgeta18a0b4c1})
\begin{eqnarray}
  \bar{\cL}_{2,r}\lp \hat{\x},\z,\gamma;c_{3,s},c_2,\nu_1\rp
 & = &
-\lp -\frac{1}{2\sqrt{n}}\nu_1 \Sigma\h +\gamma\Sigma\Sigma\c\rp^T   \lp \lambda I +\gamma\Sigma\Sigma\rp^{-1}\lp -\frac{1}{2\sqrt{n}}\nu_1 \Sigma\h +\gamma\Sigma\Sigma\c\rp
\nonumber \\
   & & -\frac{1}{\sqrt{n}}\nu_1\h^T\Sigma\c+ \frac{1}{\sqrt{n}}\nu_1\|c_2\g-\sigma \overline{\overline{\Sigma}}^{-1}\overline{\overline{U}}^T \overline{U}\overline{\Sigma} \v -\overline{\overline{\Sigma}}^{-1}\z \|_2  +\frac{c_{3,s}^2}{m}+\gamma \c^T\Sigma\Sigma\c -\gamma c_2^2. \nonumber \\ \label{eq:corrridgeta18a0b4c1}
\end{eqnarray}

We now switch to the following optimization over $\z$,
\begin{eqnarray}
\bar{\cL}_{3,r} \triangleq \min_\z & &  \|c_2\g-\sigma \overline{\overline{\Sigma}}^{-1}\overline{\overline{U}}^T \overline{U}\overline{\Sigma} \v -\overline{\overline{\Sigma}}^{-1}\z \|_2^2   \nonumber \\ \
\mbox{subject to} & & \|\z\|_2=c_{3,s}.
\label{eq:corrridgeta18a0b4c1d0}
\end{eqnarray}
After setting
\begin{eqnarray}
\theta=c_2\g-\sigma \overline{\overline{\Sigma}}^{-1}\overline{\overline{U}}^T \overline{U}\overline{\Sigma} \v,
\label{eq:corrridgeta18a0b4c1d0e0}
\end{eqnarray}
writing the Lagrangian and utilizing strong Lagrange duality, we obtain
\begin{eqnarray}
\bar{\cL}_{3,r} = \min_{\z}\max_{\gamma_1} \bar{\cL}_{4,r}(\z,\gamma_1)
=\max_{\gamma_1}\min_{\z} \bar{\cL}_{4,r}(\z,\gamma_1)
\label{eq:corrridgeta18a0b4c1d1}
\end{eqnarray}
with
\begin{eqnarray}
\bar{\cL}_{4,r}(\z,\gamma_1) = \|\theta -\overline{\overline{\Sigma}}^{-1}\z \|_2^2 +\gamma_1 \|\z\|_2^2-\gamma_1c_{3,s}^2.
\label{eq:corrridgeta18a0b4c1d2}
\end{eqnarray}
Computing the derivative with respect to $\z$ gives
\begin{eqnarray}
\frac{d \bar{\cL}_{4,r}(\z,\gamma_1)}{d\z} = 2 \overline{\overline{\Sigma}}^{-1}\overline{\overline{\Sigma}}^{-1}\z  -2\overline{\overline{\Sigma}}^{-1}\theta +2\gamma_1\z.
\label{eq:corrridgeta18a0b4c1d3}
\end{eqnarray}
Equalling the derivative to zero gives the optimal $\z$
\begin{eqnarray}
\hat{\z}=\lp \gamma_1 I + \overline{\overline{\Sigma}}^{-1}\overline{\overline{\Sigma}}^{-1}\rp^{-1} \overline{\overline{\Sigma}}^{-1}\theta,
\label{eq:corrridgeta18a0b4c1d4}
\end{eqnarray}
or in other words
\begin{eqnarray}
\hat{\z}_j=\frac{1}{\gamma_1  + \frac{1}{\os_j^2}} \frac{\theta_j}{\os_j},
\label{eq:corrridgeta18a0b4c1d5}
\end{eqnarray}
where
\begin{eqnarray}
\os\triangleq \mbox{diag} (\overline{\overline{\Sigma}}).
\label{eq:corrridgeta18a0b4c1d5e0}
\end{eqnarray}
Taking $\hat{\z}$ from (\ref{eq:corrridgeta18a0b4c1d5}) and plugging it back in (\ref{eq:corrridgeta18a0b4c1d2}) gives
\begin{eqnarray}
\bar{\cL}_{4,r}(\hat{\z},\gamma_1) =\sum_{j=1}^{m} \lp \theta_j^2 -\frac{1}{\gamma_1  + \frac{1}{\os_j^2}} \frac{\theta_j^2}{\os_j^2} \rp
 -\gamma_1 c_{3,s}^2,
\label{eq:corrridgeta18a0b4c1d6}
\end{eqnarray}
which together with (\ref{eq:corrridgeta18a0b4c1d1}) gives
\begin{eqnarray}
\bar{\cL}_{3,r}
=\max_{\gamma_1}\min_{\z} \bar{\cL}_{4,r}(\z,\gamma_1)
=\max_{\gamma_1}  \bar{\cL}_{4,r}(\hat{\z},\gamma_1)
=\max_{\gamma_1}
\sum_{j=1}^{m} \lp \theta_j^2 -\frac{1}{\gamma_1  + \frac{1}{\os_j^2}} \frac{\theta_j^2}{\os_j^2} \rp
 -\gamma_1 c_{3,s}^2.
\label{eq:corrridgeta18a0b4c1d7}
\end{eqnarray}
Combining (\ref{eq:corrridgeta18a0b4c1}), (\ref{eq:corrridgeta18a0b4c1d0}), and (\ref{eq:corrridgeta18a0b4c1d7}), we further find
\begin{eqnarray}
  \bar{\cL}_{2,r}\lp \hat{\x},\hat{\z},\gamma;c_{3,s},c_2,\nu_1\rp
 & = &
-\lp -\frac{1}{2\sqrt{n}}\nu_1 \Sigma\h +\gamma\Sigma\Sigma\c\rp^T   \lp \lambda I +\gamma\Sigma\Sigma\rp^{-1}\lp -\frac{1}{2\sqrt{n}}\nu_1 \Sigma\h +\gamma\Sigma\Sigma\c\rp
\nonumber \\
   & & -\frac{1}{\sqrt{n}}\nu_1\h^T\Sigma\c+ \frac{1}{\sqrt{n}}\nu_1\sqrt{\bar{\cL}_{3,r}}  +\frac{c_{3,s}^2}{m}+\gamma \c^T\Sigma\Sigma\c -\gamma c_2^2 \nonumber \\
    & = &
-\lp -\frac{1}{2\sqrt{n}}\nu_1 \Sigma\h +\gamma\Sigma\Sigma\c\rp^T   \lp \lambda I +\gamma\Sigma\Sigma\rp^{-1}\lp -\frac{1}{2\sqrt{n}}\nu_1 \Sigma\h +\gamma\Sigma\Sigma\c\rp
\nonumber \\
   & & -\frac{1}{\sqrt{n}}\nu_1\h^T\Sigma\c+ \frac{1}{\sqrt{n}}\nu_1
   \sqrt{\max_{\gamma_1}\bar{\cL}_{4,r}(\hat{\z},\gamma_1)}
   +\frac{c_{3,s}^2}{m}+\gamma \c^T\Sigma\Sigma\c -\gamma c_2^2 \nonumber \\
       & = & \max_{\gamma_1} \bar{\cL}_{5,r} (\gamma_1), \label{eq:corrridgeta18a0b4d8}
\end{eqnarray}
where
\begin{eqnarray}
  \bar{\cL}_{5,r}\lp \gamma_1 \rp
    & = &
-\lp -\frac{1}{2\sqrt{n}}\nu_1 \Sigma\h +\gamma\Sigma\Sigma\c\rp^T   \lp \lambda I +\gamma\Sigma\Sigma\rp^{-1}\lp -\frac{1}{2\sqrt{n}}\nu_1 \Sigma\h +\gamma\Sigma\Sigma\c\rp -\frac{1}{\sqrt{n}}\nu_1\h^T\Sigma\c
\nonumber \\
   & & + \frac{1}{\sqrt{n}}\nu_1\sqrt{\bar{\cL}_{4,r}(\hat{\z},\gamma_1)}
   +\frac{c_{3,s}^2}{m}+\gamma \c^T\Sigma\Sigma\c -\gamma c_2^2 \nonumber \\
    & = &
-\lp -\frac{1}{2\sqrt{n}}\nu_1 \Sigma\h +\gamma\Sigma\Sigma\c\rp^T   \lp \lambda I +\gamma\Sigma\Sigma\rp^{-1}\lp -\frac{1}{2\sqrt{n}}\nu_1 \Sigma\h +\gamma\Sigma\Sigma\c\rp
\nonumber \\
   & & -\frac{1}{\sqrt{n}}\nu_1\h^T\Sigma\c+ \frac{\sqrt{m}}{\sqrt{n}}\nu_1
   \sqrt{ \frac{1}{m}\sum_{j=1}^{m} \lp \theta_j^2 -\frac{1}{\gamma_1  + \frac{1}{\os_j^2}} \frac{\theta_j^2}{\os_j^2} \rp
 -\gamma_1 \frac{c_{3,s}^2}{m} }  +\frac{c_{3,s}^2}{m}+\gamma \c^T\Sigma\Sigma\c -\gamma c_2^2. \nonumber \\
  \label{eq:corrridgeta18a0b4d9}
\end{eqnarray}
Let
\begin{eqnarray}
{\mathcal U} & = & \overline{\overline{U}}^T\overline{U} \nonumber \\
\q_j & = & \frac{1}{\os_i^2}\sum_{l=1}^{m}{\mathcal U}_{il}^2\bar{\s}_l^2.
  \label{eq:corrridgeta18a0b4d9e0}
\end{eqnarray}
We then recognize that
\begin{eqnarray}
\mE_{\g,\v}\theta_j^2=c_2^2+\sigma^2\q_j.
  \label{eq:corrridgeta18a0b4d9e1}
\end{eqnarray}
Utilizing appropriate scaling $c_3\rightarrow\frac{c_{3,s}}{\sqrt{m}}$ (while keeping $c_3$, $c_2$, $\nu_1$, and $\gamma$ fixed) and relying on concentrations, one finds analogously to (\ref{eq:ta18a0b4c3}) and (\ref{eq:ridgeta18a0b4c3})
 \begin{align}
\lim_{n\rightarrow\infty} \hspace{-.06in} \mE_{\g,\h,\v}\bar{\cL}_{5,r}\lp \gamma_1\rp
 & = \hspace{-.05in}
\lim_{n\rightarrow\infty} \hspace{-.06in} \mE \Bigg ( \Big. -\lp -\frac{1}{2\sqrt{n}}\nu_1 \Sigma\h +\gamma\Sigma\Sigma\c\rp^T   \lp \lambda I +\gamma\Sigma\Sigma\rp^{-1}
  \lp -\frac{1}{2\sqrt{n}}\nu_1 \Sigma\h +\gamma\Sigma\Sigma\c\rp   \nonumber \\
   & \quad  -\frac{1}{\sqrt{n}}\nu_1\h^T\Sigma\c + \frac{\sqrt{m}}{\sqrt{n}}\nu_1
   \sqrt{ \frac{1}{m}\sum_{j=1}^{m} \lp \theta_j^2 -\frac{1}{\gamma_1  + \frac{1}{\os_j^2}} \frac{\theta_j^2}{\os_j^2} \rp
 -\gamma_1 \frac{c_{3,s}^2}{m} }
 \nonumber \\
& \quad  +\frac{c_{3,s}^2}{m}+\gamma \c^T\Sigma\Sigma\c -\gamma c_2^2 \Big.\Bigg ) \nonumber \\
  & =  \lim_{n\rightarrow\infty} \Bigg ( \Big. -\frac{\nu_1^2}{4n}\sum_{i=1}^n \frac{\s_i^2}{\lambda+\gamma\s_i^2}
  - \sum_{i=1}^n \frac{\gamma^2\s_i^4\c_i^2}{\lambda+\gamma\s_i^2}
+\nu_1\sqrt{\alpha}\sqrt{l_4} +c_3^2  + \gamma \sum_{i=1}^n \s_i^2\c_i^2 -\gamma c_2^2  \Big.\Bigg ),\label{eq:corrridgeta18a0b4c3}
\end{align}
 with
\begin{eqnarray}
l_4=\frac{1}{m}\sum_{j=1}^{m} \lp c_2^2+\sigma^2\q_j\rp -\frac{1}{m}\sum_{j=1}^{m} \frac{1}{\gamma_1  + \frac{1}{\os_j^2}} \frac{c_2^2+\sigma^2\q_j}{\os_j^2}
 -\gamma_1 c_3^2.\label{eq:corrridgeta18a0b4c3d0}
\end{eqnarray}
A combination of (\ref{eq:corrridgeta18a0b1})-(\ref{eq:corrridgeta18a0b2c1}), (\ref{eq:corrridgeta18a0b4c1d1}), and  (\ref{eq:corrridgeta18a0b4c3}) allows us to arrive at the following optimization
\begin{eqnarray}
\lim_{n\rightarrow\infty}  \mE_{\g,\h,\v}\bar{f}_{rd,r}(\g,\h,\v)
   & = & \lim_{n\rightarrow\infty} \min_{c_3,c_2\geq 0} \max_{\nu_1\geq 0,\gamma,\gamma_1}  \bar{f}_{0,r}(c_3,c_2,\nu_1,\gamma,\gamma_1), \label{eq:corrridgeta18a0b4c5}
\end{eqnarray}
where
\begin{eqnarray}
\bar{f}_{0,r}(c_3,c_2,\nu_1,\gamma,\gamma_1) \triangleq  -\frac{\nu_1^2}{4n}\sum_{i=1}^n \frac{\s_i^2}{\lambda+\gamma\s_i^2}
  - \sum_{i=1}^n \frac{\gamma^2\s_i^4\c_i^2}{\lambda+\gamma\s_i^2}
+\nu_1\sqrt{\alpha}\sqrt{l_4} +c_3^2  + \gamma \sum_{i=1}^n \s_i^2\c_i^2 -\gamma c_2^2 . \nonumber \\ \label{eq:corrridgeta18a0b4c6}
\end{eqnarray}
After taking the derivative with respect to $\nu_1$ we, analogously to (\ref{eq:ta18a0b4c8}) and (\ref{eq:ridgeta18a0b4c8}), find the following optimal $\nu_1$
\begin{eqnarray}
\hat{\nu}_1 =
 \frac{2\sqrt{\alpha}\sqrt{l_4}}{\frac{1}{n}\sum_{i=1}^n \frac{\s_i^2}{\lambda+\gamma\s_i^2}}. \label{eq:corrridgeta18a0b4c8}
\end{eqnarray}
which then,  together with (\ref{eq:corrridgeta18a0b4c6}), gives
\begin{equation}
\max_{\nu_1\geq 0} \bar{f}_{0,r}(c_3,c_2,\nu_1,\gamma,\gamma_1) = \bar{f}_{0,r}(c_3,c_2,\hat{\nu}_1,\gamma,\gamma_1) =
  - \sum_{i=1}^n \frac{\gamma^2\s_i^4\c_i^2}{\lambda+\gamma\s_i^2}
+ \frac{\alpha l_4}{\frac{1}{n}\sum_{i=1}^n \frac{\s_i^2}{\lambda+\gamma\s_i^2}} +c_3^2 + \gamma \sum_{i=1}^n \s_i^2\c_i^2 -\gamma c_2^2.  \label{eq:corrridgeta18a0b4c9}
\end{equation}
Combining further (\ref{eq:corrridgeta18a0b4c5}) and (\ref{eq:corrridgeta18a0b4c9}) we find
\begin{eqnarray}
\lim_{n\rightarrow\infty}  \mE_{\g,\h,\v}\bar{f}_{rd,r}(\g,\h,\v)
   & = & \lim_{n\rightarrow\infty}  \min_{c_3,c_2\geq 0} \max_{\gamma,\gamma_1}  \bar{f}_{0,r}(c_3,c_2,\hat{\nu}_1,\gamma,\gamma_1), \label{eq:corrridgeta18a0b4c10}
\end{eqnarray}
where $\bar{f}_{0,r}(c_3,c_2,\hat{\nu}_1,\gamma,\gamma_1)$  as in (\ref{eq:corrridgeta18a0b4c9}). Computing derivatives with respect to $c_3$ and $\gamma_1$ gives
\begin{eqnarray}
\frac{d \bar{f}_{0,r}(c_3,c_2,\hat{\nu}_1,\gamma,\gamma_1)}{dc_3}
=
 \frac{\alpha}{\frac{1}{n}\sum_{i=1}^n \frac{\s_i^2}{\lambda+\gamma\s_i^2}} \frac{dl_4}{dc_3}+2c_3
=
-2 \frac{\alpha}{\frac{1}{n}\sum_{i=1}^n \frac{\s_i^2}{\lambda+\gamma\s_i^2}} \gamma_1 c_3+2c_3
, \label{eq:corrridgeta18a0b4c10d0e0}
\end{eqnarray}
and
\begin{eqnarray}
\frac{d \bar{f}_{0,r}(c_3,c_2,\hat{\nu}_1,\gamma,\gamma_1)}{d\gamma_1}
=
 \frac{\alpha}{\frac{1}{n}\sum_{i=1}^n \frac{\s_i^2}{\lambda+\gamma\s_i^2}} \frac{dl_4}{d\gamma_1}
=
 \frac{\alpha}{\frac{1}{n}\sum_{i=1}^n \frac{\s_i^2}{\lambda+\gamma\s_i^2}}
\lp  \frac{1}{m}\sum_{j=1}^{m} \frac{1}{\lp\gamma_1  + \frac{1}{\os_j^2}\rp^2} \frac{c_2^2+\sigma^2\q_j}{\os_j^2}
 - c_3^2 \rp
  . \label{eq:corrridgeta18a0b4c10d0e1}
\end{eqnarray}
Equalling these derivatives to zero gives the following two equations for the optimal $\gamma_1$ and $c_3$
\begin{eqnarray}
\hat{\gamma}_1= \frac{\frac{1}{n}\sum_{i=1}^n \frac{\s_i^2}{\lambda+\gamma\s_i^2}}{\alpha},
 \label{eq:corrridgeta18a0b4c10d0e2}
\end{eqnarray}
and
\begin{eqnarray}
\hat{c}_3^2 =  \frac{1}{m}\sum_{j=1}^{m} \frac{1}{\lp \hat{\gamma}_1  + \frac{1}{\os_j^2}\rp^2} \frac{c_2^2+\sigma^2\q_j}{\os_j^2}. \label{eq:corrridgeta18a0b4c10d0e3}
\end{eqnarray}
 Combining (\ref{eq:corrridgeta18a0b4c10d0e2}) and (\ref{eq:corrridgeta18a0b4c10d0e3}) together with (\ref{eq:corrridgeta18a0b4c9}), gives
\begin{eqnarray}
\min_{c_3\geq 0}\max_{\gamma_1} \bar{f}_{0,r}(c_3,c_2,\hat{\nu}_1,\gamma,\gamma_1) = \bar{f}_{0,r}(\hat{c}_3,c_2,\hat{\nu}_1,\gamma,\hat{\gamma}_1),  \label{eq:corrridgeta18a0b4c10d1}
\end{eqnarray}
with
\begin{eqnarray}
 \bar{f}_{0,r}(\hat{c}_3,c_2,\hat{\nu}_1,\gamma,\hat{\gamma}_1)
 & = &
  - \sum_{i=1}^n \frac{\gamma^2\s_i^4\c_i^2}{\lambda+\gamma\s_i^2}
+ \frac{\alpha l_4}{\frac{1}{n}\sum_{i=1}^n \frac{\s_i^2}{\lambda+\gamma\s_i^2}} +\hat{c}_3^2 + \gamma \sum_{i=1}^n \s_i^2\c_i^2 -\gamma c_2^2 \nonumber \\
 & = &
  - \sum_{i=1}^n \frac{\gamma^2\s_i^4\c_i^2}{\lambda+\gamma\s_i^2}
+ \frac{1}{\hat{\gamma_1}}
\lp
\frac{1}{m}\sum_{j=1}^{m} \lp c_2^2+\sigma^2\q_j\rp -\frac{1}{m}\sum_{j=1}^{m} \frac{1}{\hat{\gamma}_1  + \frac{1}{\os_j^2}} \frac{c_2^2+\sigma^2\q_j}{\os_j^2}
 -\hat{\gamma}_1 \hat{c}_3^2
\rp \nonumber \\
& &
+\hat{c}_3^2 + \gamma \sum_{i=1}^n \s_i^2\c_i^2 -\gamma c_2^2 \nonumber \\
 & = &
  - \sum_{i=1}^n \frac{\gamma^2\s_i^4\c_i^2}{\lambda+\gamma\s_i^2}
+ \frac{1}{\hat{\gamma_1}}
\lp
\frac{1}{m}\sum_{j=1}^{m} \lp c_2^2+\sigma^2\q_j\rp -\frac{1}{m}\sum_{j=1}^{m} \frac{\frac{1}{\os_j^2}}{\hat{\gamma}_1  + \frac{1}{\os_j^2}} \lp c_2^2+\sigma^2\q_j\rp
\rp \nonumber \\
& &
  + \gamma \sum_{i=1}^n \s_i^2\c_i^2 -\gamma c_2^2 \nonumber \\
 & = &
  - \sum_{i=1}^n \frac{\gamma^2\s_i^4\c_i^2}{\lambda+\gamma\s_i^2}
+  \frac{1}{m}\sum_{j=1}^{m} \frac{   c_2^2+\sigma^2\q_j  }{\hat{\gamma}_1  + \frac{1}{\os_j^2}}
   + \gamma \sum_{i=1}^n \s_i^2\c_i^2 -\gamma c_2^2 \nonumber \\
 & = &
  - \sum_{i=1}^n \frac{\gamma^2\s_i^4\c_i^2}{\lambda+\gamma\s_i^2}
+  \frac{1}{m}\sum_{j=1}^{m} \frac{   c_2^2+\sigma^2\q_j  }{\frac{1}{\alpha n}\sum_{i=1}^n \frac{\s_i^2}{\lambda+\gamma\s_i^2} + \frac{1}{\os_j^2}}
   + \gamma \sum_{i=1}^n \s_i^2\c_i^2 -\gamma c_2^2 \nonumber \\
  \label{eq:corrridgeta18a0b4c10d1e0}
\end{eqnarray}
Taking derivatives with respect $c_2$ and $\gamma$ and equalling them to zero, we also have
\begin{eqnarray}
\frac{d \bar{f}_{0,r}(\hat{c}_3,c_2,\hat{\nu}_1,\gamma,\hat{\gamma}_1)}{d c_2} =
\frac{1}{m}\sum_{j=1}^{m}  \frac{2 c_2}{\frac{1}{\os_j^2}+\frac{1}{\alpha n}\sum_{i=1}^n \frac{\s_i^2}{\lambda+\gamma\s_i^2}}   -2\gamma c_2 =0.  \label{eq:corrridgeta18a0b4c11}
\end{eqnarray}
and
\begin{eqnarray}
\frac{d \bar{f}_{0,r}(\hat{c}_3,c_2,\hat{\nu}_1,\gamma,\hat{\gamma}_1)}{d \gamma}
& = &
  - \sum_{i=1}^n \frac{2\gamma\s_i^4\c_i^2}{\lambda+\gamma\s_i^2}
  + \sum_{i=1}^n \frac{\gamma^2\s_i^6\c_i^2}{\lp \lambda +\gamma\s_i^2\rp^2} \nonumber \\
  & &
 + \frac{1}{m}\sum_{j=1}^{m}\frac{\lp c_2^2+\sigma^2\q_j\rp}{\lp \frac{1}{\os_j^2}+ \frac{1}{\alpha n}\sum_{i=1}^n \frac{\s_i^2}{\lambda+\gamma\s_i^2}\rp^2}
 \lp \frac{1}{\alpha n}\sum_{i=1}^n \frac{\s_i^4}{\lp \lambda  +\gamma\s_i^2\rp^2}\rp
 + \sum_{i=1}^n \s_i^2\c_i^2 - c_2^2=0. \nonumber \\ \label{eq:corrridgeta18a0b4c12}
\end{eqnarray}
From (\ref{eq:corrridgeta18a0b4c11}) we find that the optimal $\hat{\gamma}$ satisfies the following equation
\begin{eqnarray}
\hat{\gamma}= \frac{1}{m}\sum_{j=1}^{m}  \frac{1}{\frac{1}{\os_j^2}+\frac{1}{\alpha n}\sum_{i=1}^n \frac{\s_i^2}{\lambda+\hat{\gamma}\s_i^2}}.  \label{eq:corrridgeta18a0b4c13}
\end{eqnarray}
Moreover, we then (for such a $\hat{\gamma}$) have from (\ref{eq:corrridgeta18a0b4c12}) that the optimal $\hat{c}_2$ satisfies
\begin{eqnarray}
   \sum_{i=1}^n \frac{\lambda^2 \s_i^2\c_i^2}{\lp \lambda +\hat{\gamma}\s_i^2\rp^2}
 + \frac{1}{m}\sum_{j=1}^{m}\frac{\lp c_2^2+\sigma^2\q_j\rp}{\lp \frac{1}{\os_j^2}+ \frac{1}{\alpha n}\sum_{i=1}^n \frac{\s_i^2}{\lambda+\gamma\s_i^2}\rp^2}
 \lp \frac{1}{\alpha n}\sum_{i=1}^n \frac{\s_i^4}{\lp \lambda  +\gamma\s_i^2\rp^2}\rp
   - c_2^2=0.  \label{eq:corrridgeta18a0b4c14}
\end{eqnarray}
After setting
\begin{eqnarray}
\bar{a}_{2,r}
&= &
  \frac{1}{n}\sum_{i=1}^n \frac{\s_i^4}{\lp \lambda+\hat{\gamma}\s_i^2\rp^2}\frac{1}{m}\sum_{j=1}^m \frac{1}{\lp \frac{\alpha}{\os_j^2}+ \frac{1}{  n}\sum_{i=1}^n \frac{\s_i^2}{\lambda +\hat{\gamma}\s_i^2}\rp^2} \nonumber  \\
\bar{a}_{3,r}
&= &
  \frac{1}{n}\sum_{i=1}^n \frac{\s_i^4}{\lp \lambda+\hat{\gamma}\s_i^2\rp^2}
  \frac{1}{m}\sum_{j=1}^m \frac{\q_j}{\lp \frac{\alpha}{\os_j^2}+ \frac{1}{  n}\sum_{i=1}^n \frac{\s_i^2}{\lambda +\hat{\gamma}\s_i^2}\rp^2},  \label{eq:corrridgeta18a0b4c15}
\end{eqnarray}
we then from (\ref{eq:corrridgeta18a0b4c14}) obtain
\begin{eqnarray}
\hat{c}_2^2= \frac{  \sum_{i=1}^n \frac{\lambda^2 \s_i^2\c_i^2}{\lp \lambda+\hat{\gamma}\s_i^2\rp^2} + \alpha \sigma^2\bar{a}_{3,r} }{1- \alpha \bar{a}_{2,r}}.  \label{eq:corrridgeta18a0b4c16}
\end{eqnarray}
Combining (\ref{eq:corrridgeta18a0b4c10d1}) and (\ref{eq:corrridgeta18a0b4c10d1e0}) together with (\ref{eq:corrridgeta18a0b4c10}), we find
\begin{eqnarray}
\lim_{n\rightarrow\infty}  \mE_{\g,\h,\v}\bar{f}_{rd,r}(\g,\h,\v)
   & = & \lim_{n\rightarrow\infty} \min_{c_2\geq 0} \max_{\gamma}  f_0(\hat{c}_3,c_2,\hat{\nu}_1,\gamma,\hat{\gamma}_1) \nonumber \\
   & = & \lim_{n\rightarrow\infty} f_0(\hat{c}_3,\hat{c}_2,\hat{\nu}_1,\hat{\gamma},\hat{\gamma}_1) \nonumber \\
   & = &  \lim_{n\rightarrow\infty} \lp    - \sum_{i=1}^n \frac{\hat{\gamma}^2\s_i^4\c_i^2}{\lambda+\hat{\gamma}\s_i^2}
+  \frac{1}{m}\sum_{j=1}^{m} \frac{   \hat{c}_2^2+\sigma^2\q_j  }{\frac{1}{\alpha n}\sum_{i=1}^n \frac{\s_i^2}{\lambda+\hat{\gamma}\s_i^2} + \frac{1}{\os_j^2}}
   + \hat{\gamma} \sum_{i=1}^n \s_i^2\c_i^2 - \hat{\gamma} \hat{c}_2^2 \rp \nonumber \\
   & = &  \lim_{n\rightarrow\infty} \lp    \sum_{i=1}^n \frac{\lambda \hat{\gamma}\s_i^2\c_i^2}{\lambda+\hat{\gamma}\s_i^2}
+  \frac{1}{m}\sum_{j=1}^{m} \frac{   \hat{c}_2^2+\sigma^2\q_j  }{\frac{1}{\alpha n}\sum_{i=1}^n \frac{\s_i^2}{\lambda+\hat{\gamma}\s_i^2} + \frac{1}{\os_j^2}}
   -\hat{\gamma} \hat{c}_2^2 \rp \nonumber \\
    & = &   \lim_{n\rightarrow\infty}  \sum_{i=1}^n \frac{\lambda \hat{\gamma}\s_i^2\c_i^2}{\lambda +\hat{\gamma}\s_i^2}
+ \sigma^2 \frac{1}{m}\sum_{j=1}^{m} \frac{ \q_j  }{\frac{1}{\alpha n}\sum_{i=1}^n \frac{\s_i^2}{\lambda+\hat{\gamma}\s_i^2} + \frac{1}{\os_j^2}}, \label{eq:corrridgeta18a10}
\end{eqnarray}
where $\hat{c}_2$ and $\hat{\gamma}$ are given through (\ref{eq:corrridgeta18a0b4c13}), (\ref{eq:corrridgeta18a0b4c15}), and (\ref{eq:corrridgeta18a0b4c16}). Taking the very same $\hat{c}_2$ and $\hat{\gamma}$, we then from (\ref{eq:corrridgeta18a0b4c10d0e3}) and (\ref{eq:corrridgeta18a0b4c8}) find the optimal $c_3$ and $\nu_1$
\begin{eqnarray}
\hat{c}_3^2
 =  \frac{1}{m}\sum_{j=1}^{m} \frac{1}{\lp \frac{1}{\alpha n}\sum_{i=1}^n \frac{\s_i^2}{\lambda+\hat{\gamma}\s_i^2}  + \frac{1}{\os_j^2}\rp^2} \frac{\hat{c}_2^2+\sigma^2\q_j}{\os_j^2}.
  \label{eq:corrridgeta18a10b0}
\end{eqnarray}
and
\begin{eqnarray}
\hat{\nu}_1 & = &
 \frac{2\sqrt{\alpha}\sqrt{l_4}}{\frac{1}{n}\sum_{i=1}^n \frac{\s_i^2}{\lambda+\hat{\gamma}\s_i^2}}\nonumber \\
 & = &
 \frac{2\sqrt{\alpha}}{\frac{1}{n}\sum_{i=1}^n \frac{\s_i^2}{\lambda+\hat{\gamma}\s_i^2}}
\sqrt{\frac{1}{m}\sum_{j=1}^{m} \lp \hat{c}_2^2+\sigma^2\q_j\rp -\frac{1}{m}\sum_{j=1}^{m} \frac{1}{\hat{\gamma}_1  + \frac{1}{\os_j^2}} \frac{\hat{c}_2^2+\sigma^2\q_j}{\os_j^2}
 -\hat{\gamma}_1 \hat{c}_3^2
} \nonumber \\
& = &
 \frac{2\sqrt{\alpha}}{\frac{1}{n}\sum_{i=1}^n \frac{\s_i^2}{\lambda+\hat{\gamma}\s_i^2}}
\sqrt{\hat{\gamma}_1\frac{1}{m}\sum_{j=1}^{m} \frac{1}{\hat{\gamma}_1  + \frac{1}{\os_j^2}} \frac{\hat{c}_2^2+\sigma^2\q_j}{\os_j^2}
 -\hat{\gamma}_1 \hat{c}_3^2
} \nonumber \\
& = &
 \frac{2\sqrt{\alpha}}{\frac{1}{n}\sum_{i=1}^n \frac{\s_i^2}{\lambda+\hat{\gamma}\s_i^2}}
\sqrt{\hat{\gamma}_1\frac{1}{m}\sum_{j=1}^{m} \frac{1}{\hat{\gamma}_1  + \frac{1}{\os_j^2}} \frac{\hat{c}_2^2+\sigma^2\q_j}{\os_j^2}
 -\hat{\gamma}_1
 \frac{1}{m}\sum_{j=1}^{m} \frac{1}{\lp \hat{\gamma}_1  + \frac{1}{\os_j^2}\rp^2} \frac{\hat{c}_2^2+\sigma^2\q_j}{\os_j^2}
} \nonumber \\
& = &
 \frac{2\sqrt{\alpha}\hat{\gamma}_1}{\frac{1}{n}\sum_{i=1}^n \frac{\s_i^2}{\lambda+\hat{\gamma}\s_i^2}}
\sqrt{
 \frac{1}{m}\sum_{j=1}^{m} \frac{\hat{c}_2^2+\sigma^2\q_j}{\lp \hat{\gamma}_1  + \frac{1}{\os_j^2}\rp^2}
 } \nonumber \\
& = &
 \frac{2}{\sqrt{\alpha}}
\sqrt{
 \frac{1}{m}\sum_{j=1}^{m} \frac{\hat{c}_2^2+\sigma^2\q_j}{\lp \hat{\gamma}_1  + \frac{1}{\os_j^2}\rp^2}
 } \nonumber \\
& = &
 \frac{2}{\sqrt{\alpha}}
\sqrt{
 \frac{1}{m}\sum_{j=1}^{m} \frac{\hat{c}_2^2+\sigma^2\q_j}{\lp \frac{1}{\alpha n}\sum_{i=1}^n \frac{\s_i^2}{\lambda+\hat{\gamma}\s_i^2}  + \frac{1}{\os_j^2}\rp^2}}. \label{eq:corrridgeta18a10b1}
\end{eqnarray}
We also have, analogously to (\ref{eq:ridgeta18a0b4c17d1e0}), for the optimal $\hat{c}_4^2=\lim_{n\rightarrow \infty} \mE_{\g,\h,\v}\|\hat{\x}\|_2^2$
\begin{eqnarray}
\hat{c}_4^2 & = & \lim_{n\rightarrow \infty} \mE_{\g,\h,\v}\| \hat{\x}\|_2^2 \nonumber \\
 & = & \lim_{n\rightarrow \infty} \mE_{\g,\h,\v}\left \| \lp \lambda I +\hat{\gamma}\Sigma\Sigma\rp^{-1}\lp -\frac{\hat{\nu}_1}{2\sqrt{n}} \Sigma\h +\gamma\Sigma\Sigma\c\rp \right \|_2^2 \nonumber \\
 & = &
\lim_{n\rightarrow \infty}  \frac{\hat{\nu}_1^2}{4n}\sum_{i=1}^n\frac{\s_i^2}{\lp \lambda+\hat{\gamma}\s_i^2\rp^2}
+
\lim_{n\rightarrow \infty}  \sum_{i=1}^n\frac{\hat{\gamma}^2\s_i^4\c_i^2}{\lp \lambda+\hat{\gamma}\s_i^2\rp^2}. \label{eq:corrridgeta18a0b4c17d1e0}
\end{eqnarray}

The above discussion in summarized in the following theorem.
\begin{theorem}(Characterization of Ridge estimator -- Row-correlated $X$) For a given fixed real positive number $\lambda$, assume the setup of Lemma \ref{lemma:corrridgelemma1} and Theorem \ref{thm:corrridgethm1} with  $\bar{\xi}_{rr}(\lambda)$ as in (\ref{eq:corrridgerandlincons1a1}) or (\ref{eq:corrridgerandlincons2}) and $\bar{\xi}_{rd,r}$ as in (\ref{eq:ridgeta16}). Let $\hat{\beta}=\bar{\beta}_{rr}(\lambda)$ where $\bar{\beta}_{rr}(\lambda)$ is as in
(\ref{eq:corrmodel014}). Also, let the Ridge estimator prediction risk, $R(\bar{\beta},\hat{\beta})=R(\bar{\beta},\bar{\beta}_{rr}(\lambda))$, be defined via (\ref{eq:model020}). Assuming a large $n$ linear regime with $\alpha=\lim_{n\rightarrow\infty}\frac{m}{n}$, let unique $\hat{\gamma}$ be such that
\begin{eqnarray}
\hat{\gamma}= \lim_{m\rightarrow\infty}\frac{1}{m}\sum_{j=1}^{m}  \frac{1}{\frac{1}{\os_j^2}+\lim_{n\rightarrow\infty}\frac{1}{\alpha n}\sum_{i=1}^n \frac{\s_i^2}{\lambda+\hat{\gamma}\s_i^2}} .  \label{eq:corrridgethmaaeq1}
\end{eqnarray}
Set
\begin{eqnarray}
\bar{a}_{2,r}
&= &
\lim_{n\rightarrow\infty}  \frac{1}{n}\sum_{i=1}^n \frac{\s_i^4}{\lp \lambda+\hat{\gamma}\s_i^2\rp^2}
\lim_{m\rightarrow\infty}\frac{1}{m}\sum_{j=1}^m \frac{1}{\lp \frac{\alpha}{\os_j^2}+ \frac{1}{  n}\sum_{i=1}^n \frac{\s_i^2}{\lambda +\hat{\gamma}\s_i^2}\rp^2} \nonumber  \\
\bar{a}_{3,r}
&= &
\lim_{n\rightarrow\infty}  \frac{1}{n}\sum_{i=1}^n \frac{\s_i^4}{\lp \lambda+\hat{\gamma}\s_i^2\rp^2}
\lim_{m\rightarrow\infty}  \frac{1}{m}\sum_{j=1}^m \frac{\q_j}{\lp \frac{\alpha}{\os_j^2}+ \frac{1}{  n}\sum_{i=1}^n \frac{\s_i^2}{\lambda +\hat{\gamma}\s_i^2}\rp^2}.  \label{eq:corrridgethmaaeq2}
\end{eqnarray}
One then has \vspace{-.0in}
\begin{eqnarray}
 \lim_{n\rightarrow\infty} \mE_{Z,\v}\bar{\xi}_{rr}(\lambda)
 &= &\bar{\xi}_{rp,r}=  \bar{\xi}_{rd,r}  = \lim_{n\rightarrow\infty} \mE_{\g,\h,\v} \bar{f}_{rd,r}(\g,\h,\v) \nonumber \\
& = &
\lim_{n\rightarrow\infty}  \sum_{i=1}^n \frac{\lambda \hat{\gamma}\s_i^2\c_i^2}{\lambda+\hat{\gamma}\s_i^2}
+
\sigma^2\lim_{m\rightarrow\infty}  \frac{1}{m}\sum_{j=1}^{m} \frac{ \q_j  }{\frac{1}{\os_j^2}+ \lim_{n\rightarrow\infty} \frac{1}{\alpha n}\sum_{i=1}^n \frac{\s_i^2}{\lambda+\hat{\gamma}\s_i^2} }
 \nonumber \\
 \lim_{n\rightarrow\infty} \mE_{Z,\v} R(\bar{\beta},\beta_{rr}(\lambda))
& = &
\lim_{n\rightarrow\infty} \hat{c}_2^2
  =
\frac{\lim_{n\rightarrow\infty}  \sum_{i=1}^n \frac{\lambda^2 \s_i^2\c_i^2}{\lp \lambda+\hat{\gamma}\s_i^2\rp^2} + \alpha \sigma^2\bar{a}_{3,r} }{1- \alpha \bar{a}_{2,r}} \nonumber \\
 & & \lim_{n\rightarrow\infty} \hat{\nu}_1
  =
 \frac{2}{\sqrt{\alpha}}
\sqrt{
\lim_{m\rightarrow\infty}  \frac{1}{m}\sum_{j=1}^{m} \frac{\hat{c}_2^2+\sigma^2\q_j}{\lp \lim_{n\rightarrow\infty} \frac{1}{\alpha n}\sum_{i=1}^n \frac{\s_i^2}{\lambda+\hat{\gamma}\s_i^2}  + \frac{1}{\os_j^2}\rp^2}}
    \nonumber \\
 \lim_{n\rightarrow\infty} \mE_{Z,\v} \|\y-X\bar{\beta}_{rr}(\lambda)\|_2^2
& = &
\lim_{n\rightarrow\infty} \hat{c}_3^2
  = \lim_{m\rightarrow\infty}\frac{1}{m}\sum_{j=1}^{m} \frac{1}{\lp \lim_{n\rightarrow\infty} \frac{1}{\alpha n}\sum_{i=1}^n \frac{\s_i^2}{\lambda+\hat{\gamma}\s_i^2}  + \frac{1}{\os_j^2}\rp^2} \frac{\hat{c}_2^2+\sigma^2\q_j}{\os_j^2} \nonumber \\
 \lim_{n\rightarrow\infty} \mE_{Z,\v} \|\bar{\beta}_{rr}(\lambda)\|_2^2
& = &
\lim_{n\rightarrow\infty} \hat{c}_4^2
 =
 \lim_{n\rightarrow \infty}  \frac{\hat{\nu}_1^2}{4n}\sum_{i=1}^n\frac{\s_i^2}{\lp \lambda+\hat{\gamma}\s_i^2\rp^2}
+
\lim_{n\rightarrow \infty}  \sum_{i=1}^n\frac{\hat{\gamma}^2\s_i^4\c_i^2}{\lp \lambda+\hat{\gamma}\s_i^2\rp^2},
\label{eq:corrridgethmaaeq3}
\end{eqnarray}
and for any fixed $\epsilon>0$
\begin{eqnarray}
 \lim_{n\rightarrow\infty}\mP_{Z,\v} \lp  (1-\epsilon)\mE_{Z,\v}\bar{\xi}_{rr}(\lambda) \leq  \bar{\xi}_{rr}(\lambda) \leq (1+\epsilon)\mE_{Z,\v}\bar{\xi}_{rr}(\lambda)\rp
& \longrightarrow & 1 \nonumber \\
 \lim_{n\rightarrow\infty}\mP_{Z,\v} \lp  (1-\epsilon)\mE_{Z,\v} R(\bar{\beta},\bar{\beta}_{rr}(\lambda))  \leq  R(\bar{\beta},\bar{\beta}_{rr}(\lambda)) \leq (1+\epsilon)\mE_{Z,\v}R(\bar{\beta},\bar{\beta}_{rr}(\lambda))\rp
& \longrightarrow & 1  \nonumber \\
 \lim_{n\rightarrow\infty}\mP_{\g,\h,\v} \lp  (1-\epsilon)\hat{\nu}_1 \leq  \nu_1 \leq (1+\epsilon)\hat{\nu}_1\rp
& \longrightarrow & 1 \nonumber \\
 \lim_{n\rightarrow\infty}\mP_{Z,\v} \lp  (1-\epsilon)\mE_{Z,\v}\|\y-X\bar{\beta}_{rr}(\lambda)\|_2^2 \leq \|\y-X\bar{\beta}_{rr}(\lambda)\|_2^2  \leq (1+\epsilon)\mE_{Z,\v}\|\y-X\bar{\beta}_{rr}(\lambda)\|_2^2 \rp
& \longrightarrow & 1\nonumber \\
 \lim_{n\rightarrow\infty}\mP_{Z,\v} \lp  (1-\epsilon)\mE_{Z,\v}\|\bar{\beta}_{rr}(\lambda)\|_2^2 \leq \|\bar{\beta}_{rr}(\lambda)\|_2^2  \leq (1+\epsilon)\mE_{Z,\v}\|\bar{\beta}_{rr}(\lambda)\|_2^2 \rp
& \longrightarrow & 1. \nonumber \\\label{eq:corrridgethmaaeq4}
\end{eqnarray}
\label{thm:corrridgethm2}
\end{theorem}\vspace{-.17in}
\begin{proof}
Follows from the above discussion, trivial concentrations of $\bar{f}_{rd,r}(\g,\h,\v)$, $c_3$, $c_2$, and $\nu_1$, and utilization of literally the asme arguments as in the proof of Theorems \ref{thm:thm2} and \ref{thm:ridgethm2}.
\end{proof}

\vspace{.1in}
\noindent \underline{ 4) \textbf{\emph{Double checking strong random duality:}}} As was the case when we considered the uncorrelated rows of $X$, due to the underlying convexity one has that the reversal arguments of \cite{StojnicGorEx10,StojnicRegRndDlt10} and the strong random duality  are in place as well.

\subsubsection{Additional observations}
\label{sec:addobs}

The corresponding LS estimator results are obtained by taking $\lambda=0$ in the above theorem. The expressions do simplify a bit but not as they do in case of uncorrelated rows of $X$. We therefore skip stating them as a separate corollary. We also recall from the very beginning of this section, that the results for the GLS interpolator when rows of $X$ are correlated are identical to the ones obtained in Theorem \ref{thm:thm2} with $\bar{\sigma}$ as defined in (\ref{eq:corrmodel010a0b0}). Moreover, all results obtained for the GLS interpolator (with rows of $X$ correlated or uncorrelated) remain unaltered even if the rank of $A$ or $\Sigma$ is not full. In, fact, assuming that the $\mbox{rank}(A)=\mbox{rank}(\Sigma)=k$, where
\begin{eqnarray}
\kappa\triangleq \lim_{n\rightarrow\infty} \frac{k}{n}.\label{eq:lrklsridgethmaaeq4}
 \end{eqnarray}
then all GLS related results hold as long as $\kappa>\alpha$. When $\kappa<\alpha$, interpolation is trivially impossible. All results related to the Ridge estimator hold for any rank of $A$ or $\Sigma$.

\subsection{Numerical results}
\label{sec:numrescorr}

As earlier, we supplement the above theoretical results with the ones obtained through numerical simulations. We take for the covariance matrices
\begin{eqnarray}
\overline{\overline{A}} & = & \cA (q_y) \nonumber \\
A & = &\cA (q) \nonumber \\
\overline{A} &= & \cA (q_v).
  \label{eq:numuncorr3}
\end{eqnarray}
In Figure \ref{fig:fig2}, we take $q_y=0.7,q=0.5,q_v=0.4$, and $n=600$. As earlier, the full curves are the theoretical predictions and the dots are the corresponding simulated values. We again observe an excellent agreement between the theoretical predictions and simulations results.

\begin{figure}[h]
\centering
\centerline{\includegraphics[width=1\linewidth]{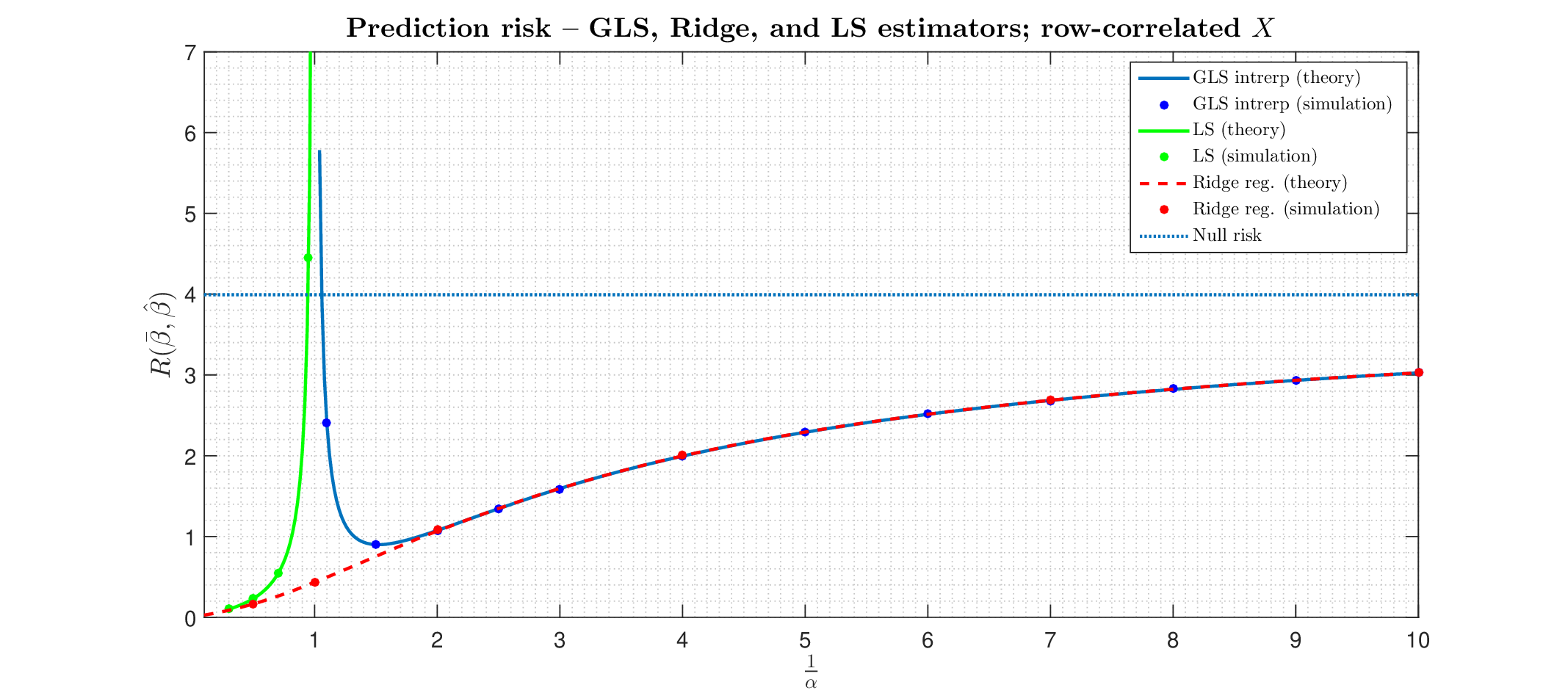}}
\caption{Prediction risk -- all three estimators (GLS, Ridge, and LS); row-correlated features $X$; Covariance matrices are: $\overline{\overline{A}}=\cA(q_y)$, $A=\cA(q)$, and $\overline{A}=\cA(q_v)$; $q_y=0.7,q=0.5,q_v=0.4$.}
\label{fig:fig2}
\end{figure}

In Figure \ref{fig:fig3}, we show what kind of effect intra-sample correlation can have on the prediction risk. All parameters take the same values as in Figure \ref{fig:fig2}  except $\alpha=0.5$ and  $qy$ which is varied between zero and one. Also to adjust that no correlation correspond to unitary covariance matrices, all ${\mathcal A}$'s are scaled by 2.  Keeping in mind the adopted convention that the full curves are the theoretical predictions and the dots are the corresponding simulated values, we again observe an excellent agreement between the theoretical predictions and simulations results. In low intra-sample correlated regime, we also observe that the risk is not always increasing as the correlation increases.

\begin{figure}[h]
\centering
\centerline{\includegraphics[width=1\linewidth]{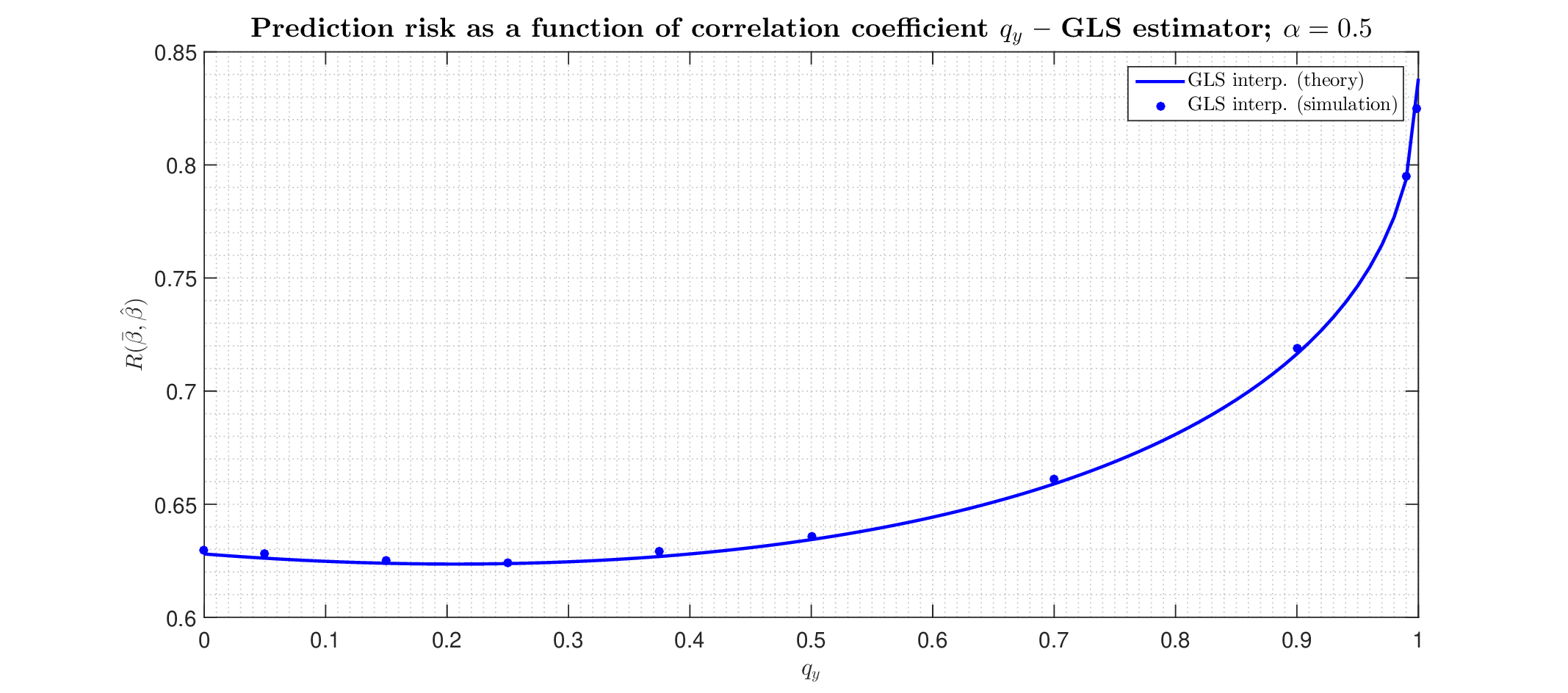}}
\caption{Prediction risk as a function of intra-sample correlation -- GLS;  Covariance matrices are: $\overline{\overline{A}}=\frac{1}{2}\cA(q_y)$, $A=\frac{1}{2}\cA(q)$, and $\overline{A}=\frac{1}{2}\cA(q_v)$; $q=0.5,q_v=0.4$; $q_y\in[0,1]$.}
\label{fig:fig3}
\end{figure}

\section{Conclusion}
\label{sec:conc}

We studied classical linear regression models and three of the most well known estimators associated with them: (i) minimum norm interpolators (generalized least squares (GLS)); (ii) ordinary least squares (LS);  and (ii) ridge estimators. A statistical context with Gaussian feature matrix and noise is considered. In addition to inter-feature (cross-sectional) correlations, the intra-sample ones for both the features and noise are considered as well. The (random) optimization programs that produce the above three estimators are statistically analyzed through a utilization of a powerful mathematical engine called \emph{Random Duality Theory} (RDT). Precise closed form characterizations of all optimizing quantities associated with all three optimization programs are obtained. Among these quantities, we particularly focused on the \emph{prediction risk} (generalization/testing error) and observed that it exhibits the well known non-monotonic (so-called double-descent) behavior as the over-parametrization increases. As our results are relatively simple closed forms, they explicitly show how the risk depends on all key model parameters, including the problem dimensions ratios and covariance matrices. When the intra-sample (or time-series) correlations are absent, our results precisely match the corresponding ones obtained via spectral methods in \cite{HMRT22,Dicker16,DobWag18,BHX20}.

The presented methodology is fairly generic and allows for a variety of extensions and/or generalizations. For example all those mentioned in companion paper \cite{Stojnicridgefrm24} apply here as well. Imperfections in observable features, their potential absence, and adversarial responses are only a few among many popular (yet practically very relevant) scenarios that received a strong attention in recent years in both machine learning and statistics communities. For each of these scenarios,various different types of estimators could be of interest as well. All these extensions can be handled through the program presented here without any further conceptually novel insights. As is usually the case within the context of RDT, the concrete technical realizations of such extensions are problem specific and as such need to be tailored for each of them separately. We will present these adjustments together with concrete results that they produce in separate papers.

As  pointed out in \cite{StojnicRegRndDlt10}, the RDT considerations do not require the standard Gaussianity. We have utilized it, however, as it allowed for the presentation to look neater. The needed conceptual adjustments, as one deviates from the Gaussianity, are minimal (the underlying writing, however, becomes a bit lengthy and cumbersome). An application of the Lindeberg central limit theorem variant (see, e.g.,  \cite{Lindeberg22}) is in our view the most elegant way to extend the RDT results to various different statistics and the approach of \cite{Chatterjee06} is particularly convenient in that regard.

\begin{singlespace}
\bibliographystyle{plain}
\bibliography{nflgscompyxRefs}

\begin{thebibliography}{10}

\bibitem{ASS20}
M.~S. Advani, A.~M. Saxe, and H.~Sompolinsky.
\newblock High-dimensional dynamics of generalization error in neural networks.
\newblock {\em Neural Networks}, 132:428--446, 2020.

\bibitem{AliKT19}
A.~Ali, J.~Z. Kolter, and R.~J. Tibshirani.
\newblock A continuous-time view of early stopping for least squares
  regression.
\newblock In {\em The 22nd International Conference on Artificial Intelligence
  and Statistics, {AISTATS} 2019, 16-18 April 2019, Naha, Okinawa, Japan},
  volume~89 of {\em Proceedings of Machine Learning Research}, pages
  1370--1378. {PMLR}.

\bibitem{ADHLW19}
S.~Arora, S.~S. Du, W.~Hu, Z.~Li, and R.~Wang.
\newblock Fine-grained analysis of optimization and generalization for
  overparameterized two-layer neural networks.
\newblock In {\em Proceedings of the 36th International Conference on Machine
  Learning, {ICML} 2019, 9-15 June 2019, Long Beach, California, {USA}},
  volume~97 of {\em Proceedings of Machine Learning Research}, pages 322--332.
  {PMLR}, 2019.

\bibitem{BLT20}
P.~L. Bartlett, P.~M. Lugosi, and A.~Tsigler.
\newblock Benign overfitting in linear regression.
\newblock {\em Proc. Natl. Acad. Sci. USA}, 117:30063--30070, 2020.

\bibitem{BHMM19}
M.~Belkin, D.~Hsu, S.~Ma, and S.~Mandal.
\newblock Reconciling modern machine-learning practice and the classical
  bias-variance trade-off.
\newblock {\em Proc. Natl. Acad. Sci. USA}, 116:15849--15854, 2019.

\bibitem{BHX20}
M.~Belkin, D.~Hsu, and J.~Xu.
\newblock Two models of double descent for weak features.
\newblock {\em SIAM Journal on Mathematics of Data Science}, 2(4):1167--1180,
  2020.

\bibitem{BelkinHM18}
M.~Belkin, D.~J. Hsu, and P.~Mitra.
\newblock Overfitting or perfect fitting? risk bounds for classification and
  regression rules that interpolate.
\newblock In {\em Advances in Neural Information Processing Systems 31: Annual
  Conference on Neural Information Processing Systems 2018, NeurIPS 2018,
  December 3-8, 2018, Montr{\'{e}}al, Canada}, pages 2306--2317, 2018.

\bibitem{BelkinMM18}
M.~Belkin, S.~Ma, and S.~Mandal.
\newblock To understand deep learning we need to understand kernel learning.
\newblock In {\em Proceedings of the 35th International Conference on Machine
  Learning, {ICML} 2018, Stockholmsm{\"{a}}ssan, Stockholm, Sweden, July 10-15,
  2018}, volume~80 of {\em Proceedings of Machine Learning Research}, pages
  540--548. {PMLR}, 2018.

\bibitem{BRT19}
M.~Belkin, A.~Rakhlin, and A.~Tsybakov.
\newblock Does data interpolation contradict statistical optimality?
\newblock In Kamalika Chaudhuri and Masashi Sugiyama, editors, {\em Proceedings
  of the Twenty-Second International Conference on Artificial Intelligence and
  Statistics}, volume~89 of {\em Proceedings of Machine Learning Research},
  pages 1611--1619. PMLR, 16--18 Apr 2019.

\bibitem{BBSMW21}
X.~Bing, F.~Bunea, S.~Strimas-Mackey, and M.~Wegkamp.
\newblock Prediction under latent factor regression: Adaptive pcr,
  interpolating predictors and beyond.
\newblock {\em Journal of Machine Learning Research}, 22(177):1--50, 2021.

\bibitem{BF83}
L.~Breiman and D.~Freedman.
\newblock How many variables should be entered in a regression equation?
\newblock {\em Journal of the American Statistical Association},
  78(381):131--136, 1983.

\bibitem{BSMW22}
F.~Bunea, S.~Strimas-Mackey, and M.~H. Wegkamp.
\newblock Interpolating predictors in high-dimensional factor regression.
\newblock {\em J. Mach. Learn. Res.}, 23:10:1--10:60, 2022.

\bibitem{Chatterjee06}
S.~Chatterjee.
\newblock A generalization of the {L}indenberg principle.
\newblock {\em The Annals of Probability}, 34(6):2061--2076.

\bibitem{ChizatB18}
L.~Chizat and F.~R. Bach.
\newblock On the global convergence of gradient descent for over-parameterized
  models using optimal transport.
\newblock In {\em Advances in Neural Information Processing Systems 31, NeurIPS
  2018, December 3-8, 2018, Montr{\'{e}}al, Canada}, pages 3040--3050, 2018.

\bibitem{CuckerS02}
F.~Cucker and S.~Smale.
\newblock Best choices for regularization parameters in learning theory: On the
  bias-variance problem.
\newblock {\em Found. Comput. Math.}, 2(4):413--428, 2002.

\bibitem{Dicker16}
L.~H. Dicker.
\newblock Ridge regression and asymptotic minimax estimation over spheres of
  growing dimension.
\newblock {\em Bernoulli}, 22:1--37, 2016.

\bibitem{DobWag18}
E.~Dobriban and S.~Wager.
\newblock High-dimensional asymptotics of prediction: Ridge regression and
  classification.
\newblock {\em Annals of Statistics}, 46:246--279, 2018.

\bibitem{DuLL0Z19}
S.~S. Du, J.~D. Lee, H.~Li, L.~Wang, and X.~Zhai.
\newblock Gradient descent finds global minima of deep neural networks.
\newblock In {\em Proceedings of the 36th International Conference on Machine
  Learning, {ICML} 2019, 9-15 June 2019, Long Beach, California, {USA}},
  volume~97 of {\em Proceedings of Machine Learning Research}, pages
  1675--1685. {PMLR}, 2019.

\bibitem{DuZPS19}
S.~S. Du, X.~Zhai, B.~P{\'{o}}czos, and A.~Singh.
\newblock Gradient descent provably optimizes over-parameterized neural
  networks.
\newblock In {\em 7th International Conference on Learning Representations,
  {ICLR} 2019, New Orleans, LA, USA, May 6-9, 2019}. OpenReview.net, 2019.

\bibitem{Farrell11}
B.~Farrell.
\newblock Limiting empirical singular value distribution of restrictions of
  discrete fourier transform matrices.
\newblock {\em Journal of Fourier Analysis and Applications}, 17(4):733--753,
  2011.

\bibitem{Gordon88}
Y.~Gordon.
\newblock On {M}ilman's inequality and random subspaces which escape through a
  mesh in ${R}^n$.
\newblock {\em Geometric Aspect of of functional analysis, Isr. Semin. 1986-87,
  Lect. Notes Math}, 1317, 1988.

\bibitem{GWBNS17}
S.~Gunasekar, B.~E. Woodworth, S.~Bhojanapalli, B.~Neyshabur, and N.~Srebro.
\newblock Implicit regularization in matrix factorization.
\newblock In {\em Advances in Neural Information Processing Systems 30: Annual
  Conference on Neural Information Processing Systems 2017, December 4-9, 2017,
  Long Beach, CA, {USA}}, pages 6151--6159, 2017.

\bibitem{HastieFT01}
T.~Hastie, J.~H. Friedman, and R.~Tibshirani.
\newblock {\em The Elements of Statistical Learning: Data Mining, Inference,
  and Prediction}.
\newblock Springer Series in Statistics. Springer, 2001.

\bibitem{HMRT22}
T.~Hastie, A.~Montanari, S.~Rosset, and R.~J. Tibshirani.
\newblock Surprises in high-dimensional ridgeless least squares interpolation.
\newblock {\em Annals of Staristics}, 50(2):949--986, 2022.

\bibitem{HastieTF09}
T.~Hastie, R.~Tibshirani, and Jerome~H. Friedman.
\newblock {\em The Elements of Statistical Learning: Data Mining, Inference,
  and Prediction, 2nd Edition}.
\newblock Springer Series in Statistics. Springer, 2009.

\bibitem{HuL23}
H.~Hu and Y.~M. Lu.
\newblock Universality laws for high-dimensional learning with random features.
\newblock {\em {IEEE} Trans. Inf. Theory}, 69(3):1932--1964, 2023.

\bibitem{JGH18}
A.~Jacot, F.~Gabriel, and C.~Honger.
\newblock Neural tangent kernel: Convergence and generalization in neural
  networks.
\newblock In {\em Advances in Neural Information Processing Systems 31, NeurIPS
  2018, December 3-8, 2018, Montr{\'{e}}al, Canada}, 2018.

\bibitem{LP11}
O.~Ledoit and S.~Peche.
\newblock Eigenvectors of some large sample covariance matrix ensembles.
\newblock {\em Probab. Theory Related Fields}, 155:233--264, 2011.

\bibitem{LXSBNSP20}
J.~Lee, L.~Xiao, S.~S. Schoenholz, Y.~Bahri, Y.~R. Novak, J.~Sohl-Dickstein,
  and J.~Pennington.
\newblock Wide neural networks of any depth evolve as linear models under
  gradient descent.
\newblock {\em J. Stat. Mech. Theory Exp.}, 12:124002, 2020.

\bibitem{LiL18a}
Y.~Li and Y.~Liang.
\newblock Learning overparameterized neural networks via stochastic gradient
  descent on structured data.
\newblock In {\em Advances in Neural Information Processing Systems 31: Annual
  Conference on Neural Information Processing Systems 2018, NeurIPS 2018,
  December 3-8, 2018, Montr{\'{e}}al, Canada}, pages 8168--8177, 2018.

\bibitem{LiR21}
T.~Liang and A.~Rakhlin.
\newblock Just interpolate: Kernel {``Ridgeless''} regression can generalize.
\newblock {\em Annals of Staristics}, 48(3):1329--1347, 2020.

\bibitem{LRZ20}
T.~Liang, A.~Rakhlin, and X.~Zhai.
\newblock On the multiple descent of minimum-norm interpolants and restricted
  lower isometry of kernels.
\newblock In {\em Conference on Learning Theory, {COLT} 2020, 9-12 July 2020},
  volume 125 of {\em Proceedings of Machine Learning Research}, pages
  2683--2711. {PMLR}, 2020.

\bibitem{Lindeberg22}
J.~W. Lindeberg.
\newblock Eine neue herleitung des exponentialgesetzes in der
  wahrscheinlichkeitsrechnung.
\newblock {\em Math. Z.}, 15:211--225, 1922.

\bibitem{MeiMon22}
S.~Mei and A.~Montanari.
\newblock The generalization error of random features regression: Precise
  asymptotics and the double descent curve.
\newblock {\em Communiactions on Pure and Applied Mathematics}, 75(4):667--766,
  2022.

\bibitem{MRSY19}
A.~Montanari, E.~Ruan, Y.~Sohn, and J.~Yan.
\newblock The generalization error of max-margin linear classifiers:
  High-dimensional asymptotics in the overparametrized regime.
\newblock 2019.
\newblock available online at
  {\small\bl{\url{http://arxiv.org/abs/1911.01544}}}.

\bibitem{MuthukumarVSS20}
V.~Muthukumar, K.~Vodrahalli, V.~Subramanian, and A.~Sahai.
\newblock Harmless interpolation of noisy data in regression.
\newblock {\em {IEEE} J. Sel. Areas Inf. Theory}, 1(1):67--83, 2020.

\bibitem{NeyshaburTS14}
B.~Neyshabur, R.~Tomioka, and N.~Srebro.
\newblock In search of the real inductive bias: On the role of implicit
  regularization in deep learning.
\newblock In {\em 3rd International Conference on Learning Representations,
  {ICLR} 2015, San Diego, CA, USA, May 7-9, 2015, Workshop Track Proceedings},
  2015.

\bibitem{OymSol19}
S.~Oymak and M.~Soltanolkotabi.
\newblock Towards moderate overparameterization: global convergence guarantees
  for training shallow neural networks.
\newblock 2019.
\newblock available online at
  {\small\bl{\url{http://arxiv.org/abs/1902.04674}}}.

\bibitem{RahimiR07}
A.~Rahimi and B.~Recht.
\newblock Random features for large-scale kernel machines.
\newblock In {\em Advances in Neural Information Processing Systems 20,
  Proceedings of the Twenty-First Annual Conference on Neural Information
  Processing Systems, Vancouver, British Columbia, Canada, December 3-6, 2007},
  pages 1177--1184. Curran Associates, Inc., 2007.

\bibitem{RMR20}
D.~Richards, J.~Mourtada, and L.~Rosasco.
\newblock Asymptotics of ridge(less) regression under general source condition.
\newblock In {\em International Conference on Artificial Intelligence and
  Statistics}, 2020.

\bibitem{SGDSBW19}
S.~Spigler, M.~Geiger, S.~D{'}Ascoli, L.~Sagun, G.~Birolli, and M.~Wyart.
\newblock A jamming transition from under- to over-parametrization affects
  generalization in deep learning.
\newblock {\em J. Phys. A}, 52:474001, 2019.

\bibitem{StojnicGenLasso10}
M.~Stojnic.
\newblock A framework for perfromance characterization of \emph{LASSO}
  algortihms.
\newblock available online at \bl{\url{http://arxiv.org/abs/1303.7291}}.

\bibitem{StojnicGenSocp10}
M.~Stojnic.
\newblock A performance analysis framework for \emph{SOCP} algorithms in noisy
  compressed sensing.
\newblock available online at \bl{\url{http://arxiv.org/abs/1304.0002}}.

\bibitem{StojnicPrDepSocp10}
M.~Stojnic.
\newblock A problem dependent analysis of \emph{SOCP} algorithms in noisy
  compressed sensing.
\newblock available online at \bl{\url{http://arxiv.org/abs/1304.0480}}.

\bibitem{StojnicUpper10}
M.~Stojnic.
\newblock Upper-bounding $\ell_1$-optimization weak thresholds.
\newblock available online at \bl{\url{http://arxiv.org/abs/1303.7289}}.

\bibitem{StojnicCSetam09}
M.~Stojnic.
\newblock Various thresholds for $\ell_1$-optimization in compressed sensing.
\newblock available online at \bl{\url{http://arxiv.org/abs/0907.3666}}.

\bibitem{StojnicICASSP10var}
M.~Stojnic.
\newblock $\ell_1$ optimization and its various thresholds in compressed
  sensing.
\newblock {\em ICASSP, IEEE International Conference on Acoustics, Signal and
  Speech Processing}, pages 3910--3913, 14-19 March 2010.
\newblock Dallas, TX.

\bibitem{StojnicISIT2010binary}
M.~Stojnic.
\newblock Recovery thresholds for $\ell_1$ optimization in binary compressed
  sensing.
\newblock {\em ISIT, IEEE International Symposium on Information Theory}, pages
  1593 -- 1597, 13-18 June 2010.
\newblock Austin, TX.

\bibitem{StojnicGardGen13}
M.~Stojnic.
\newblock Another look at the {G}ardner problem.
\newblock 2013.
\newblock available online at \bl{\url{http://arxiv.org/abs/1306.3979}}.

\bibitem{StojnicDiscPercp13}
M.~Stojnic.
\newblock Discrete perceptrons.
\newblock 2013.
\newblock available online at \bl{\url{http://arxiv.org/abs/1303.4375}}.

\bibitem{StojnicGorEx10}
M.~Stojnic.
\newblock Meshes that trap random subspaces.
\newblock 2013.
\newblock available online at \bl{\url{http://arxiv.org/abs/1304.0003}}.

\bibitem{StojnicRegRndDlt10}
M.~Stojnic.
\newblock Regularly random duality.
\newblock 2013.
\newblock available online at \bl{\url{http://arxiv.org/abs/1303.7295}}.

\bibitem{Stojnicgscomp16}
M.~Stojnic.
\newblock Generic and lifted probabilistic comparisons -- max replaces minmax.
\newblock 2016.
\newblock available online at \bl{\url{http://arxiv.org/abs/1612.08506}}.

\bibitem{Stojnicridgefrm24}
M.~Stojnic.
\newblock Ridge interpolators in correlated \emph{factor} regression models --
  exact risk analysis.
\newblock 2024.
\newblock available online at arxiv.

\bibitem{TB20}
A.~Tsigler and P.~L. Bartlett.
\newblock Benign overfitting in ridge regression.
\newblock 2020.
\newblock available online at \bl{\url{http://arxiv.org/abs/2009.14286}}.

\bibitem{VCR89}
F.~Vallet, J.~G. Cailton, and Ph. Refregier.
\newblock Linear and nonlinear extension of the pseudo-inverse solution for
  learning boolean functions.
\newblock {\em Europhys. Lett.}, 9:315--320, 1989.

\bibitem{WuXu20}
D.~Wu and J.~Xu.
\newblock On the optimal weighted $\ell_2$ regularization in overparameterized
  linear regression.
\newblock In H.~Larochelle, M.~A. Ranzato, R.~Hadsell, M.~F. Balcan, and
  Hsuan-Tien Lin, editors, {\em Advances in Neural Information Processing
  Systems 33: Annual Conference on Neural Information Processing Systems 2020,
  NeurIPS 2020, December 6-12, 2020, virtual}, 2020.

\bibitem{XMRH21}
J.~Xu, A.~Maleki, K.~R. Rad, and D.~Hsu.
\newblock Consistent risk estimation in moderately high-dimensional linear
  regression.
\newblock {\em {IEEE} Trans. Inf. Theory}, 67(9):5997--6030, 2021.

\bibitem{YRC07}
Y.~Yao, L.~Rosasco, and A.~Caponnetto.
\newblock On early stopping in gradient descent learning.
\newblock {\em Constr. Approx.}, 26:289--315, 2007.

\bibitem{ZBHRV17}
C.~Zhang, S.~Bengio, M.~Hardt, B.~Recht, and O.~Vinyals.
\newblock Understanding deep learning requires rethinking generalization.
\newblock {\em ICLR}, 2017.

\bibitem{ZhangBHRV21}
C.~Zhang, S.~Bengio, M.~Hardt, B.~Recht, and O.~Vinyals.
\newblock Understanding deep learning (still) requires rethinking
  generalization.
\newblock {\em Commun. {ACM}}, 64(3):107--115, 2021.

\bibitem{ZCZG18}
D.~Zou, Y.~Cao, D.~Zhou, and Q.~Gu.
\newblock Stochastic gradient descent optimizes over\-parameterized deep relu
  networks.
\newblock 2018.
\newblock available online at
  {\small\bl{\url{http://arxiv.org/abs/1811.08888}}}.

\end{thebibliography}
\end{singlespace}

\end{document}